\documentclass{article}

\usepackage[utf8]{inputenc}
\usepackage[ruled,linesnumbered]{algorithm2e}
\usepackage[final, babel]{microtype}
\usepackage{algpseudocode}
\usepackage{amsmath,amsthm,amsfonts,amssymb,mathtools}
\usepackage{hyperref}
\usepackage{color}
\usepackage{float}
\usepackage{mathrsfs}
\usepackage{enumitem}
\usepackage{multirow}
\usepackage{booktabs}
\usepackage{makecell}
\usepackage{graphicx}
\usepackage{subfigure}
\usepackage{caption}
\usepackage{thmtools}
\usepackage{thm-restate}
\usepackage{hhline}
\usepackage{cite}
\usepackage{xcolor}
\usepackage{rotating} 
\usepackage{natbib}
\usepackage{bm}

\usepackage{microtype}
\usepackage{graphicx}
\usepackage{subfigure}
\usepackage{booktabs} %

\usepackage{hyperref}

\usepackage[accepted]{icml2024}

\usepackage[capitalize,noabbrev]{cleveref}

\renewcommand{\epsilon}{\varepsilon}

\newcommand{\trans}{^{\top}}

\newcommand{\cA}{\mathcal{A}}
\newcommand{\cB}{\mathcal{B}}

\newcommand{\cE}{\mathcal{E}}
\newcommand{\cF}{\mathcal{F}}

\newcommand{\cM}{\mathcal{M}}
\newcommand{\cN}{\mathcal{N}}

\newcommand{\cP}{\mathcal{P}}

\newcommand{\cS}{\mathcal{S}}

\newcommand{\cU}{\mathcal{U}}

\newcommand{\cX}{\mathcal{X}}

\newcommand{\cZ}{\mathcal{Z}}

\newcommand{\EE}{\mathbb{E}}

\newcommand{\ph}{{h'}}

\newcommand{\mP}{{\mathbb{P}}}
\newcommand{\mH}{{\mathbb{H}}}
\newcommand{\mR}{{\mathbb{R}}}
\newcommand{\mN}{{\mathbb{N}}}

\newcommand\numberthis{\addtocounter{equation}{1}\tag{\theequation}}

\let\hat\widehat
\let\tilde\widetilde

\newtheorem{theorem}{Theorem}[section]
\newtheorem{lemma}[theorem]{Lemma}
\newtheorem{corollary}[theorem]{Corollary}
\newtheorem{remark}[theorem]{Remark}

\newtheorem{proposition}[theorem]{Proposition}
\theoremstyle{definition}
\newtheorem{definition}[theorem]{Definition}
\newtheorem{condition}[theorem]{Condition}

\newtheorem{assumption}{Assumption}

\def\##1\#{\begin{align}#1\end{align}}
\def\$#1\${\begin{align*}#1\end{align*}}

\newcount\Comments  %
\newcommand{\kibitz}[2]{\ifnum\Comments=1\textcolor{#1}{#2}\fi}

\newcommand{\blue}[1]{{\color{blue} #1}}

\DeclareMathOperator*{\argmax}{arg\,max}
\DeclareMathOperator*{\argmin}{arg\,min}

\newcommand{\tM}{{\tilde{M}}}
\newcommand{\tpi}{{\tilde{\pi}}}
\newcommand{\hpi}{{\hat{\pi}}}
\newcommand{\MLE}{\text{MLE}}
\newcommand{\Central}{\text{Ctr}}
\newcommand{\Bridge}{\text{Br}}
\newcommand{\bM}{{\bar{M}}}

\newcommand{\NE}{\text{NE}}
\newcommand{\SNE}{\text{SNE}}

\newcommand{\dimCPE}{\dim_{{\rm CPE}|\Pi^\dagger}}
\newcommand{\dimPE}{\dim_{\rm PE}}
\newcommand{\dimE}{\dim_{\rm E}}
\newcommand{\Rank}{{\rm Rank}}
\newcommand{\Span}{\text{Span}}

\newcommand{\hvecmu}{{\hat\vecmu}}
\newcommand{\vecmu}{{\bm{\mu}}}
\newcommand{\vecnu}{{\bm{\nu}}}
\newcommand{\tvecnu}{{\tilde{\bm{\nu}}}}
\newcommand{\vecpi}{{\bm{\pi}}}

\newcommand{\vecs}{{\bm{s}}}
\newcommand{\vecL}{{\bm{L}}}
\newcommand{\veca}{{\bm{a}}}
\newcommand{\tvecpi}{{\tilde{\bm{\pi}}}}
\newcommand{\hvecpi}{{\hat{\bm{\pi}}}}
\newcommand{\vecM}{{\bm{M}}}
\newcommand{\tvecM}{{\tilde{\bm{M}}}}
\newcommand{\bvecM}{{\bar{\bm{M}}}}
\newcommand{\vecPi}{{\bm{\Pi}}}

\newcommand{\bcM}{{\bar{\cM}}}

\newcommand{\dmP}{{\ddot{\mP}}}
\newcommand{\dM}{{\ddot{M}}}

\newcommand{\dtM}{{\ddot{\tM}}}
\newcommand{\dcM}{{\ddot{\cM}}}
\newcommand{\dQ}{{\ddot{Q}}}
\newcommand{\dV}{{\ddot{V}}}
\newcommand{\dJ}{{\ddot{J}}}
\newcommand{\dr}{{\ddot{r}}}

\newcommand{\dvecM}{{\ddot{\vecM}}}

\newcommand{\veccM}{{\bm{\cM}}}

\newcommand{\dveccM}{{\ddot{\bm{\cM}}}}

\newcommand{\teps}{{\tilde{\epsilon}}}

\newcommand{\share}{\text{share}}

\newcommand{\Poly}{\text{Poly}}

\newcommand{\piref}{\pi_{\text{ref}}}
\newcommand{\ts}{\tilde{s}}
\newcommand{\ta}{\tilde{a}}
\newcommand{\tw}{{\tilde{w}}}

\newcommand{\tldh}{{\tilde{h}}}

\newcommand{\beps}{{\bar{\epsilon}}}

\newcommand{\tpsi}{\tilde{\psi}}

\newcommand{\Tr}{\text{Tr}}
\newcommand{\act}{\text{active}}

\newcommand{\MFG}{\text{MFG}}

\newcommand{\Cstr}{\text{Cstr}}

\newcommand{\RI}{{\rm \uppercase\expandafter{\romannumeral 1 \relax}}}
\newcommand{\RII}{{\rm \uppercase\expandafter{\romannumeral 2 \relax}}}

\newcommand{\dimPEI}{\dim_{\rm PE}}
\newcommand{\dimPEII}{\dim_{\rm PE}^\RII}

\newcommand{\dimCPEI}{\dim_{{\rm CPE}|\Pi^\dagger}}
\newcommand{\dimCPEII}{\dim_{{\rm CPE}|\Pi^\dagger}^\RII}

\newcommand{\para}[1]{\textbf{#1}~}

\newcommand{\MEBP}{{\texttt{MEBP}}}

\newcommand{\dimPERII}{{\dimPE}}

\newcommand{\dimMTPE}{\dim_{\rm MTPE}}

\newcommand{\Alg}{\text{Alg}}

\newcommand{\textPr}{\text{Pr}}

\newcommand{\Opt}{\text{Opt}}

\usepackage[textsize=tiny]{todonotes}

\allowdisplaybreaks
\begin{document}

\twocolumn[
\icmltitle{Model-Based RL for Mean-Field Games is not Statistically Harder than Single-Agent RL}
\icmlsetsymbol{equal}{*}

\begin{icmlauthorlist}
\icmlauthor{Jiawei Huang}{ETH}
\icmlauthor{Niao He}{ETH}
\icmlauthor{Andreas Krause}{ETH}
\end{icmlauthorlist}

\icmlaffiliation{ETH}{Department of Computer Science, ETH Zurich}

\icmlcorrespondingauthor{Jiawei Huang}{jiawei.huang@inf.ethz.ch}
\icmlkeywords{Machine Learning, ICML}

\vskip 0.3in
]

\printAffiliationsAndNotice{}  %

\begin{abstract}
    \looseness -1 We study the sample complexity of reinforcement learning (RL) in Mean-Field Games (MFGs) with model-based function approximation that requires strategic exploration to find a Nash Equilibrium policy. We introduce the Partial Model-Based Eluder Dimension (P-MBED), a more effective notion to characterize the model class complexity. Notably, P-MBED measures the complexity of the single-agent model class converted from the given mean-field model class, and potentially, can be exponentially lower than the MBED proposed by \citet{huang2023statistical}. We contribute a model elimination algorithm featuring a novel exploration strategy and establish sample complexity results polynomial w.r.t.~P-MBED. Crucially, our results reveal that, under the basic realizability and Lipschitz continuity assumptions, \emph{learning Nash Equilibrium in MFGs is no more statistically challenging than solving a logarithmic number of single-agent RL problems}. We further extend our results to Multi-Type MFGs, generalizing from conventional MFGs and involving multiple types of agents. This extension implies statistical tractability of a broader class of Markov Games through the efficacy of mean-field approximation. Finally, inspired by our theoretical algorithm, we present a heuristic approach with improved computational efficiency and empirically demonstrate its effectiveness.
\end{abstract}

\section{Introduction}
Multi-Agent Reinforcement Learning (MARL) has excelled in modeling cooperative and competitive interactions among agents in unknown environments.
However, the well-known ``curse of multi-agency'' poses a challenge in equilibrium solving for MARL systems with large populations. 
Yet, for MARL systems with symmetric agents, such as human crowds or road traffic, one can leverage such special structure by employing mean-field approximation, leading to the RL for Mean-Field Games (MFGs) setting~\citep{lasry2007mean, huang2006large}. 
Notably, MFGs offer a promising framework where the complexity of learning Nash Equilibrium (NE) needs not depend on the number of agents \citep{lauriere2022learning}. 
It has found successful applications in various domains, including financial markets \citep{cardaliaguet2018mean}, economics \citep{gomes2015economic} and energy management \citep{djehiche2016mean}. 

Similar to single-agent RL \citep{jin2018q,jiang2017contextual}, for MFGs, one of the most important questions is to understand how many samples are required to explore the unknown environment and solve the equilibrium, a.k.a. the \emph{sample complexity}.
Given the complex dynamics of mean-field systems and high cost of generating samples from large population, designing strategic exploration methods for sample-efficient learning becomes imperative. 

Existing works on learning MFGs primarily focus on model-free approaches such as Q-learning~\citep{guo2019learning,anahtarci2023q}, policy gradient~\cite{ subramanian2019reinforcement,yardim2022policy}, fictitious play~\cite{xie2021learning,perrin2020fictitious}, etc.
Several recent works further extend these model-free approaches with value function approximation to handle large state-action space~\cite{mao2022mean,zhang2023learning}. However, their sample complexity results ubiquitously rely on strong structural assumptions such as contractivity \citep{guo2019learning} or monotonicity \citep{perolat2021scaling}.
Their methods, moreover, are usually specialized and lack generalizability, leaving an open challenge of efficiently exploring mean-field systems \emph{without those structural assumptions}.

To address this gap, \citet{huang2023statistical} establish general sample complexity results for model-based RL in MFGs\footnote{Model-based RL has also been explored in Mean-Field Control (MFC) setting, where all the agents are cooperative~\citep{pasztor2021efficient,huang2023statistical}.}. They introduce a complexity measure known as Model-Based Eluder Dimension (MBED) to characterize the complexity of the model function class. Their algorithm, under basic realizability and Lipschitz continuity assumptions, achieves a sample complexity upper bound polynomial w.r.t. MBED.
However, as we will show in Prop.~\ref{prop:tabular}, even for the tabular setting, MBED can be exponential in the number of states in the worst case. This observation, coupled with the tractability of tabular MFGs under additional structural assumptions \citep{guo2019learning,perolat2021scaling},  prompts a fundamental question:
\begin{center}
    \textbf{Is learning MFGs statistically harder than single-agent RL in general?} 
\end{center}
In this paper, we provide a definitive answer to this question. 
Our main contributions are summarized as follows:
\begin{itemize}[leftmargin=*]
\item In Sec.~\ref{sec:P_MBED}, we introduce a novel complexity measure for any given mean-field model class $\cM$, called \emph{Partial Model-Based Eluder Dimension (P-MBED)}. 
P-MBED represents the complexity of the single-agent model class derived from $\cM$ after (adversarially) fixing the state density for the transition functions in $\cM$. We show that P-MBED can be significantly lower than MBED \citep{huang2023statistical}. 
For example, in the tabular setting, P-MBED is always bounded by the number of states and actions, yielding an exponential improvement over MBED.

\item In Sec.~\ref{sec:learning_MFG}, we propose a model elimination algorithm capable of exploring the mean-field system and returning an approximate NE policy with sample complexity polynomial w.r.t. P-MBED. 
From the algorithmic perspective, our results indicate that under the basic realizability and Lipschitz assumptions, \emph{learning MFGs is no more statistically challenging than solving $\log|\cM|$ single-agent RL problems}.
As a direct implication, the sample complexity of tabular MFGs only polynomially depends on the number of states, actions, horizon and $\log|\cM|$. This is the first result indicating that learning tabular MFGs is provably sample-efficient in general, even without the contractivity or monotonicity assumptions.

\item In Sec~\ref{sec:experiments}, we design a heuristic algorithm with improved computational efficiency building upon our insights in theory. We evaluate it in a synthetic linear MFGs setting and validate its effectiveness. 
\end{itemize}

As a substantial extension,  we further examine the sample complexity of more general MFGs with heterogeneous population, specifically Multi-Type MFGs (MT-MFGs)~\citep{ghosh2020model, subramanian2020multi,perolat2021scaling}.
MT-MFGs comprise multiple types of agents with distinct transition models, reward functions or even state-action spaces. MT-MFGs have stronger capacity in modeling the diversity of agents, while being more tractable than general Markov Games\footnote{The general Markov Games (MGs) framework considers individually distinct agents. However, this generality comes with challenges. Existing results in MGs are restricted in learning (Coarse) Correlated Equilibria~\citep{jin2021v,bai2020near,daskalakis2023complexity} and the sample complexity in function approximation setting may still depend on the number of agents \citep{wang2023breaking,cui2023breaking}. MT-MFGs can be regarded an intermediary between standard MFGs and general MGs.}.
However, the fundamental sample complexity in the setting remains largely unexplored. Our additional contribution includes:

\begin{itemize}[leftmargin=*]
    \item In Sec.~\ref{sec:MT_MFG}, we show that finding the NE in an MT-MFG is equivalent to finding the NE in a lifted MFG with constraints on policies. Building on this insight, we establish the first sample complexity upper bound for learning MT-MFGs. Our results identify statistical tractability of a broad class of MARL systems, potentially offering new insights to the sample complexity analysis for solving NE in general Markov Games.
\end{itemize}

\subsection{Related Work}\label{sec:related_work} 
Within the abundant literature on single-agent RL and MFGs, below we focus primarily on sample complexity results for solving these problems in unknown environments. We defer additional related works to Appx.~\ref{appx:related_works}.

\textbf{Single-Agent RL}\quad  When the number of states and actions is extremely large, sample complexity bounds derived for tabular RL \citep{auer2008near, azar2017minimax, jin2018q} become vacuous. Instead, function approximation is usually considered, where a model or value function class containing the true model or optimal value functions is available, and the sample complexity is governed by the complexity of the function classes \citep{jin2020provably, agarwal2020flambe, jiang2017contextual, sun2019model, jin2021bellman, du2021bilinear, foster2021statistical}.
Compared with single-agent RL, the main challenge in MFGs is the additional dependence on density in transition and reward functions, especially that the density space is continuous. 
Although our P-MBED is inspired by the eluder dimension in the single-agent setting \citep{russo2013eluder, osband2014model, levy2022eluderbased}, it is a novel complexity notion in characterizing the sample efficiency of RL in MFGs.

\textbf{Mean-Field Games}\quad 
Most existing results for learning MFGs primarily focus on tabular setting and model-free approaches \citep{guo2019learning, xie2021learning, cui2021approximately, elie2020convergence},
where strong structural assumptions, such as contractivity \citep{guo2019learning}, monotonicity and density independent transition \citep{perrin2020fictitious}, or non-vanishing regularization \citep{yardim2022policy} are usually required. 
In contrast, we focus on addressing the fundamental exploration challenge for general MFGs.
\citet{mishra2020model} study non-stationary MFG without strong structural assumptions, but their algorithm is inefficient and no sample complexity results were provided.
Beyond the tabular setting, \citet{huang2023statistical} is the most related to us. However, as implied by our results in this paper, their sample complexity bound are suboptimal.

\noindent\textbf{Multi-Type Mean-Field Games}\quad
\citet{subramanian2020multi} study more general multi-type cases, but they consider the transition model depending on action density instead of state density.
Besides, the multi-type setting has been investigated in special cases, such as LQR \citep{moon2018linear, uz2023reinforcement}, and leader-follower structures \citep{vasal2022master}.
\citet{ghosh2020model} is the most related to us. However, they consider the discounted stationary setting and assume the state density is fixed, while ours is more challenging since we need to keep tracking the evolution of state density.
\citet{perolat2021scaling} also consider the multi-type setting, but they require the monotonicity assumption.
Moreover, they only provide asymptotic rates without sample complexity guarantees.

\section{Background}\label{sec:preliminary}

\para{Notations}
Throughout the paper, 
we will use standard big-oh notations $O(\cdot),\Omega(\cdot),\Theta(\cdot)$, and notations such as $\tilde{O}(\cdot)$ to (partially) suppress logarithmic factors.
In Appx.~\ref{appx:notations}, we list all the frequently used notations in this paper.

\subsection{Mean-Field Games}
\para{Mean-Field Markov Decision Process}
We consider the finite-horizon non-stationary Mean-Field MDP (MF-MDP) $M := (\mu_1,\cS,\cA,H,\mP_{M},r)$, where $\mu_1$ is the known initial state distribution; $\cS=(\cS_1=...=\cS_H)$ and $\cA=(\cA_1=...=\cA_H)$ are the state and action spaces, which are discrete but can be arbitrarily large; $\mP_M:=\{\mP_{M,h}\}_{h\in[H]}$ with $\mP_{M,h}:\cS_h\times\cA_h\times\Delta(\cS_h)\rightarrow\Delta(\cS_{h+1})$ is the transition function and $r:=\{r_h\}_{h\in[H]}$ with $r_h:\cS_h\times\cA_h\times\Delta(\cS_h)\rightarrow[0,\frac{1}{H}]$ is the deterministic reward function, where $\Delta(\cX)$ denotes the probability measure over $\cX$.
We use $\Pi:=\{\pi:=\{\pi_h\}_{h\in[H]}|\pi_h:\cS_h\rightarrow\Delta(\cA_h)\}$ to denote the policy class including all non-stationary Markovian policies, and we only focus on policies in $\Pi$.
Given $\pi \in \Pi$ and initial density $\mu^\pi_{M,1}:=\mu_1$, the state density $\mu^\pi_M:=\{\mu^\pi_{M,h}\}_{h\in[H]}$ evolves according to $\mu^\pi_{M,h+1} = \Gamma^\pi_{M,h}(\mu^\pi_{M,h}), h\in[H]$ where $\Gamma^\pi_{M,h}(\mu_h)(\cdot) := \sum_{s_h,a_h} \mu_h(s_h)\pi(a_h|s_h) \mP_{M,h}(\cdot|s_h,a_h,\mu_h)$. 

Given any $\pi,\tpi\in\Pi$, we use $\EE_{\tpi,M(\pi)}[\cdot]$ to denote the expectation over trajectories generated by executing policy $\tpi$ 
while fixing the transitions and rewards to $\mP_{M,h}(\cdot|\cdot,\cdot,\mu^\pi_{M,h})$, $r_h(\cdot,\cdot,\mu^\pi_{M,h})$.
These trajectories can be interpreted as the observations of a deviated agent taking $\tpi$ while all the others take $\pi$.
Besides, we define $V^{\tpi}_{M,h}(\cdot;\mu^\pi_{M}):=\EE_{\tpi,M(\pi)}[$ $\sum_{\ph=h}^H r_\ph(s_\ph,a_\ph,\mu^\pi_{M,\ph})|s_h=\cdot]$ to be the value function at step $h$ if the agent deploys policy $\tpi$ in model $M$ conditioning on $\pi$, and define $J_M(\tpi;\pi):=\EE_{s_1\sim\mu_1}[V^\tpi_{M,1}(s_1;\mu^\pi_M)]$ to be the total return of policy $\tpi$ conditioning on $\pi$. The Nash Equilibrium (NE) $\pi^\NE_M$ of model $M$ is defined to be the policy s.t. no agent tends to deviate, i.e., 
    $
    \forall \tpi \in \Pi,~J_M(\tpi;\pi^\NE_M) \leq J_M(\pi^\NE_M;\pi^\NE_M).
    $
We denote $\cE^\NE_M(\pi) := \max_{\tpi} \Delta_M(\tpi,\pi)$ to be the NE-Gap, where $\Delta_M(\tpi,\pi) := J_M(\tpi;\pi) - J_M(\pi;\pi)$.

\noindent\textbf{Model-Based Setting}
In our model-based setting, the learner can get access to a transition function class $\cM\subset\{\{\mP_{M,h}\}_{h\in[H]}|\forall h,\mP_{M,h}:\cS_h\times\cA_h\times\Delta(\cS_h)\rightarrow\Delta(\cS_{h+1})\}$ to approximate the true model $M^*$.
We assume the reward function $r$ is known. In Appx.~\ref{appx:extension_to_unknown_reward}, we provide informal discussion about how to extend our results to the setting when $r$ is unknown.
Our main objective is to find an $\epsilon$-approximate NE $\hpi^\NE_{M^*}$, satisfying $\cE^\NE_{M^*}(\hpi^\NE_{M^*}) \leq \epsilon$.
Same as \citet{huang2023statistical}, we only make two basic assumptions about the function class $\cM$: realizability and Lipschitz continuity.
\begin{assumption}[Realizability]\label{assump:realizability}
    $M^* \in \cM$.
\end{assumption}
\begin{assumption}[Lipschitz Continuity]\label{assump:Lipschitz}
    For any $M\in\cM$, and arbitrary policies $\pi,\tpi \in \Pi$, $\forall h,s_h,a_h$, we have:
    \begin{align*}
        & \|\mP_{M,h}(\cdot|s_h,a_h,\mu_{M,h}^\pi) - \mP_{M,h}(\cdot|s_h,a_h,\mu_{M,h}^{\tpi})\|_1 \\
        &\quad\quad\quad\quad\quad\quad\quad \leq L_T \|\mu_{M,h}^\pi-\mu_{M,h}^\tpi\|_1,\\
          &  |r_h(s_h,a_h,\mu_{M,h}^\pi) - r_h(s_h,a_h,\mu_{M,h}^\tpi)\|_1 \\ &\quad \quad\quad\quad\quad\quad\quad \leq L_r \|\mu_{M,h}^\pi-\mu_{M,h}^\tpi\|_1. 
    \end{align*}
\end{assumption}
Note that our Assump.~\ref{assump:Lipschitz} only requires Lipschitz continuity on feasible densities. In contrast, contractivity assumes $L_r$ and $L_T$ are sufficiently small \citep{guo2019learning,yardim2022policy}, and prior works considering monotonicity \citep{perolat2021scaling,zhang2023learning} usually assume the transition is independent w.r.t.~density, i.e., $L_T = 0$.

\noindent We consider the same trajectory sampling model as \citet{huang2023statistical},
which is much weaker than the generative model assumptions requiring trajectories conditioning on arbitrary state densities in most MFGs literatures \citep{guo2019learning,perrin2020fictitious,anahtarci2023q}.
\begin{definition}\label{def:collection_process}
    The sampling model can be queried with arbitrary $\tpi,\pi\in\Pi$, and return a trajectory by executing $\tpi$ while transition and reward functions are fixed to $\mP_{M^*,h}(\cdot|\cdot,\cdot,\mu^\pi_{M^*,h})$ and $r_h(\cdot,\cdot,\mu^\pi_{M^*,h})$ for all $h$.
\end{definition}

\noindent\textbf{MFGs and $N$-Player Symmetric Anonymous Games}
MFGs can be regarded as the limit of Symmetric Anonymous Games (SAGs) when the number of agents $N$ approaches infinity~\citep{guo2019learning,yardim2022policy}.
As explained in \citep{huang2023statistical}, the sampling model (Def.~\ref{def:collection_process}) is reasonable for $N$-player SAGs with central controllers, which can manipulate all the agents' policies.
Given a SAG, it is known that the NE of its MFG approximation is a $O(N^{-1/2})$-approximate NE for the SAG \citep{yardim2024mean}, if all the agents execute that same NE policy.
In this way, one may interpret our setting as centralized training with decentralized execution.

\subsection{Multi-Type Mean-Field Games}
\noindent\textbf{Multi-Type Mean-Field MDP}
A finite horizon non-stationary Multi-Type (or Multi-Group) MF-MDP $\vecM$ with $W$ types of agents can be described by a collection of tuples $\vecM:=\{(\mu_1^w,\cS^w,\cA^w,H,\mP_\vecM^w,r^w)_{w\in[W]}\}$, where  we use $w$ in superscription to distinguish the initial state distribution, state-action spaces, the transition and reward functions in different groups. Besides, for any $w$, the transition and reward functions depend on densities in all types. More concretely, we have $\mP_{\vecM}^w:=\{\mP_{\vecM,h}^w\}_{h\in[H]}$ with $\mP_{\vecM,h}^w:\cS^w_h\times\cA^w_h\times\Delta(\cS^1_h)\times...\times\Delta(\cS^W_h)\rightarrow \Delta(\cS^w_{h+1})$ and $r^w:=\{r^w_h\}_{h\in[H]}$ with $r_h^w:\cS^w_h\times\cA^w_h\times\Delta(\cS^1_h)\times...\times\Delta(\cS^W_h)\rightarrow [0,\frac{1}{H}]$.
For each type of agents, we consider the Markovian policies $\Pi^w:=\{\pi^w:=\{\pi_h^w\}_{h\in[H]}|\forall h,~\pi_h^w:\cS^w_h\rightarrow\Delta(\cA^w_h)\}$, and use $\vecPi:=\{\vecpi:=\{\pi^w\}_{w\in[W]}|\forall w\in[W],\pi^w\in\Pi^w\}$ to denote the collection of policies for all types.
For the function approximation setting, we assume $W$ function classes $\cM^1,...,\cM^W$ are available, where $\forall w\in[W]$, $\cM^w\subset\{\{\mP^w_h\}_{h\in[H]}|\forall h\in[H],~\mP^w_h:\cS^w_h\times\cA^w_h\times\Delta(\cS^1_h)\times...\times\Delta(\cS^W_h)\rightarrow\Delta(\cS^w_h)\}$ is used to approximate the transition function for the $w$-th group. The MT-MFG function class $\veccM$ is then defined by $\veccM \gets \{\vecM := M^1\times...\times M^W|\forall w\in[W],~M^w\in\cM^w\}$, which we use to approximate the true model $\vecM^*$.

For the lack of space, we defer the definitions of value functions, Nash Equilibrium, and other related details to Appx.~\ref{appx:more_details_setting_MT}. 
For the assumptions in MT-MFG setting, we defer to Appx.~\ref{appx:MT_assumps_defs}.

\section{Partial Model-Based Eluder Dimension}\label{sec:P_MBED}
In the function approximation setting, the exploration challenge is related to the complexity of the function classes.
In this section, we introduce new notions to characterize the complexity of model function class for MFGs and its extension to Multi-Type MFGs setting. 
The proofs and additional discussions can be found in Appx.~\ref{appx:P_MBED}.

Inspired by the Eluder dimension of single-agent value function classes \citep{russo2013eluder,jin2021bellman} and mean-field model function classes \citep{huang2023statistical}, similarly, we use the length of independent sequences to characterize the complexity of function classes.
In Def.~\ref{def:eps_independent}, we first introduce the definition of standard $\epsilon$-independence in previous Eluder dimension literature, to highlight the difference from our partial $\epsilon$-independence.
Although we only consider the $l_1$-distance here, similar discussion can be generalized to other distances, e.g., the Hellinger distance.
\begin{definition}[$\epsilon$-Independence; \citep{huang2023statistical}]\label{def:eps_independent}
    Given $\cM$ and a data sequence $\{(s^i_h,a^i_h,\mu^i_h)\}_{i=1}^n \subset \cS_h\times\cA_h\times\Delta(\cS_h)$, we say $(s_h,a_h,\mu_h)$ is $\epsilon$-independent of $\{(s^i_h,a^i_h,\mu^i_h)\}_{i=1}^n$ w.r.t. $\cM$ if there exists $M,\tM\in\cM$ such that $\sum_{i=1}^{n} \|\mP_{M,h}(\cdot|s^i_h,a^i_h,\mu^i_h)- \mP_{\tM,h}(\cdot|s^i_h,a^i_h,\mu^i_h)\|_1^2 \leq \epsilon^2$ but $\|\mP_{M,h}(\cdot|s_h,a_h,\mu_h)- \mP_{\tM,h}(\cdot|s_h,a_h,\mu_h)\|_1 > \epsilon$.
    We call $\{(s^i_h,a^i_h,\mu^i_h)\}_{i=1}^n$ an $\epsilon$-independent sequence w.r.t. $\cM$ (at step $h$) if for any $i\in[n]$, $(s^i_h,a^i_h,\mu^i_h)$ is $\epsilon$-independent w.r.t. $\{(s^t_h,a^t_h,\mu^t_h)\}_{t=1}^{i-1}$.
\end{definition}
\begin{definition}[Partial $\epsilon$-Independence]\label{def:partial_eps_indp}
    Given $\cM$, a mapping $\nu_h:\cM\rightarrow\Delta(\cS_h)$, and a data sequence $\{(s_h^i,a_h^i)\}_{i=1}^n $ $\subset \cS_h\times\cA_h$, we say $(s_h,a_h)$ is partially $\epsilon$-independent of $\{(s^i_h,a^i_h)\}_{i=1}^n\subset \cS_h\times\cA_h$ w.r.t. $\cM$ and $\nu_h$, if there exists $M,\tM\in\cM$, s.t. $\sum_{i=1}^{n} \|\mP_{M,h}(\cdot|s^i_h,a^i_h,\nu_h(M))- \mP_{\tM,h}(\cdot|s^i_h,a^i_h,\nu_h(\tM))\|_1^2 \leq \epsilon^2$ but $\|\mP_{M,h}(\cdot|s_h,a_h,$ $\nu_h(M)) - \mP_{\tM,h}(\cdot|s_h,a_h,\nu_h(\tM))\|_1 > \epsilon$.
    We call $\{(s^i_h,a^i_h)\}_{i=1}^n$ a partially $\epsilon$-independent sequence w.r.t. $\cM$ and $\nu_h$ (at step $h$) if for any $i\in[n]$, $(s^i_h,a^i_h)$ is partially $\epsilon$-independent on $\{(s^t_h,a^t_h)\}_{t=1}^{i-1}$.
\end{definition}
Intuitively, a partially $\epsilon$-independent sequence of $\cM$ is  an independent sequence w.r.t. the function class converted from $\cM$ by using some mapping $\nu_h$ to ``partially'' fix the input (the density part) for each function in $\cM$.
We use $\dim_{{\rm E}|\nu_h}(\cM,\epsilon)$ to denote the length of the longest partially $\epsilon$-independent sequence w.r.t. $\cM$ and $\nu_h$ (at step $h$). 
Now, we are ready to define the Partial-MBED.
\begin{definition}[Partial MBED]\label{def:Partial_Eluder_Dim}
    Given a model class $\cM$, and a policy $\pi$, we define the mapping $\nu^\pi_h$: $\forall M\in\cM$, $\nu^\pi_h(M) := \mu^\pi_{M,h}$.
    The P-MBED of $\cM$ is defined by:
    $\dimPERII(\cM,\epsilon) := \max_{h\in[H]} \max_\pi \dim_{{\rm E}|\nu^\pi_h}(\cM,\epsilon)$. 
\end{definition}
By definition, P-MBED can be interpreted as the complexity of the \emph{single-agent model class} converted from the Mean-Field model class $\cM$ by \emph{partially (adversarially) fixing the density} of the functions' input.
In fact, different choices of $\nu$ in Def.~\ref{def:partial_eps_indp} may lead to different notions of complexity. In our main text, we stick to the choice in Def.~\ref{def:Partial_Eluder_Dim}, but in Appx.~\ref{appx:alter_P_MBED}, we discuss an alternative choice of $\nu$, its induced P-MBED and associated properties.

Next, we take the tabular setting as an example, and show that P-MBED of any function class for tabular MFGs can be controlled by $|\cS||\cA|$, while MBED \citep{huang2023statistical} can be exponential in $|\cS|$ in the worst case.
This is reasonable given the single-agent nature of P-MBED.
\begin{restatable}{proposition}{PropTabular}\label{prop:tabular}
    (Tabular Setting) For any $\cM$ and $\epsilon > 0$, $\dimPE(\cM,\epsilon) \leq |\cS||\cA|$, while there exists a concrete example of $\cM$ such that $\dimE(\cM,\epsilon) = \Omega(\exp(|\cS|))$.
\end{restatable}
Below we provide the linear mean-field model classes with decomposable transition functions as another example.
When the transition is independent w.r.t. density $\mu$ (i.e. $G(\mu)$ is constant), the linear MFGs reduce to the single-agent linear MDP \citep{jin2020provably}.
As we can see, the P-MBED of the model class of linear MFGs is only related to the dimension of the state-action feature, which matches the complexity of their single-agent correspondence.
\begin{proposition}[Linear MFGs; Informal version of Prop.~\ref{prop:MBED_Linear_MFMDP_formal}]\label{prop:Linear_MF_MDP_informal} Consider the  model class: $\cM_\Psi := \{\mP_\psi|\mP_\psi(\cdot|s,a,\mu):=\phi(s,a)\trans G(\mu)\psi(s');\psi\in \Psi\}$, with known feature $\phi(\cdot,\cdot)\in\mR^{\tilde{d}}$, $G(\cdot)\in\mR^{\tilde{d}\times d}$, and a next-state feature class $\Psi$ satisfying some normalization conditions. Then $\dimPE(\cM,\epsilon) = \tilde{O}(\tilde{d})$.
\end{proposition}

Similarly, for model classes in Multi-Type MFGs setting, we can define the Multi-Type P-MBED generalized from $\dimPE$ in MFGs, which we denote as $\dimMTPE$. We defer its formal definition to Appx.~\ref{appx:P_MBED_MT}. 
Likewise, $\dimMTPE$ can be regarded as the complexity measure for a collection of $W$ single-agent model classes converted from $\veccM$.
In the tabular case (resp. Prop.~\ref{prop:Tabular_MF_MDP_MT}), we have $\dimMTPE(\veccM,\epsilon') = \tilde{O}(\sum_{w\in[W]}|\cS^w||\cA^w|)$, and in linear MT-MFG setting with decomposable transitions (resp. Prop.~\ref{prop:Linear_MF_MDP}), $\dimMTPE(\veccM,\epsilon')=\tilde{O}(\sum_{w\in[W]} d^w)$ where $\{d^w\}_{w\in[W]}$ are the dimensions of the state-action features.
\section{Sample Efficiency of Learning in MFGs}\label{sec:learning_MFG}
\begin{algorithm*}
    \textbf{Input}: Model Class $\cM$; Parameters $\epsilon_0,\teps,\beps,\delta$.\\
    \textbf{Initialize}: $\cM^1 \gets \cM$, $\delta_0 \gets \frac{\delta}{\log_2|\cM| + 1}$\\
    \For{$k=1,2,...$}{
        $\pi^k \gets \arg\min_{\pi}|\cB_{\pi}^{\epsilon_0}(M_\Central^\pi;\cM^k)|$; \label{line:find_distinguish_policy}\\
        \lIf{$|\cB_{\pi^k}^{\epsilon_0}(M_\Central^{\pi^k};\cM^k)| \leq \frac{|\cM^k|}{2}$}{\label{line:if_branch}
            $\cM^{k+1} \gets {\texttt{ModelElim}}(\pi^k, \cM^k, \teps, \delta_0)$
        }
        \Else{\label{line:else_branch}
            $\pi^{\NE,k}_{\Bridge} \gets \texttt{BridgePolicy}(\cM^k,\beps)$; \quad $\cM^{k+1} \gets {\texttt{ModelElim}}(\pi^{\NE,k}_{\Bridge}, \cM^k, \teps, \delta_0)$;\\
            Randomly pick $\tM^k$ from $\cM^{k+1}$; \quad $\cE^\NE_{\tM^k}(\pi^{\NE,k}_\Bridge) \gets \max_\pi \Delta_{\tM^k}(\pi,\pi^{\NE,k}_\Bridge)$\\ 
            \lIf{$\cE^\NE_{\tM^k}(\pi^{\NE,k}_\Bridge) \leq \frac{3\epsilon}{4}$}{\label{line:else_if_branch}
                \Return{$\pi^{\NE,k}_\Bridge$}
            }
        }
        \lIf{$|\cM^k| = 1$}{
            Return the NE of the model
        }
    }
    \caption{\MEBP: \textbf{M}odel \textbf{E}limination via \textbf{B}ridge \textbf{P}olicy}\label{alg:learning_with_DCP}
\end{algorithm*}
In this section, we show that the sample complexity of learning NE in MFGs is indeed governed by our new complexity notion P-MBED. We highlight our main algorithm and sample complexity results in Sec.~\ref{sec:main_results_MFG}, and then explain details in the algorithm design and technical novelty in Sec.~\ref{sec:discussion_MFG}. 
The missing details and proofs for results in this section are deferred to Appx.~\ref{appx:learning_with_DCP}.

\subsection{Main Algorithm and Highlight of Main Results}\label{sec:main_results_MFG}

Before proceeding to the algorithms, we first introduce several useful notions. Given a reference policy $\pi$, we denote 
$
d(M,\tM|\pi) := \max_\tpi~~d^{\tpi}(M,\tilde{M}|\pi) \vee d^{\tpi}(\tilde{M}, M,|\pi)$
as the conditional model distance between $M$ and $\tilde{M}$, where 
$d^{\tpi}(M,\tilde{M}|\pi):=\EE_{\tpi,M(\pi)}[\sum_{h=1}^H\|\mP_{M,h}(\cdot|\cdot,\cdot,\mu^\pi_{M,h})-\mP_{\tM,h}(\cdot|\cdot,\cdot,\mu^\pi_{\tM,h})\|_1].$
Given a model class $\cM'$, any $M\in\cM'$, and any policy $\pi$, we define the $\epsilon_0$-neighborhood of $M$ in $\cM'$ to be: $\cB_{\pi}^{\epsilon_0}(M;\cM') := \{M'\in\cM' | d(M,M'|\pi)\leq \epsilon_0\}$.
The ``Central Model'' (abbr. CM) in $\cM'$ w.r.t. policy $\pi$ is defined to be the model with the most number of neighbors: $M_{\Central}^{\epsilon_0}(\pi;\cM') := \arg\max_{M\in\cM'} |\cB_{\pi}^{\epsilon_0}(M;\cM')|$.
Besides, when $\epsilon_0$ and $\cM'$ is clear from the context, we will use $M_\Central^\pi$ as a short note of $M_{\Central}^{\epsilon_0}(\pi;\cM')$.
Lastly, $\forall \pi,\pi'\in\Pi$, we define $d_{\infty,1}(\pi,\pi'):=\max_{h,s_h}\|\pi(\cdot|s_h)-\pi'(\cdot|s_h)\|_1$.

We provide our main algorithm in Alg.~\ref{alg:learning_with_DCP}.
The basic idea is to find a sequence of ``reference policies'' ($\pi^k$ or $\pi^{\NE,k}_\Bridge$, $k=1,2,...$) and run the model elimination steps (Alg.~\ref{alg:elimination_DCP_formal} as \texttt{ModelElim}) to gradually remove models in $\cM$ that distinct from $M^*$ conditioning on these reference policies, until find an approximate NE.
Next, we highlight our main results and its implications.

\begin{restatable}{theorem}{ThmMainResult}[Informal version of Thm.~\ref{thm:sample_complexity_MFG}]\label{thm:sample_complexity_MFG_informal}
    Under Assump.~\ref{assump:realizability} and~\ref{assump:Lipschitz}, 
    with appropriate hyperparameter choices, 
    w.p. $1-\delta$, Alg.~\ref{alg:learning_with_DCP} terminates at some $k\leq \log_2|\cM| + 1$ and returns an $\epsilon$-NE of $M^*$
   after consuming at most 
   $$
   \tilde{O}\left(\frac{H^7}{\epsilon^2}(1+L_rH)^2 \dimPE(\cM,\epsilon') \log^3\frac{|\cM|}{\delta}\right)
   $$
   trajectories, 
   where $\epsilon'=O(\frac{\epsilon}{H^3(1+L_rH)(1+L_T)^H})$.
\end{restatable}

\para{Model-Based RL for MFGs is not Statistically Harder than Single-Agent RL}
As we will explain more in the next section, \texttt{ModelElim} only needs to be a \emph{single-agent model elimination subroutine}, and it is the only step consuming samples.
Therefore, Thm.~\ref{thm:sample_complexity_MFG_informal} suggests that the sample complexity of learning MFGs can be characterized by a $O(\log|\cM|)$ number of single-agent model elimination sub-problems, whose learning complexity is controlled by P-MBED.
As a result, the total sample complexity only scales with P-MBED and the log-covering number of $\cM$.

Based on the discussion in Sec.~\ref{sec:P_MBED}, we can expect for many model classes with low P-MBED (e.g. tabular setting Prop.~\ref{prop:tabular}, linear setting Prop.~\ref{prop:Linear_MF_MDP_informal}), learning MFGs is provable sample-efficient.  In particular, for tabular MFGs where $\dimPE(\cM,\epsilon') \leq |\cS||\cA|$, our result yields a sample complexity with polynomial dependence on $|\cS|,|\cA|, H$, which implies that \emph{tabular MFGs are provably efficient in general} if considering the model-based function approximation, even without assuming contractivity or monotonicity that are often required in existing works~\citep{guo2019learning,yardim2022policy, perrin2020fictitious}.   
Compared with recent results in function approximation setting for MFGs \citep{huang2023statistical} or MFC \citep{pasztor2021efficient} with similar Lipschitz assumptions, our result does not suffer the exponential term $(1+L_T)^{H}$ (see Remark~\ref{remark:exp_disappears} for more explanation).

\textbf{Additional Remarks on $\log|\cM|$}\quad 
Although low log-covering number of function class is regarded as a standard assumption in many MARL works \citep{cui2023breaking,wang2023breaking}, we would like to take the tabular MFGs as an example and provide remarks about the magnitude of $\log|\cM|$.
Under Assump.~\ref{assump:Lipschitz}, with appropriate discretization, the $\epsilon$-cover for all possible transition functions could be $ \Omega(\exp(SAH\cN_\epsilon(\Delta(\cS))))$, where $\cN_\epsilon(\Delta(\cS))$ denotes the covering number of density space $\Delta(\cS)$ and we omit $L_T,L_r$.
As a result, in the worst case, $\log|\cM|=\Omega(\cN_\epsilon(\Delta(\cS)))$, which could be exponential in $SA$.
Nonetheless, there are many examples, such that, even in the worst case, $\log|\cM|$ is acceptable.
For instance, if the model class is parameterized by some $\theta\in\Theta$ (e.g. Neural Networks) and take the concatenation of $[s,a,\mu]\in\mR^{\dim(\cS) + \dim(\cA) + S}$ as inputs, where we use $\dim(\cdot)$ to denote the dimension of a given set. 
Then $\log|\cM|=\tilde{O}(\dim(\Theta))$, which could just scale with $\tilde{O}(\Poly(\dim(\cS), \dim(\cA), S)\})$.
As another example, when the transition function only depends on some sufficient statistics of density instead of the exact density, (e.g. $\mP(\cdot|s,a,\mu) = \mP(\cdot|s,a,\EE_{\tilde{s}\sim\mu}[\tilde{s}])$), we may have $\log|\cM|=\tilde{O}(\Poly(S,A,H))$ in the worst case. Note that, for the single-agent RL, the largest log-covering number of models is also bounded by $\tilde{O}(\Poly(S,A,H))$ (folklore).

\begin{algorithm*}
    \textbf{Input}: Reference Policy $\pi$; $\bcM$; $\teps,\delta$\\
    \textbf{Initialize}: $\bcM^1 \gets \bcM$;~$\cZ^0\gets\{\}$; Set $T$ by Thm.~\ref{thm:Elimination_Alg};\\
    \For{$t=1,2,...,T$}{
        $\tpi^t \gets \arg\max_{\tpi}\max_{M,M' \in \bcM^t}\EE_{\tpi,M(\pi)}[$ $\sum_{h=1}^H \|\mP_{M,h}(\cdot|\cdot,\cdot,\mu^\pi_{M,h}) - \mP_{M',h}(\cdot|\cdot,\cdot,\mu^\pi_{M',h})\|_1]$. \label{line:max_margin_formal}\\
        $\Delta_{\max}^t\gets $ the maximal value achieved above. \\
        \lIf{$\Delta_{\max}^t \leq \teps$}{\label{line:elimination_termination_formal}
            \Return $\bcM^t$
        }
        \Else{
            \For{$h=1,2...,H$}{\label{line:start_for}
                Query Sampling Oracle (Def.~\ref{def:collection_process}) with $(\pi,\pi)$; collect the data at step $h$: $z^t_h:=\{(s_h^{t},a_h^{t},s_{h+1}'^{t})\}$. \label{line:sample_1}\\
                Query Sampling Oracle (Def.~\ref{def:collection_process}) with $(\tpi^t,\pi)$; collect the data at step $h$: $\tilde{z}^t_h:=\{(s_h^{t},a_h^{t},s_{h+1}'^{t})\}$.\\
                $\cZ^t \gets \cZ^{t-1} \cup z^t_h \cup \tilde{z}^t_h$.\label{line:end_for}\label{line:sample_2}
            }
            $\forall M \in \bcM^t$,
            $l^\pi_\MLE(M;\cZ^t):=\sum_{i\in[t],h\in[H]} \log \mP_{M,h}(s_{h+1}'^i|s_h^i,a_h^i,\mu^\pi_{M,h}) + \log \mP_{M,h}(\ts_{h+1}'^i|\ts_h^i,\ta_h^i,\mu^\pi_{M,h}).
            $\label{line:MLE_loss}\\
            $\bcM^{t+1} \gets \{M\in\bcM^t|~l^\pi_{\MLE}(M;\cZ^t) \geq \max_{\tM}l^\pi_{\MLE}(\tM;\cZ^t) - \log\frac{HT|\cM|}{\delta}\}$.
        }
    }
    \Return $\bcM^T$.
    \caption{Model Elimination given a Policy}\label{alg:elimination_DCP_formal}
\end{algorithm*}

\textbf{Exponential Separation between MFGs and MFC}~
Different from MFGs, in Mean-Field Control (MFC) setting, agents cooperate to find an optimal policy to maximize the total return. 
Previous work \citep{huang2023statistical} suggests that both MFC and MFGs can be solved via a unified MLE framework with similar sample complexity upper bounds. 
One natural question is: \emph{whether learning MFC can also be as sample-efficient as single-agent RL}?

We provide a negative answer to this question. In Thm.~\ref{thm:LB_MFC}, we show that even in tabular setting, there exists  a hard instance such that learning MFC requires $\Omega(\exp(|\cS|))$ samples.
This suggests an exponential separation between learning MFC and MFGs from information-theoretical perspective. Intuitively, for MFC, in the worst case, the agent should explore the entire $\cS\times\cA\times\Delta(\cS)$ space to identify the policy that achieves the maximal return.
In contrast, as we will explain in Lem.~\ref{lem:local_alignment}, in MFGs setting, the learner does not have to explore the entire state-action-density space; instead, finding a ``locally-aligned equilibrium policy'' is enough. 

\subsection{Algorithm Design and Proof Sketch}\label{sec:discussion_MFG}

\subsubsection{\texttt{ModelElim}: The Model Elimination Step}
\texttt{ModelElim} can be arbitrary single-agent model elimination procedures.
Here we provide an example in Alg.~\ref{alg:elimination_DCP_formal}.

The basic idea of Alg.~\ref{alg:elimination_DCP_formal} is to eliminate models not aligned with $M^*$ conditioning on the given reference policy $\pi$.
In each iteration, we first find a tuple $(\tpi^t,M^t,M'^t)$ resulting in the maximal discrepancy $\Delta_{\max}^t$.
As long as $\Delta_{\max}^t > \teps$, we collect samples and remove models with low likelihood.
With high probability, on the one hand, $M^*$ will never be ruled out under Assump.~\ref{assump:realizability}; 
on the other hand, the growth of $\sum_t \Delta_{\max}^t$ is controlled by P-MBED.
As a result, the algorithm will terminate eventually and return a model class only including those $M$ with small $d(M,M^*|\pi)$.
We summarize the result in the theorem below.
\begin{restatable}{theorem}{ThmElimAlg}[Informal version of Thm.~\ref{thm:Elimination_Alg}]\label{thm:Elimination_Alg_Informal}
    Given any reference policy $\pi,\teps,\delta \in (0,1)$, if $M^*\in\bcM$, by choosing $T = \tilde{O}(\frac{H^4 \dimPE(\cM,\epsilon')}{\teps^2})$
    with $\epsilon' = O(\frac{\teps}{H^{2}(1+L_T)^{H}})$,
    w.p. $1-\delta$,
    Alg.~\ref{alg:elimination_DCP_formal} terminates at some $T_0 \leq T$, and return $\bcM^{T_0}$ s.t. (i) $M^* \in \bcM^{T_0}$ (ii) $\forall M \in \bcM^{T_0}$, $d(M^*,M|\pi)\leq \teps$.
\end{restatable}

We claim Alg.~\ref{alg:elimination_DCP_formal} is a single-agent model elimination subroutine, because from Line~\ref{line:start_for}-\ref{line:end_for}, we can see that Alg.~\ref{alg:elimination_DCP_formal} only eliminates those $M\in\cM^k$ s.t. $\mP_{M,h}(\cdot|\cdot,\cdot,\mu^{\pi}_{M,h})$ distinct from $\mP_{M^*,h}(\cdot|\cdot,\cdot,\mu^{\pi}_{M^*,h})$ under some adversarial policy $\tpi^t$ or $\pi$ itself.
Here the density part of $\mP_{M,h}$ is fixed by $\mu^\pi_M$, so during the elimination, all the transitions reduce to single-agent functions only depending on states and actions.

\para{Beyond P-MBED} 
Notably, although we focus on P-MBED in this paper, one may consider other complexity measures generalized from single-agent RL setting \citep{sun2019model, foster2021statistical} and our analysis can be extended correspondingly.
That's because as long as \texttt{ModelElim} satisfies the (i) and (ii) in Thm.~\ref{thm:Elimination_Alg_Informal}, it can be arbitrary and does not affect the function of other components in Alg.~\ref{alg:learning_with_DCP}.
\subsubsection{Fast Elimination with Bridge Policy}
We seek to construct reference policies that allow to eliminate models as efficient as possible, more specifically, to halve the model candidates every iteration until finding the NE.
We first consider the simple case, where the models are ``scattered'' and easy to be distinguished: there exists a policy $\pi^k$, such that, no more than $|\cM^k|/2$ models are around its CM (resp. \texttt{If}-branch, Line~\ref{line:if_branch} in Alg.~\ref{alg:learning_with_DCP}).
In this case, after running \texttt{ModelElim} with $\pi^k$, $\cM^{k+1}$ only contains those models locating at the neighborhood of $M^*$ conditioning $\pi^k$, which implies $|\cM^{k+1}| \leq |\cB_{\pi^k}^{\epsilon_0}(M_\Central^{\pi^k};\cM^k)|\leq |\cM^{k}|/2$.

The challenging scenario is that, for any policy, the corresponding CM is surrounded by over a half of models (resp. \texttt{Else}-branch, Line~\ref{line:else_branch} in Alg.~\ref{alg:learning_with_DCP}). In that case, unstrategically selecting reference policies leads to inefficient model elimination. We present a subtle choice of reference policy, called \emph{Bridge Policy}, which can be constructed by Alg.~\ref{alg:BridgePolicy}.
Before diving into the details of our constructions, we first explain the key insights behind it.
Our first insight is summarized in the lemma below.
\begin{restatable}{lemma}{LemLocalAlign}[Implication of Local Alignment]\label{lem:local_alignment}
    Given any $M, \tM$ with transition $\mP_M$ and $\mP_{\tM}$, denote $\hpi^\NE_M$ to be an $\epsilon_1$-approximate NE of $M$, suppose $d(M,\tM|\hpi^\NE_M) \leq \epsilon_2$, then $\hpi^\NE_M$ is also an $O(\epsilon_1+\epsilon_2)$-approximate NE of $\tM$.
\end{restatable}
Lem.~\ref{lem:local_alignment} states that, if two models $M$ and $\tM$ align with each other conditioning on the NE of one of them, then they approximately share that NE.
Therefore, in the \texttt{Else}-branch, after calling \texttt{ModelElim} with $\pi^{\NE,k}_\Bridge$ as the reference policy, if the NE-Gap $\cE^{\NE}_{\tM^k}(\pi^{\NE,k}_\Bridge)$ is small for some randomly selected $\tM^k\in\cM^{k+1}$ (resp. Line~\ref{line:else_if_branch}), we can claim $\pi^{\NE,k}_\Bridge$ is an approximate NE of $M^*$ by Lem.~\ref{lem:local_alignment}.

However, the remaining challenge is that, if $\cE^{\NE}_{\tM^k}(\pi^{\NE,k}_\Bridge)$ is large, we cannot conclude anything about it.
Hence, $\pi^{\NE,k}_\Bridge$ should be chosen in a strategic way, so that in this case, we can guarantee the elimination is efficient, i.e. $|\cM^{k+1}| \leq |\cM^{k}|/2$.
Our second key insight to overcome this challenge is summarized in Thm.~\ref{lem:exists_bridge_policy}, which indicates that the existence of a ``Bridge Policy'' that coincides with the NE of its corresponding CM.
\begin{restatable}{theorem}{ThmExistBrPi}[Bridge Policy]\label{lem:exists_bridge_policy}
    If the \texttt{Else}-branch in Line~\ref{line:else_branch} in Alg.~\ref{alg:learning_with_DCP} is activated, running Alg.~\ref{alg:BridgePolicy} returns a bridge policy $\pi^{\NE,k}_\Bridge$, such that, $\pi^{\NE,k}_\Bridge$ is an approximate NE of $M^{\pi^{\NE,k}_\Bridge}_\Central$.
\end{restatable}

Before we explain how to prove Thm.~\ref{lem:exists_bridge_policy}, we first check the implication of this result.
Based on Thm.~\ref{lem:exists_bridge_policy}, if $\cE^{\NE}_{\tM^k}(\pi^{\NE,k}_\Bridge) > \frac{3\epsilon}{4}$ in Line~\ref{line:else_if_branch}, by Lem.~\ref{lem:local_alignment}, we know $d(M^{\pi^{\NE,k}_\Bridge}_\Central,M^*|\pi^{\NE,k}_\Bridge)$ cannot be small.
Therefore, 
with appropriate hyperparameter choices, we can assert that all models in the neighborhood $\cB^{\epsilon_0}_{\pi^{\NE,k}_\Bridge}(M^{\pi^{\NE,k}_\Bridge}_\Central,\cM^k)$ should have been eliminated, implying $|\cM^{k+1}|\leq |\cM^k|/2$.

Combining the discussions above, we know our Alg.~\ref{alg:learning_with_DCP} guarantees to either return an approximate NE, or at least halve the model sets. 
We summarize to the following theorem, which paves the way to our main theorem Thm.~\ref{thm:sample_complexity_MFG_informal}.
\begin{restatable}{theorem}{ThmIfElseBranches}\label{thm:if_else_branches}
    In Alg.~\ref{alg:learning_with_DCP}, by choosing $\epsilon_0 = \frac{\epsilon}{8(1+L_rH)(H+4)}$, $\teps = \frac{\epsilon_0}{6}$, and choosing $\beps$ according to Thm.~\ref{thm:ConstructionErr}, w.p. $1-\delta$,
    (1) if the \texttt{If-Branch} in Line~\ref{line:if_branch} is activated: we have $|\cM^{k+1}| \leq |\cM^k|/2$;
    (2) otherwise, in the \texttt{Else-Branch} in Line~\ref{line:else_branch}: either we return the $\pi^{\NE,k}_\Bridge$ which is an $\epsilon$-approximate NE for $M^*$; or the algorithm continues with $|\cM^{k+1}| \leq |\cM^k|/2$.
\end{restatable}

\begin{algorithm*}
    \textbf{Input}: MF-MDP model class $\cM$; $\epsilon_0, \beps$;\\
    Convert $\cM$ to a PAM class $\dcM$ via Eq.~\eqref{eq:conversion_to_dM}.\\
    Construct $\beps$-cover of the policy space $\Pi$ w.r.t. $d_{\infty,1}$, denoted as $\Pi_\beps$ (see Def.~\ref{prop:eps_cover_Pi}).\\
    \lFor{$\tpi \in \Pi_\beps$}{
        Find the Central Model $\dM_\Central^{\epsilon_0}(\tpi;\dcM) \gets \arg\max_{\dM \in \dcM}|\cB^{\epsilon_0}_\pi(\dM;\dcM)|$
    }
    Construct the new PAM $\dM_\Bridge$ s.t. for any $s_h,a_h,\pi$,
    $
    \dmP_{\Bridge,h}(\cdot|s_h,a_h,\pi) := \frac{\sum_{\tpi \in \Pi_\beps}[2\beps - d_{\infty,1}(\pi,\tpi)]^+\dmP_{\dM_\Central^{\epsilon_0}(\tpi;\dcM),h}(\cdot|s_h,a_h,\tpi)}{\sum_{\tpi \in \Pi_\beps}[2\beps - d_{\infty,1}(\pi,\tpi)]^+},$ $\dr_{\Bridge,h}(s_h,a_h,\pi) := \frac{\sum_{\tpi \in \Pi_\beps}[2\beps - d_{\infty,1}(\pi,\tpi)]^+\dr_{\dM_\Central^{\epsilon_0}(\tpi;\dcM),h}(s_h,a_h,\tpi)}{\sum_{\tpi \in \Pi_\beps}[2\beps - d_{\infty,1}(\pi,\tpi)]^+}
    $,
    \qquad\qquad\qquad where $[x]^+ := \max\{0,x\}$\\
    Compute the NE of $\dM_\Bridge$: $\pi^\NE_\Bridge \gets \arg\min_{\pi} \max_{\tpi}\dJ_{\dM_\Bridge}(\tpi;\pi) - \dJ_{\dM_\Bridge}(\pi;\pi)$. \label{line:return_NE}\\
    \Return $\pi^\NE_\Bridge$.
    \caption{Bridge Policy Construction}\label{alg:BridgePolicy}
\end{algorithm*}
\paragraph{Proof Sketch of Thm.~\ref{lem:exists_bridge_policy}}
An informal way to interpret the existence of such bridge policy in Thm.~\ref{lem:exists_bridge_policy} is to consider a mapping $\mathcal{T}$ from an arbitrary $\pi\in\Pi$ to the NE of its CM $M^{\pi}_\Central$. 
Then, Thm.~\ref{lem:exists_bridge_policy} states that $\mathcal{T}$ has an approximate fixed point $\pi^\NE_\Bridge \approx \mathcal{T}(\pi^\NE_\Bridge)$.
However, given that it's hard to evaluate the continuity of $\mathcal{T}$ and moreover, $\mathcal{T}$ can be a one-to-many mapping if multiple NEs exist, we prove the existence of such $\pi^{\NE}_\Bridge$ by the non-trivial construction in Alg.~\ref{alg:BridgePolicy}. We leave the connection between our proofs and the fixed-point theorems as an open problem.

Before explaining our construction in Alg.~\ref{alg:BridgePolicy}, we first introduce a new notion called \emph{``Policy-Aware Model''} (abbr. PAM) denoted by $\dM$. 
The main motivation for introducing PAM is that we want to focus on the policy space, because the feasible densities $\{\mu_{M^*,h}^\pi\}_{\pi\in\Pi}$ may not cover the entire density space $\Delta(\cS_h)$, and it is not easy to characterize.
We defer to Appx.~\ref{appx:details_PAM} for the formal definition of PAM and also new notations in Alg.~\ref{alg:BridgePolicy} (e.g. $\dM_\Central^{\epsilon_0}(\cdot,\cdot)$ denotes Central Model, $\dJ$ denotes the total return), and only summarize the main idea here to save space. Briefly speaking, a PAM $\dM:=\{\cS,\cA,\mu_1,H,\dmP,\dr\}$ is an MDP whose transition $\dmP:\cS\times\cA\times\Pi\rightarrow\Delta(\cS)$ and reward functions $\dr:\cS\times\cA\times\Pi\rightarrow[0,\frac{1}{H}]$ depend on state, action and a ``reference policy''.
PAM can be regarded as a higher-level abstraction of MF-MDP (i.e. MF-MDP $\subset$ PAM), where we replace the dependence on $\mu^\pi_M$ in MF-MDP by $\pi$.
We can convert a MF-MDP $M$ to a PAM $\dM$ sharing the same $\cS,\cA,\mu_1,H$ by assigning the following for any $h\in[H]$ with $\mu_{M,1}^\pi = \mu_1$, $\forall \pi\in\Pi$:
\begin{align*}
    \dmP_{\dM,h}(\cdot|\cdot,\cdot,\pi) :=& \mP_{M,h}(\cdot|\cdot,\cdot,\mu^\pi_{M,h}) \numberthis \label{eq:conversion_to_dM}\\ 
    \dr_{\dM,h}(\cdot,\cdot,\pi) :=& r_h(\cdot,\cdot,\mu^\pi_{M,h}),~ \mu_{M,h+1}^\pi \gets \Gamma^\pi_{M,h}(\mu_{M,h}^\pi).
\end{align*}

In Alg.~\ref{alg:BridgePolicy}, we first convert each MF-MDP to its PAM version.
Then, we find an $\beps$-cover of the policy space w.r.t. $d_{\infty,1}$, denoted by $\Pi_\teps$, and construct the ``Bridge PAM'' $\dM_\Bridge$ by interpolating among CMs w.r.t. $\tpi \in \Pi_\beps$.
Here the weights $[2\beps - d_{\infty,1}(\pi,\tpi)]^+$ is chosen carefully for the following considerations:
\begin{itemize}[leftmargin=*]
    \item 
\textbf{(I)} since $\Pi_\beps$ is an $\beps$-cover, for any $\pi\in\Pi$, the denominator $\sum_{\tpi\in\Pi_\beps} [2\beps - d_{\infty,1}(\pi,\tpi)]^+$ is always larger than $\beps$, which implies both $\dmP_{\Bridge,h}$ and $\dr_{\Bridge,h}$ are well-defined and continuous in $\pi$.
The continuity is important since it implies that $\dM_\Bridge$ has at least one NE (Def.~\ref{def:NE_in_PAM}), denoted as $\pi^\NE_\Bridge$;

\item 
\textbf{(II)} $[2\beps - d_{\infty,1}(\pi,\tpi)]^+$ decays to zero if $\pi$ largely disagrees with $\tpi$, so $\dmP_{\Bridge}(\cdot|\cdot,\cdot,\pi)$ is only determined by CMs of those $\tpi$ close to $\pi$.
\end{itemize}

Next, we discuss what we can conclude from the above two points. 
Based on the triggering condition in Line~\ref{line:else_branch} in Alg.~\ref{alg:learning_with_DCP}, for any $\pi,\tpi$, the neighbors of $M_{\Central}^\pi$ and $M_{\Central}^\tpi$ share at least one common model $M_\share$. By using $M_\share$ as a bridge, we have $\|\dmP_{\dM_{\Central}^\pi}(\cdot|\cdot,\cdot,\pi) - \dmP_{\dM_{\Central}^\tpi}(\cdot|\cdot,\cdot,\tpi)\|_1 = O(d_{\infty,1}(\pi,\tpi))$.
Combining with \textbf{(II)}, we know $\forall \pi,~\|\dmP_{\Bridge}(\cdot|\cdot,\cdot,\pi) - \dmP_{\dM_{\Central}^\pi}(\cdot|\cdot,\cdot,\pi)\|_1 = O(\beps)$, which implies $\dmP_{\Bridge}(\cdot|\cdot,\cdot,\pi^\NE_\Bridge)\approx \mP_{\dM^{\pi^\NE_\Bridge}_\Central}(\cdot|\cdot,\cdot,\pi^\NE_\Bridge)$ if $\beps$ is small enough.
By the definition of NE in PAM, the conversion rules in Eq.~\eqref{eq:conversion_to_dM} and Lem.~\ref{lem:local_alignment}, we can conclude that $\pi^\NE_\Bridge$ is an approximate NE of $M^{\pi^\NE_\Bridge}_\Central$, and finish the proof of Thm.~\ref{lem:exists_bridge_policy}.

\section{Learning in Multi-Type MFGs}\label{sec:MT_MFG}
In this section, we extend our results to the more general Multi-Type MFGs setting\footnote{
    The existence of NE in MT-MFGs can be found in Thm.~\ref{prop:existence_MT_MFG}.
}, allowing to address heterogeneous agents. 

\para{Reduction to Lifted MFGs with Constrained Policy}
Our key observation is that, a MT-MFG $\vecM$ to can be lift to a new MF-MDP $M_\MFG:=\{\cS_\MFG,\cA_\MFG,\mu_1,H,\mP_\MFG,r_\MFG\}$ by augmenting the original states and actions with the type index.
The new state and action spaces are given by: $\cS_\MFG:=\bigcup_{w\in[W]}\{\cS^w\times \{w\}\}$ and $\cA_\MFG:=\bigcup_{w\in[W]}\{\cA^w\times \{w\}\}$.
We defer the detailed description for the conversion process and the definition of initial state distribution, transition and reward functions in $M_\MFG$ to Appx.~\ref{appx:conversion}.

For policies in $M_\MFG$, we only consider $\Pi^\dagger := \{\pi|\forall w\in[W],\pi(a^w\circ w|s^w\circ w) = \pi^w(a^w|s^w),~\pi^w\in\Pi^w\}$, including all policies which only take actions with the same type as the states.
Similar to the NE defined in full policy space $\Pi$, we can define the ``constrained NE'' when agents are constrained to only take policies in the subset $\Pi^\dagger$.
More concretely, we call $\hpi^\NE_\Cstr\in\Pi^\dagger$ the $\epsilon$-approximate Constrained Nash Equilibrium if $\forall \pi\in\Pi^\dagger,~J_{M_\MFG}(\pi,\hpi^\NE_\Cstr) \leq J_{M_\MFG}(\hpi^\NE_\Cstr,\hpi^\NE_\Cstr) + \epsilon$.
The following property reveals the connection between constrained NE in $M_\MFG$ and the NE in the original multi-type model $\vecM$.
\begin{restatable}{proposition}{PropNEConversion}\label{prop:NE_conversion}
    Given a MT-MFG $\vecM$ and its lifted MFG $M_\MFG$, we have:
    (1) an $\epsilon$-constrained NE $\hpi^\NE_\Cstr\in\Pi^\dagger$ for $M_\MFG$ is a ($W\epsilon$)-NE in $\vecM$; 
    (2) an $\epsilon$-NE $\hvecpi^\NE$ in $\vecM$ is an $\epsilon$-constrained NE for $M_\MFG$.
\end{restatable}
The above result not only implies the existence of constrained NE in $M_\MFG$ given the existence of NE in $\vecM$ by letting $\epsilon\rightarrow 0$, but also suggests one can solve NE of MT-MFG by solving the constrained NE in its lifted MFG.
The second point is very important since the constrained NE can be solved via almost the same procedures in Sec.~\ref{sec:learning_MFG}, as long as we constrain the policy space to $\Pi^\dagger$.
We defer algorithm details to Appx.~\ref{appx:MT_algorithm}, and summarize our main result in the following theorem.

\begin{restatable}{theorem}{ThmMFGCstr}[Informal version of Thm.~\ref{thm:sample_complexity_MT_MFG}]\label{thm:sample_complexity_MT_MFG_informal}
    Under Assump.~\ref{assump:realizability_MT} and~\ref{assump:Lipschitz_MT}, there exists an algorithm (Alg.~\ref{alg:learning_with_DCP_Cstr}), s.t.
    w.p. $1-\delta$, 
    it returns an $\epsilon$-NE of $\vecM^*$ after consuming at most 
    $$
    \tilde{O}(\frac{W^2H^7}{\epsilon^2}(1+\vecL_rH)^2\dimMTPE(\veccM,\epsilon'))
    $$ 
    trajectories, where $\epsilon'=O(\frac{\epsilon}{WH^3(1+\vecL_rH)(1+\vecL_T)^H})$.
\end{restatable}

In Appx.~\ref{appx:approx_MT_MFG}, we investigate a practical multi-agent system called $N$-player Multi-Type Symmetric Anonymous Games (MT-SAGs) generalized from SAGs.
We establish approximation error between MT-MFGs and MT-SAGs.
Our results reveal a larger class of Multi-Agent systems where NE can be solved in a sample-efficient way.

\section{A Heuristic Algorithm with Improved Computational Efficiency}\label{sec:experiments}
Although Alg.~\ref{alg:learning_with_DCP} is sample-efficient, it requires exponential computation.
In this section, we aim to design a heuristic algorithm\footnote{The code is available at \url{https://github.com/jiaweihhuang/Heuristic_MEBP}.} sharing the main insights as Alg.~\ref{alg:learning_with_DCP} while more computationally tractable.
For the lack of space, we defer the concrete algoirthm (Alg.~\ref{alg:empirical_version}), the experiment setting and evaluation results (Fig.~\ref{fig:experiments}) to Appx.~\ref{appx:experiments}.
In this section, we just highlight the algorithm design.

\para{Highlights of Algorithm Design}
We assume a NE Oracle is available, such that given a known MFG model, the Oracle can return its NE.
We argue that such oracle can be easily implemented if the model is smooth enough or the monotonicity condition is satisfied \citep{guo2019learning, perolat2021scaling}.
Besides, in our experiments, we observe that repeatedly mixing the policy with its best response can converge to a good solution.
Given such oracle, Alg.~\ref{alg:empirical_version} only involves $|\cM|$ calls of NE oracle, and $\Poly(|\cM|, |\cS|, |\cA|, H)$ arithmetic operations in computing model difference or likelihood, which avoids exponential computation in Alg.~\ref{alg:learning_with_DCP}.

For the algorithm design, Alg.~\ref{alg:empirical_version} follows the same if-else structure as Alg.~\ref{alg:elimination_DCP_formal}, but we improve the computational efficiency in two aspects.
Firstly, we avoid procedures optimizing over the entire policy class, including Line~\ref{line:find_distinguish_policy} in Alg.~\ref{alg:learning_with_DCP} and Line~\ref{line:max_margin_formal} in Alg.~\ref{alg:elimination_DCP_formal}.
Instead, we only search over the NE policies of model candidates, which can be computed by calling the NE Oracle $|\cM|$ times at the beginning.
As long as the models in $\cM$ are diverse enough, we can expect their NEs to be reasonable representatives for $\Pi$ in distinguishing models.
Secondly, we replace the $\pi_\Bridge^{\NE,k}$ in Alg.~\ref{alg:learning_with_DCP} with the NE of the model $M^k\gets\argmax_{M\in\cM^k}|\cB^{\epsilon_0}_{\pi^{\NE}_M}(M,\cM^k)|$, and do not have to solve the NE of the complicated bridge model in Alg.~\ref{alg:BridgePolicy}.
We claim that this modification still aligns with Alg.~\ref{alg:BridgePolicy} in principle. Note that the main intuition behind Alg.~\ref{alg:BridgePolicy} is that, when Line~\ref{line:else_branch} in Alg.~\ref{alg:learning_with_DCP} is activated, the reference policy used for elimination should be a policy $\piref$, such that, $\piref$ collapses with the NE of the model with the maximal number of neighbors conditioning on $\piref$.

\section{Conclusion}
In this paper, we reveal that learning MFGs can be as sample-efficient as single-agent RL under mild assumptions, and the sample complexity of RL in MFGs can be characterized by a novel complexity measure called Partial Model-Based Eluder Dimension (P-MBED).
Besides, we extend our algorithms to the more general Multi-Type MFGs setting.
Lastly, we contribute an empirical algorithm with improved computational efficiency.

As for the future, one interesting direction is to study the sample complexity when only value function approximations are available.
Besides, while our focus is the sample efficiency in this paper, it would be valuable to identify general conditions, under which computationally efficient algorithms exist.
Lastly, our results underscore the power of mean-field approximation, and it would be worthwhile to investigate other generalizations of the MFGs setting, in order to deepen our understanding on the sample efficiency of learning NE in other MARL systems.

\section*{Impact Statement}
This paper presents work whose goal is to advance the field of Machine Learning. There are many potential societal consequences of our work, none which we feel must be specifically highlighted here.

\section*{Acknowledgements}
This research was supported by Swiss National Science Foundation (SNSF) Project Funding No. 200021-207343, SNSF starting grant, SNSF grant agreement 51NF40 180545, and
by the European Research Council (ERC) under the European Union’s Horizon grant 815943.

\bibliography{references}
\bibliographystyle{icml2024}

\newpage
\appendix
\onecolumn

\section*{Outline of the Appendix}
\begin{itemize}
    \item Appx.~\ref{appx:notations}: Frequently used notations.
    \item Appx.~\ref{appx:related_works}: Additional related works.
    \item Appx.~\ref{appx:extension_to_unknown_reward}: Informal Discussions about how to extend our results to unknown reward function setting.
    \item Appx.~\ref{appx:P_MBED}: Missing details and proofs related to P-MBED.
    \item Appx.~\ref{appx:proofs_PAM}: Missing details related to Single/Multi-Type Policy Aware Models (PAM).
    \item Appx.~\ref{appx:learning_with_DCP}: Proofs for lemma and theorems related to learning MFGs in Sec.~\ref{sec:learning_MFG}.
    \item Appx.~\ref{appx:proof_MT_PAM}: Proofs for lemma and theorems related to learning MultiType MFGs in Sec.~\ref{sec:MT_MFG}.
    \item Appx.~\ref{appx:approx_MT_MFG}: Introduction to Multi-Type Symmetric Anonymous Games (MT-SAGs) and approximation error between MT-MFGs and MT-SAGs.
    \item Appx.~\ref{appx:basic_lemma}: Some basic lemma useful in our proofs.
    \item Appx.~\ref{appx:experiments}: Experiment details and results.
\end{itemize}

\section{Frequently Used Notations}\label{appx:notations}

\begin{table}[h]
    \centering
    \def\arraystretch{1.2}
    \begin{tabular}{c|c}
    \hline
    \textbf{Notation} & \textbf{Explanation} \\ \hline
    $M$ & Mean-Field MDP \\
    $\cM$ & Model class for (single-type) Mean-Field MDP \\
    $\pi$ & Non-stationary policy for (single-type) Mean-Field MDP  \\
    $Q_M,V_M,J_M$ & Value functions for (single-type) MF-MDP \\
    $\dimPE(\cM,\epsilon)$ & Type-$\RI$ Partial-MBED, Def.~\ref{def:Partial_Eluder_Dim}\\
    $\dimPEII(\cM,\epsilon)$ & Type-$\RII$ Partial-MBED, Def.~\ref{def:Partial_Eluder_Dim_II} \\
    $d(\cdot,\cdot|\pi)$ & Conditional distance given a reference policy $\pi$ \\
    $\cB^{\epsilon_0}_\pi(M,\cM)$ & $\epsilon_0$-neighborhood of $M$ in $\cM$ conditioning on $\pi$ \\
    $M^{\epsilon_0}_{\Central}(\pi,\cM)$ & The ``Central Model'' \\
    $d_{\infty,1}(\pi,\pi')$ & Policy distance \\
    $\vecM$ & Multi-Type Mean-Field MDP \\
    $\veccM$ & Model class for multi-type Mean-Field MDP \\
    $\vecpi$ & Non-stationary policy for multi-type Mean-Field MDP \\
    $Q_\vecM,V_\vecM,J_\vecM$ & Value functions for multi-type MF-MDP \\
    $\dimMTPE(\veccM,\epsilon)$ & Type-$\RI$ Multi-Type Partial-MBED, Def.~\ref{def:P_MBED_MT}\\
    $\dimMTPE^\RII(\veccM,\epsilon)$ & Type-$\RII$ Multi-Type Partial-MBED, Def.~\ref{def:P_MBED_MT} \\
    $\dM/\dvecM$ & Policy-aware model (single-type/multi-type) \\
    $\dcM/\dveccM$ & Model class for the policy-aware model (single-type/multi-type) \\
    $\dQ_\dM,\dV_\dM,\dJ_\dM/\dQ_\dvecM,\dV_\dvecM,\dJ_\dvecM$ & Value functions for the policy-aware model (single-type/multi-type)\\
    \hline
    \end{tabular}
\end{table}

\newpage
\section{Additional Related Works}\label{appx:related_works}
\paragraph{Single-Agent/Multi-Agent RL with General Function Approximation}
For the single-agent RL with function approximation setting, besides the literature we mentioned in the main text, there are multiple other insightful works \citep{zanette2020learning, modi2021model, xie2022role, uehara2021representation,huang2022towards, chen2022general,zhong2022posterior, ayoub2020model}.

As for the multi-agent setting, sample complexity of Markov Games has been extensively studied in both tabular \citep{jin2021v,bai2020near,chen2022almost,zhang2019policy,zhang2021multi} and function approximation setting \citep{huang2021towards, ni2022representation, wang2023breaking,cui2023breaking,foster2023complexity}.
These papers study a general MARL setting with individually distinct agents, which is quite different from our MFG or MT-MFG. Besides, many of them study the decentralized training setting, which requires much less communication cost than our centralized setting.
However, because of the difficulty in learning NE in general Markov Games setting, most of them focus on the convergence to weaker notions of equilibrium instead, e.g. the Correlated Equilibrium or the Coarse Correlated Equilibria, and those results in function approximation setting \citep{wang2023breaking,cui2023breaking} may still depend on the number of agents, although in polynomial.
In contrast, although we specify in mean-field approximation setting, we can have more ambitious goals on solving Nash Equilibrium, and our sample complexity bounds are totally independent w.r.t. the number of agents.
Moreover, we also reveal some cases when learning (MT-)MFG can be as sample-efficient as single-agent RL by investigating the Partial Model-Based Eluder Dimension.

\section{Extension to the Setting when the Reward Function is Unknown}\label{appx:extension_to_unknown_reward}
We remark that our current results extend to the unknown reward setting. Below we elucidate the key modifications needed for this extension. 

Firstly, for the problem setup, we instead assume a model class $\mathcal{M}$ available, where each element $M := (r_M, \mathbb{P}_M) \in \mathcal{M}$ corresponds to a (reward, transition) tuple.
The definition of the P-MBED can be amended by incorporating both reward and transition differences in Def.~\ref{def:Partial_Eluder_Dim}.

Secondly, for the algorithm design:
\begin{itemize}
    \item For Algorithm~\ref{alg:learning_with_DCP}, we redefine the model distance the definition $d^{\tpi}(M,\tilde{M}|\pi)$ (introduced at the beginning of Sec.~\ref{sec:main_results_MFG}) to include the expectation of distances in both reward and transition functions:
    $$d^{\tpi}(M,\tilde{M}|\pi) := \mathbb{E}_{\tilde{\pi}, M(\pi)}[\sum_{h=1}^H \|r_{M,h}(\cdot|\cdot,\cdot,\mu^\pi_{M,h})-r_{\tilde{M},h}(\cdot|\cdot,\cdot,\mu^\pi_{\tilde{M},h})\|_1 + \|\mathbb{P}_{M,h}(\cdot|\cdot,\cdot,\mu^\pi_{M,h})-\mathbb{P}_{\tilde{M},h}(\cdot|\cdot,\cdot,\mu^\pi_{\tilde{M},h})\|_1].$$ 
    The definition of $\epsilon_0$-neighborhood and “Central Model” will adjust correspondingly. 

    \item For Algorithm~\ref{alg:elimination_DCP_formal}: we should augment the reward difference into the right-hand side of Line~\ref{line:max_margin_formal}, integrate reward into the dataset in Lines~\ref{line:sample_1} and~\ref{line:sample_2}, and include the likelihood of reward functions in Line~\ref{line:MLE_loss}.

    \item For Algorithm 3, the construction of bridge policy will follow the new definition of model distance $d^{\tpi}(M,\tilde{M}|\pi)$.
\end{itemize}

Finally, for the analysis, based on the modified algorithms, under realizability assumption, we can extend Lemma~\ref{lem:Partial_Eluder_Bound} and prove that the accumulative estimation errors of reward and transition are controlled by P-MBED. The current analysis can be seamlessly generalized to establish sample complexity upper bounds depending on the P-MBED of reward and transition function classes.
\newpage
\section{Missing Details about Partial Model-Based Eluder Dimension}\label{appx:P_MBED}
\subsection{Alternative Notions of Partial MBED in MFGs}\label{appx:alter_P_MBED}
In this section, we introduce a choice of $\nu$ different from the one in Def.~\ref{def:partial_eps_indp}, which also leads to a valid P-MBED. 
In the following, we introduce another choice of $\nu$, which results in a different P-MBED. We will call it "Type $\RII$" P-MBED to distinguish the one in Def.~\ref{def:Partial_Eluder_Dim}.
\begin{definition}[Type $\RII$ Partial MBED]\label{def:Partial_Eluder_Dim_II}
    Given a model class $\cM$, 
    define the mapping $\nu^\pi_{M^*,h}:\cM\rightarrow\Delta(\cS_h)$ such that $\forall M\in\cM$, $\nu^\pi_{M^*,h}(M) := \mu^\pi_{M^*,h}$, then the type $\RII$ P-MBED of $\cM$ is defined by:
    $$
    \dimPE^\RII(\cM,\epsilon) := \max_\pi \max_{h\in[H]} \dim_{{\rm E}|\nu^\pi_{M^*,h}}(\cM,\epsilon).
    $$
\end{definition}
We want to highlight here that each of the two types P-MBED has advantages over the other.
As we will see in the Thm.~\ref{thm:sample_complexity_MFG}, if we use $\dimPEII$ to derive the sample complexity upper bound, we have to suffer the exponential term of $(1+L_T)^H$.
On the other hand, in the following proposition, we can see $\dimPEII$ is directly comparable with MBED \citep{huang2023statistical} ($\alpha = 1$ case), 
while we can not have the similar guarantee for $\dimPE$.
\begin{proposition}[Low MBED $\subset$ Low Type $\RII$ P-MBED]\label{prop:PMBED_vs_MBED}
    $\dimPEII(\cM,\epsilon) \leq \dimE(\cM,\epsilon)$.
\end{proposition}

\subsection{Proofs Related to P-MBED in MFGs Setting}

\PropTabular*
\begin{proof}
    When the density is fixed, any MF-MDP reduces to a single-agent MDP.
    For any single-agent MDP, there are at most $|\cS||\cA|$ different $(s_h,a_h)$ pairs, for any $h$. Therefore, the P-MBED can be upper bounded by $|\cS||\cA|$.

    In contrast, for MBED in \citep{huang2023statistical}, we consider the model class constructed in Thm.~\ref{thm:LB_MFC}.
    Consider the sequence $\{(s_h^1,a_h^1,\mu^i)\}_{i\in[n]}$ with $\mu^i \in \cU_{\zeta=\lfloor\frac{L_T}{5\epsilon}\rfloor}$ for all $i\in[n]$, but $\mu^i \neq \mu^j$ if $i\neq j$. For any $i\in[n-1]$, there exists two models $\mP_{\mu^i}$ and $\mP_{\mu^{i+1}}$, such that,
    \begin{align*}
        \sum_{t=1}^{i-1}\|\mP_{\mu^i}(\cdot|s_h^1,a_h^1,\mu^t) - \mP_{\mu^{i+1}}(\cdot|s_h^1,a_h^1,\mu^t)\|_1^2 = 0
    \end{align*}
    but
    \begin{align*}
        \|\mP_{\mu^i}(\cdot|s_h^1,a_h^1,\mu^i) - \mP_{\mu^{i+1}}(\cdot|s_h^1,a_h^1,\mu^i)\|_1 = 4\epsilon.
    \end{align*}
    Note that $|\cU_{\zeta=\lfloor\frac{L_T}{5\epsilon}\rfloor}| = O((\frac{L_T}{S\epsilon})^{S-1})$, by choosing $\epsilon \leq \frac{L_T}{2S}$, we have $\dimE(\cM_h,\epsilon)=\Omega(\exp(S))$.
\end{proof}
Similarly, we can show the type $\RII$ P-MBED in tabular setting can also be upper bounded by $|\cS||\cA|$, because there are at most $|\cS||\cA|$ different state-action tuples.
\begin{proposition}[Type $\RII$ P-MBED in the Tabular Setting]
    $\dimPE^\RII(\cM,\epsilon) \leq |\cS||\cA|$.
\end{proposition}

Next, we study the linear setting. Given a mapping $f:\cS\rightarrow \mR^d$, we use $\Rank([f(x)]_{x\in\cX})$ to denote the rank of matrix concatenated by $[f(x)]_{x\in\cX} \in \mR^{|\cX|\times d}$.
\begin{proposition}[Linear Setting; Formal version of Prop.~\ref{prop:Linear_MF_MDP_informal}]\label{prop:MBED_Linear_MFMDP_formal}
    Consider the Low-Rank MF-MDP with known feature $\phi:\cS\times\cA\times\Delta(\cS)\rightarrow \mR^d$ satisfying $\|\phi\| \leq C_\phi$, and unknown next state feature $\psi:\cS\rightarrow\mR^d$. Given a next state feature function class $\Psi$ satisfying $\forall \psi\in\Psi,~\forall s'\in\cS,~\forall g:\cS\rightarrow \{-1,1\}$, $\|\sum_{s'}\psi(s')g(s')\|_2\leq C_\Psi$, consider the following model class:
    \begin{align*}
        \cM_\Psi := \{\mP_\psi|\mP_\psi(\cdot|s,a,\mu):=\phi(s,a,\mu)\trans\psi(s');\forall s,a,\mu,~\mP_\psi(\cdot|s,a,\mu)\in\Delta(\cS);\psi\in \Psi\},
    \end{align*}
    we have $\dimPEII(\cM_\Psi,\epsilon) = \tilde{O}(\max_{\pi,h}\Rank([\phi_h(s_h,a_h,\mu^\pi_{M^*,h})]_{s_h\in\cS,a_h\in\cA}))$.

    Moreover, if $\phi(s,a,\mu)$ has decomposition: $\phi(s,a,\mu)\trans=\phi(s,a)\trans G(\mu)$ with $\phi(\cdot,\cdot)\in\mR^{\tilde{d}}$ and $G(\cdot)\in\mR^{\tilde{d}\times d}$,
    we have $\dimPEI(\cM_\Psi,\epsilon) = \tilde{O}(\tilde{d})$ and $\dimPEII(\cM_\Psi,\epsilon) = \tilde{O}(\min\{\tilde{d}, d\})$.
\end{proposition}
\begin{remark}
    As we can see, the P-MBED is related to the ``activated dimension'' of features after partially fixing the density, which can be much lower than its MBED $\approx d$.
    Moreover, when the feature is decomposable, the dimension of state-action feature will also serve as an upper bound.
\end{remark}
\begin{proof}
    \textbf{Proof for Type-$\RII$ P-MBED (Def.~\ref{def:Partial_Eluder_Dim_II})}
    In the following, we first consider a fixed policy $\pi$ and $h$.
    To simplify the notation, we denote $\Phi := [\phi(s_h,a_h,\mu^\pi_{M^*,h})]_{s\in\cS,a\in\cA} \in \mR^{d\times|S||A|}$ to be the matrix concatenated by vectors $\phi(\cdot,\cdot,\mu^\pi_{M^*,h})$, and denote $d_\act := \Rank(\Phi)$ to be its rank. We use $U := [u_1,u_2,...,u_{d_\act}] \in \mR^{d\times d_\act}$ to denote a normalized orthogonal basis in $\Span(\Phi) = \Span(U)$ satisfying $\|u_i\|_2 = 1$ for all $i\in[d_\act]$ and $u_i\trans u_j = 0$ for any $i\neq j$. Easy to verify that for any $s_h,a_h$, the following equation
    \begin{align*}
        U\phi_\act(s_h,a_h,\mu^\pi_{M^*,h}) = \phi(s_h,a_h,\mu^\pi_{M^*,h}).
    \end{align*}
    has a solution satisfying:
    \begin{align*}
        \|\phi_\act(s_h,a_h,\mu^\pi_{M^*,h})\|_2 = \|U\trans U \phi_\act(s_h,a_h,\mu^\pi_{M^*,h})\|_2 = \|U\trans \phi(s_h,a_h,\mu^\pi_{M^*,h})\|_2 \leq \|\phi(s_h,a_h,\mu^\pi_{M^*,h})\|_2 \leq C_\phi.
    \end{align*}
    Given a fixed policy $\pi$, $h\in[H]$, suppose $(s_h^1,a_h^1),...,(s_h^n,a_h^n)$ is a partially $\epsilon$-independent sequence w.r.t. $\cM_\Psi$ and $\nu^\pi_{M^*,h}$ defined in \ref{def:Partial_Eluder_Dim_II}.
    Then for each $i\in[n]$, there should exists $\psi^i,\tpsi^i \in \Psi$, such that:
    \begin{align*}
        \epsilon^2 \geq \sum_{t=1}^{i-1} \|\mP_{\psi^i}(\cdot|s_h^t,a_h^t,\mu^\pi_{M^*,h}), \mP_{\tpsi^i}(\cdot|s_h^t,a_h^t,\mu^\pi_{M^*,h})\|_1^2.
    \end{align*}
    and
    \begin{align*}
        \epsilon^2 \leq & \|\mP_{\psi^i}(\cdot|s_h^i,a_h^i,\mu^\pi_{M^*,h}) - \mP_{\tpsi^i}(\cdot|s_h^i,a_h^i,\mu^\pi_{M^*,h})\|_1^2\\
        =&\Big(\phi_\act(s_h^i,a_h^i,\mu^\pi_{M^*,h})\trans U\trans\sum_{s'\in\cS}({\psi^i}(s')-\tpsi^i(s'))g_{{\psi^i},\tpsi^i}(s_h^i,a_h^i,\mu^\pi_{M^*,h},s')\Big)^2\\
        \leq & \|\phi_\act(s_h^i,a_h^i,\mu^\pi_{M^*,h})\|_{(\Lambda_h^i)^{-1}}^2\|U\trans\sum_{s'\in\cS}({\psi^i}(s')-\tpsi^i(s'))g_{{\psi^i},\tpsi^i}(s_h^i,a_h^i,\mu^\pi_{M^*,h},s')\|_{\Lambda_h^i}^2.
    \end{align*}
    where we define:
    \begin{align*}
        \Lambda_h^i := & \lambda I + \sum_{t=1}^{i-1} \phi_\act(s_h^i,a_h^i,\mu^\pi_{M^*,h})\phi_\act(s_h^i,a_h^i,\mu^\pi_{M^*,h})\trans \in \mR^{d_\act\times d_\act}; \\
        g_{\psi^i,\tpsi^i}(s_h,a_h,\mu,s'):=&\begin{cases}
            1,\quad \text{if~} \phi_\act(s_h^i,a_h^i,\mu^\pi_{M^*,h})\trans U\trans({\psi^i}(s')-\tpsi^i(s')) \geq 0;\\
            -1,\quad \text{otherwise}.
        \end{cases}
    \end{align*}
    For simplicity, we use $v_{\psi,\tpsi}(s_h,a_h,\mu) := U\trans \sum_{s'}({\psi}(s')-\tpsi(s'))g_{{\psi},\tpsi}(s_h,a_h,\mu,s')$ as a shortnote. Therefore, for each $i$,
    \begin{align*}
        &\|v_{\psi^i,\tpsi^i}(s_h^i,a_h^i,\mu^\pi_{M^*,h})\|_{\Lambda_h^i}^2 \\
        =&\lambda \|v_{\psi^i,\tpsi^i}(s_h^i,a_h^i,\mu^\pi_{M^*,h})\|^2 + \sum_{t=1}^{i-1}\Big(\phi_\act(s_h^t,a_h^t,\mu^\pi_{M^*,h})\trans v_{\psi^i,\tpsi^i}(s_h^i,a_h^i,\mu^\pi_{M^*,h})\Big)^2\\
        =&\lambda \|v_{\psi^i,\tpsi^i}(s_h^i,a_h^i,\mu^\pi_{M^*,h})\|^2 + \sum_{t=1}^{i-1}\Big(\phi_\act(s_h^t,a_h^t,\mu^\pi_{M^*,h})\trans U\trans\sum_{s'} (\psi^i(s)-\tpsi^i(s'))g_{\psi^i,\tpsi^i}(s_h^i,a_h^i,\mu^\pi_{M^*,h},s')\Big)^2\\
        \leq & 4\lambda C_\Psi^2 + \sum_{t=1}^{i-1} \|\mP_{\psi^i}(\cdot|s_h^t,a_h^t,\mu^\pi_{M^*,h}) - \mP_{\tpsi^i}(\cdot|s_h^t,a_h^t,\mu^\pi_{M^*,h})\|_1^2\\
        \leq & 4\lambda C_\Psi^2 + \epsilon^2.
    \end{align*}
    By choosing $\lambda = \epsilon^2/4C^2_\Psi$, we have:
    \begin{align*}
        \|\phi_\act(s_h^i,a_h^i,\mu^\pi_{M^*,h})\|_{(\Lambda_h^i)^{-1}}^2 \geq \frac{\epsilon^2}{4\lambda C^2_\Psi + \epsilon^2} = \frac{1}{2}.
    \end{align*}
    On the one hand,
    \begin{align*}
        \det \Lambda_h^{n+1} =& \det (\Lambda_h^n + \phi_\act(s_h^n,a_h^n,\mu^\pi_{M^*,h})\phi_\act(s_h^n,a_h^n,\mu^\pi_{M^*,h})\trans ) = (1 + \|\phi_\act(s_h^n,a_h^n,\mu^\pi_{M^*,h})\|_{(\Lambda_h^n)^{-1}}^2) \cdot \det \Lambda_h^n\\
        \geq & \frac{3}{2}\det \Lambda_h^n\geq (\frac{3}{2})^n \det \Lambda_h^1 = \lambda_h^{d_\act} (\frac{3}{2})^n.
    \end{align*}
    Therefore,
    \begin{align*}
        \lambda_h^{d_\act} (\frac{3}{2})^n \leq \det \Lambda_h^{n+1} \leq (\frac{\Tr(\Lambda_h^n)}{d_\act})^{d_\act} \leq (\lambda + \frac{n C^2_\phi}{d_\act})^{d_\act}.
    \end{align*}
    which implies $n = O(d_\act\log(1+\frac{d_\act C_\phi C_\Psi}{\epsilon}))$.

    Finally, if we take the maximum over all policy $\pi$, we have 
    $$
    \dimPEII(\cM_\Psi,\epsilon) = \tilde{O}(\max_{\pi,h}\Rank([\phi_h(s_h,a_h,\mu^\pi_{M^*,h})]_{s_h\in\cS,a_h\in\cA})).
    $$
    When $\phi(s_h,a_h,\mu)$ can be decomposed to $\phi(s_h,a_h)\trans G(\mu)$ for some $\phi(s_h,a_h) \in \mR^{\tilde d}$, easy to verify that for any $\pi$, the corresponding $d_{\act} \leq \tilde{d}$.
    By combining with Prop.~\ref{prop:PMBED_vs_MBED}, we can conclude $\dimPEII(\cM,\epsilon) = \tilde{O}(\min\{d,\tilde{d}\})$.

    \paragraph{Proofs for P-MBED (Def.~\ref{def:Partial_Eluder_Dim})}
    As for the first type of P-MBED, we only study the decomposable feature setting.
    Given a fixed policy $\pi$, $h\in[H]$, suppose $(s_h^1,a_h^1),...,(s_h^n,a_h^n)$ is a partially $\epsilon$-independent sequence w.r.t. $\cM_\Psi$ and the mapping $\nu^\pi_h$ defined in Def.~\ref{def:Partial_Eluder_Dim}, then for each $i\in[n]$, there should exists $\psi^i,\tpsi^i \in \Psi$, such that:
    \begin{align*}
        \epsilon^2 \geq \sum_{t=1}^{i-1} \|\mP_{\psi^i}(\cdot|s_h^t,a_h^t,\mu^\pi_{\psi^i,h}), \mP_{\tpsi^i}(\cdot|s_h^t,a_h^t,\mu^\pi_{\tpsi^i,h})\|_1^2.
    \end{align*}
    and
    \begin{align*}
        \epsilon^2 \leq & \|\mP_{\psi^i}(\cdot|s_h^i,a_h^i,\mu^\pi_{\psi^i,h}) - \mP_{\tpsi^i}(\cdot|s_h^i,a_h^i,\mu^\pi_{\tpsi^i,h})\|_1^2\\
        =&\Big(\phi(s_h^i,a_h^i)\trans \sum_{s'\in\cS}(G(\mu^\pi_{\psi^i,h}){\psi^i}(s')-G(\mu^\pi_{\tpsi^i,h})\tpsi^i(s'))\tilde{g}_{{\psi^i},\tpsi^i}(s_h^i,a_h^i,s')\Big)^2\\
        \leq & \|\phi(s_h^i,a_h^i)\|_{(\Lambda_h^i)^{-1}}^2\|\sum_{s'\in\cS}(G(\mu^\pi_{\psi^i,h}){\psi^i}(s')-G(\mu^\pi_{\tpsi^i,h})\tpsi^i(s'))\tilde{g}_{{\psi^i},\tpsi^i}(s_h^i,a_h^i,s')\|_{\Lambda_h^i}^2.
    \end{align*}
    where we define:
    \begin{align*}
        \Lambda_h^i := & \lambda I + \sum_{t=1}^{i-1} \phi(s_h^i,a_h^i)\phi(s_h^i,a_h^i)\trans \in \mR^{\tilde{d}\times \tilde{d}}; \\
        g_{\psi^i,\tpsi^i}(s_h,a_h,s'):=&\begin{cases}
            1,\quad \text{if~} \phi(s_h^i,a_h^i)\trans (G(\mu^\pi_{\psi^i,h}){\psi^i}(s')-G(\mu^\pi_{\tpsi^i,h})\tpsi^i(s')) \geq 0;\\
            -1,\quad \text{otherwise}.
        \end{cases}
    \end{align*}
    The rest analysis is similiar to the non-decomposable setting above. As a result, we can show:
    \begin{align*}
        n = O(\tilde{d}\log(1 + \frac{\tilde{d} C_\phi C_\Psi}{\epsilon})).
    \end{align*}
    This holds for any $\pi$, which finishes the proof.
\end{proof}
\begin{remark}
    Following similar analyses as Prop.~\ref{prop:MBED_Linear_MFMDP_formal} and Prop.~B.6 and Prop.~B.7 in \citep{huang2023statistical}, we can compute the P-MBED for kernel MF-MDP and generalized linear function classes. All we need to do is to replace $d_{eff}$ or $d$ in \citep{huang2023statistical} with the corresponding dimensions conditioning on the adversarial densities.
\end{remark}

\begin{restatable}{lemma}{LemEDtoReg}\label{lem:Partial_Eluder_Bound}
    Under Def.~\ref{def:Partial_Eluder_Dim} and Def.~\ref{def:Partial_Eluder_Dim_II}, consider a fixed $\pi$ and an arbitrary $h\in[H]$, Suppose we have a sequence $\{\mP_{M^k,h}\}_{k=1}^K \in \cF$ and $\{(s^k_h,a^k_h)\}_{k=1}^K \subset \cS\times\cA$,
    \begin{itemize}
        \item if for all $k\in[K]$, $\sum_{i=1}^{k-1} \|\mP_{M^k,h}(\cdot|s^i_h,a^i_h,\mu^\pi_{M^k,h})-\mP_{M^*,h}(\cdot|s^i_h,a^i_h,\mu^\pi_{M^*,h})\|_1^2 \leq \beta$, then for any $\epsilon > 0$, we have $\sum_{k=1}^K \|\mP_{M^k,h}(\cdot|s^k_h,a^k_h,\mu^\pi_{M^k,h})-\mP_{M^*,h}(\cdot|s^k_h,a^k_h,\mu^\pi_{M^*,h})\|_1 = O(\sqrt{\beta K \dimPEI(\cM, \epsilon)} + K \epsilon)$,
        \item if for all $k\in[K]$, $\sum_{i=1}^{k-1} \|\mP_{M^k,h}(\cdot|s^i_h,a^i_h,\mu^\pi_{M^*,h})-\mP_{M^*,h}(\cdot|s^i_h,a^i_h,\mu^\pi_{M^*,h})\|_1^2 \leq \beta$, then for any $\epsilon > 0$, we have $\sum_{k=1}^K \|\mP_{M^k,h}(\cdot|s^k_h,a^k_h,\mu^\pi_{M^*,h})-\mP_{M^*,h}(\cdot|s^k_h,a^k_h,\mu^\pi_{M^*,h})\|_1 = O(\sqrt{\beta K \dimPEII(\cM, \epsilon)} + K \epsilon)$.
    \end{itemize}
\end{restatable}
\begin{proof}
    Let's consider a single-agent model-class $\cP\subset \{P|P:\cS\times\cA\rightarrow\Delta(\cS)\}$. 
    We first introduce the notion of independent state-action sequences given a single-agent model class.
    \begin{definition}[$\epsilon$-Independent state action pairs]
        Given single-agent model $\cP$ and a data sequence $\{(s^i_h,a^i_h)\}_{i=1}^n \subset \cS_h\times\cA_h$, we say $(s_h,a_h)$ is $\epsilon$-independent of $\{(s^i_h,a^i_h)\}_{i=1}^n$ w.r.t. $\cP$ if there exists $\mP,\tilde{\mP}\in\cM$ such that $\sum_{i=1}^{n} \|\mP_{h}(\cdot|s^i_h,a^i_h)- \tilde{\mP}_h(\cdot|s^i_h,a^i_h)\|_1^2 \leq \epsilon^2$ but $\|\mP_{M,h}(\cdot|s_h,a_h)- \tilde{\mP}_{h}(\cdot|s_h,a_h)\|_1 > \epsilon$.
        We call $\{(s^i_h,a^i_h)\}_{i=1}^n$ an $\epsilon$-independent sequence w.r.t. $\cP$ (at step $h$) if for any $i\in[n]$, $(s^i_h,a^i_h,\mu^i_h)$ is $\epsilon$-independent w.r.t. $\{(s^t_h,a^t_h,\mu^t_h)\}_{t=1}^{i-1}$.
    \end{definition}
    We use $\dimE(\cP,\epsilon)$ to denote the maximal length of $\epsilon$-independent sequence $\{(s^i,a^i)\}_{i\in[n]}$ for single-agent model class $\cP$.
    Since single-agent RL is a special case of MF-MDP where the transition is independent w.r.t. density. As implied by Lem. 4.4 in \citep{huang2023statistical} when $\alpha = 1$, suppose there is a sequence $\{\mP^k\}_{k\in[K]} \subset \cP$ and a sequence of states and actions $\{(s^k_h,a^k_h)\}_{k\in K}$, such that:
    \begin{align*}
        \sum_{i=1}^{k-1} \|\mP^k(\cdot|s^i_h,a^i_h) - \mP^*(\cdot|s^i_h,a^i_h)\|_1^2 \leq \beta
    \end{align*}
    where $\mP^* \in \cP$ is some fixed function, then for any $\epsilon > 0$,
    \begin{align*}
        \sum_{k=1}^{K} \|\mP^k(\cdot|s^k_h,a^k_h) - \mP^*(\cdot|s^k_h,a^k_h)\|_1 \leq O(\sqrt{\beta K \dimE(\cP,\epsilon)} + K \epsilon).
    \end{align*}
    By choosing $\cP:=\{\mP_{M,h}|\mP_{M,h}(\cdot|\cdot,\cdot)\gets \mP_{M,h}(\cdot|\cdot,\cdot,\mu^\pi_{M,h}), M\in\cM\}$ with $\mP^* := \mP_{M^*,h}(\cdot|\cdot,\cdot,\mu^\pi_{M^*,h})$ and combining the definition in Def.~\ref{def:Partial_Eluder_Dim}, we can finish the proof of the first statement.

    By choosing $\cP:=\{\mP_{M,h}|\mP_{M,h}(\cdot|\cdot,\cdot)\gets \mP_{M,h}(\cdot|\cdot,\cdot,\mu^\pi_{M^*,h}), M\in\cM\}$ with $\mP^* := \mP_{M^*,h}(\cdot|\cdot,\cdot,\mu^\pi_{M^*,h})$ and combining the definition in Def.~\ref{def:Partial_Eluder_Dim_II}, we can finish the proof of the second statement.
\end{proof}

\subsection{Partial MBED for Model Classes in Multi-Type MFGs Setting}\label{appx:P_MBED_MT}
\begin{definition}[Partial $\epsilon$-Independence for Multi-Type Mean-Field Model Classes]\label{def:partial_eps_indp_MT}
    Given a multi-type model class $\veccM$, consider a $w\in[W]$ and a mapping $\nu^w_h:\veccM\rightarrow\Delta(\cS^1_h)\times...\times\Delta(\cS^W_h)$, and a sequence of data $\{(s^{w,i}_h,a^{w,i}_h)\}_{i=1}^n\subset\cS^w_h\times\cA^w_h$, we say $(s^w_h,a^w_h)$ is partially $\epsilon$-independent on $\{(s^{w,i}_h,a^{w,i}_h)\}_{i=1}^n$ w.r.t. $\veccM$ and $\nu^w_h$ at step $h$, if there exists $\vecM,\tvecM\in\veccM$, s.t. $\sum_{i=1}^{n} \|\mP_{\vecM,h}^w(\cdot|s^{w,i}_h,a^{w,i}_h,\nu^w_h(\vecM))- \mP_{\tvecM,h}^w(\cdot|s^{w,i}_h,a^{w,i}_h,\nu^w_h(\tvecM))\|_1^2 \leq \epsilon^2$ but $\|\mP_{\vecM,h}^w(\cdot|s^{w}_h,a^{w}_h,\nu^w_h(\vecM))- \mP_{\tvecM,h}^w(\cdot|s^{w}_h,a^{w}_h,\nu^w_h(\tvecM))\|_1 > \epsilon$.
\end{definition}
Besides, we call $\{(s^{w,i}_h,a^{w,i}_h)\}_{i=1}^n$ is a partially $\epsilon$-independent sequence w.r.t. $\veccM$ and $\nu^w_h$ if for any $i\in[n]$, $(s^{w,i}_h,a^{w,i}_h)$ is partially $\epsilon$-independent on $\{(s^{w,t}_h,a^{w,t}_h)\}_{t=1}^{i-1}$.
In the following, we use $\dim_{{\rm E}|\nu^w_h}(\veccM,\epsilon)$ to denote the length of the longest partially $\epsilon$-independent sequence w.r.t. $\veccM$ and $\nu^w_h$ for type $w$ (at step $h$).

\begin{definition}\label{def:P_MBED_MT}
    Given a model class $\veccM$ and an arbitrary $w$, we define the mapping $\nu^{w,\vecpi}_h:\veccM\rightarrow\Delta(\cS^1)\times...\Delta(\cS^W)$ s.t. $\nu^{w,\vecpi}_h(\vecM) = \vecmu^\vecpi_{\vecM,h}$, and the mapping $\nu^{w,\vecpi}_{\vecM^*,h}:\veccM\rightarrow\Delta(\cS^1)\times...\Delta(\cS^W)$ s.t. $\nu^{w,\vecpi}_{\vecM^*,h}(\vecM) = \vecmu^\vecpi_{\vecM^*,h}$. Then, the Multi-Type P-MBEDs are defined by:
    \begin{itemize}[leftmargin=*]        
        \item Type $\RI$ MT-P-MBED: $\dimMTPE(\veccM,\epsilon) := \sum_{w\in[W]} \max_{h\in[H]} \max_{\vecpi\in\vecPi}  \dim_{{\rm E}|\nu^{w,\vecpi}_h}(\veccM,\epsilon)$;
        \item Type $\RII$ MT-P-MBED: $\dimMTPE^\RII(\veccM,\epsilon) := \sum_{w\in[W]} \max_{h\in[H]} \max_{\vecpi\in\vecPi}   \dim_{{\rm E}|\nu_{\vecM^*,h}^{w,\vecpi}}(\veccM,\epsilon)$.
    \end{itemize}
\end{definition}

\begin{proposition}[Tabular Multi-Type MF-MDP]\label{prop:Tabular_MF_MDP_MT}
    $$
    \max\{\dimMTPE(\veccM,\epsilon), \dimMTPE^\RII(\veccM,\epsilon)\}\leq \sum_{w\in[W]}|\cS^w||\cA^w|.
    $$
\end{proposition}
\begin{proof}
    By definition, for each $w$, for any fixed $h$ and $\vecpi$, the longest partially independent state-action sequence would have length $|\cS^w||\cA^w|$. 
\end{proof}
\begin{proposition}[Linear Multi-Type MF-MDP]\label{prop:Linear_MF_MDP}
    Consider the Low-Rank Multi-Type MF-MDP with known feature $\phi^w:\cS^w\times\cA^w\times\Delta(\cS^1)\times...\Delta(\cS^W)\rightarrow \mR^{d^w}$ satisfying $\|\phi^w\| \leq C_\phi$ for any $w\in[W]$, and unknown next state feature $\psi^w:\cS^w\rightarrow\mR^d$. Given a next state feature function class $\Psi^1,...,\Psi^W$ satisfying $\forall \psi^w\in\Psi^w,~\forall s'^w\in\cS^w,~\forall g:\cS\rightarrow \{-1,1\}$, $\|\sum_{s'^w}\psi(s'^w)g(s'^w)\|_2\leq C_\Psi$, define the model class:
    \begin{align*}
        \cM_{\Psi^w} := \{\mP_{\psi^w}|\forall s^w,a^w,\vecmu,~\mP_\psi^w(\cdot|s^w,a^w,\vecmu):=\phi^w(s^w,a^w,\vecmu)\trans\psi^w(s'^w);~\mP^w_\psi(\cdot|s^w,a^w,\vecmu)\in\Delta(\cS^w);\psi^w\in \Psi^w\},
    \end{align*}
    then, we have:
    \begin{align*}
        \forall w\in[W],\quad \dimMTPE^\RII(\cM_{\Psi^w},\epsilon) = O(\sum_{w\in[W]} \max_{\vecpi\in\vecPi, h\in[H]} d^w_{\act,\vecpi,h}\log(1+\frac{d^w_{\act,\vecpi,h} C_\phi C_\Psi}{\epsilon}))
    \end{align*}
    where $d^w_{\act,\vecpi,h} := \Rank([\phi^w_h(s_h^w,a_h^w,\vecmu^\vecpi_{\vecM^*,h})]_{s_h^w\in\cS^w,a_h^w\in\cA^w}).$

    Moreover, if $\phi^w(s^w,a^w,\vecmu)$ is decomposable, i.e. for any $w\in[W]$, $\phi^w(s^w,a^w,\vecmu)\trans=\phi^w(s^w,a^w)\trans G^w(\vecmu)$ with $\phi^w(\cdot,\cdot)\in\mR^{\tilde{d}^w}$ and $G^w(\cdot)\in\mR^{\tilde{d}^w\times d^w}$, 
    we have $\dimMTPE(\cM_{\Psi^w},\epsilon) = \tilde{O}(\sum_{w\in[W]} \tilde d^w)$ and $\dimMTPE^\RII(\cM_{\Psi^w},\epsilon) = \tilde{O}(\sum_{w\in[W]} \min\{\tilde{d}^w, d^w\})$.
\end{proposition}
\begin{proof}
    The proof is a direct generalization of Prop.~\ref{prop:MBED_Linear_MFMDP_formal} by applying the same techniques in the proof of Prop.~\ref{prop:MBED_Linear_MFMDP_formal} for each type $w\in[W]$, so we omit it here.
\end{proof}

\subsection{Partial MBED in Constrained Policy Spaces}\label{appx:CP_MBED}
Next, we define the Constrained Partial MBED extended from Def.~\ref{def:Partial_Eluder_Dim}, where the main difference is that we constrain the set of adversarial policies.
\begin{definition}[Constrained Partial MBED in MFRL]\label{def:Constr_Partial_Eluder_Dim}
    Given a (single-type) Mean-Field model class $\cM$, and $M^*$ denotes the true model, we consider the same $\nu^\pi_h$ and $\nu_{M^*,h}^\pi$ function defined in Def.~\ref{def:Partial_Eluder_Dim} and Def.~\ref{def:Partial_Eluder_Dim_II}, respectively. Then, the constrained P-MBEDs are defined by:
    \begin{itemize}[leftmargin=*]        
        \item Type $\RI$ P-MBED: $\dimCPE^\RI(\cM,\epsilon) := \max_{h\in[H]} \max_{\pi\in\Pi^\dagger} \dim_{{\rm E}|\nu^\pi_h}(\cM,\epsilon)$;
        \item Type $\RII$ P-MBED: $\dimCPE^\RII(\cM,\epsilon) := \max_{h\in[H]} \max_{\pi\in\Pi^\dagger} \dim_{{\rm E}|\nu_{M^*,h}^\pi}(\cM,\epsilon)$.
    \end{itemize}
\end{definition}
Comparing with P-MBED, the main difference is that in constrained P-MBED the adversarial policies are only chosen from the constrained policy set.
Recall the definition of $\Pi^\dagger$ in Sec.~\ref{sec:MT_MFG}.
Given a $\veccM$ and a model class $\cM_\MFG$ converted from $\veccM$ according to Appx.~\ref{appx:conversion}, we have the following relationship between the P-MBED of $\cM_\MFG$ constrained on $\Pi^\dagger$ and P-MBED of $\veccM$.
\begin{proposition}\label{prop:conversion_CP2P}
Given a model class $\veccM$ and its corresponding lifted MFG class $\cM_\MFG$:
\begin{align}
    \dimCPEI(\cM_\MFG,\epsilon) \leq \dimMTPE(\veccM,\epsilon), \label{eq:Cstr_MT_PMBED_1}\\
    \dimCPEII(\cM_\MFG,\epsilon) \leq \dimMTPE^\RII(\veccM,\epsilon). \label{eq:Cstr_MT_PMBED_2}
\end{align}
\end{proposition}
\begin{proof}
    Let's consider a fixed policy $\pi \in \Pi^\dagger$. Note that, $\pi$ corresponds to a $\vecpi:=\{\pi^w\}_{w\in[W]}$ with $\pi^w:\cS^w\rightarrow\Delta(\cA^w)$ and $\pi^w(a^w_h|s^w_h) = \pi(a^w_h\circ w|s^w_h\circ w)$.
    Given any $\epsilon > 0$, and $h\in[H]$, suppose we have a partial $\epsilon$-independent sequence w.r.t. the mapping $\nu_h^\pi$ (or $\nu_{M^*,h}^\pi$),
    denoted as $\{(s_h^{w_i,i}\circ w_i, a_h^{w_i,i}\circ w_i)\}_{i\in[n]}$.
    We divide this sequence according to its group $w_i$, which we denote as $\{\{s_h^{w,i_w}\circ w, a_h^{w,i_w}\circ w\}_{i_w\in[n_w]}\}_{w\in[W]}$ with $\sum_{w}n_w = n$.
    By construction of $\cM_\MFG$, for any $w\in[W]$, $\{s_h^{w,i_w}, a_h^{w,i_w}\}_{i_w\in[n_w]}$ is a partial $\epsilon$-independent sequence w.r.t. function class $\cM^w$ and the mapping $\nu^{w,\vecpi}_h$ (or $\nu^{w,\vecpi}_{\vecM^*,h}$), which is upper bounded by the Multi-Type P-MBED of model class $\cM^w$.
    
    We finish the proof of Eq.~\eqref{eq:Cstr_MT_PMBED_1} by maximizing over $\pi\in\Pi^\dagger$.
\end{proof}
As directly implied by Prop.~\ref{prop:conversion_CP2P}, Prop.~\ref{prop:Tabular_MF_MDP_MT} and Prop.~\ref{prop:Linear_MF_MDP}, we can upper bound the constrained P-MBED in some special cases.

\newpage
\section{Details about Single-Type/Multi-Type Policy Aware Models}\label{appx:proofs_PAM}
\subsection{(Single-Type) Policy-Aware Model}\label{appx:details_PAM}
Concretely, Policy-Aware Model (PAM) is specified by a tuple $\dM:=\{\cS,\cA,H,\mu_1,\dmP_\dM,\dr_\dM\}$, where $\cS,\cA,H,\mu_1$ are the state space, action space, horizon length, initial state distribution which are the same as the normal MF-MDP setting; $\dmP_\dM := \{\dmP_{\dM,1},...,\dmP_{\dM,H}\}$ is the transition function with $\dmP_{\dM,h} : \cS_h\times\cA_h\times\Pi \rightarrow \Delta(\cS_{h+1})$, and $\dr_\dM:=\{\dr_{\dM,1},...,\dr_{\dM,H}\}$ is the reward function\footnote{Here we specify the model in the subscription, because for those PAM converted from MF-MDPs, even if they share the reward function in mean-field systems, the reward functions in PAM version can be different because of the difference in transition functions.} satisfying $\dr_{\dM,h}:\cS_h\times\cA_h\times\Pi \rightarrow[0,1/H]$, where recall $\Pi$ denotes the collection of all Markov policies.
Given any reference policy $\pi$, we define the value function $\dQ^{\tpi}_{\dM,h}:\cS_h\times\cA_h\times\Pi\rightarrow\mR$ and $\dV^{\tpi}_{\dM,h}:\cS_h\times\Pi\rightarrow\mR$ regarding $\tpi$ in the following way:
\begin{align*}
    \dQ^\tpi_{\dM,h}(\cdot,\cdot,\pi) :=& \EE\Big[\sum_{\ph=h}^H \dr_{\dM,\ph}(s_\ph,a_\ph,\pi)\Big|s_h=\cdot,~a_h=\cdot,~\forall \ph \geq h,~s_{\ph+1}\sim\dmP_{\dM,\ph}(\cdot|s_{\ph},a_\ph,\pi),a_{\ph+1} \sim \tpi(\cdot|s_{\ph+1})\Big],\\
    \dV^\tpi_{\dM,h}(\cdot,\pi) :=& \EE\Big[\sum_{\ph=h}^H \dr_{\dM,\ph}(s_\ph,a_\ph,\pi)\Big|s_h=\cdot,~a_h\sim\tpi,~\forall \ph \geq h,~s_{\ph+1}\sim\dmP_{\dM,\ph}(\cdot|s_{\ph},a_\ph,\pi),a_{\ph+1} \sim \tpi(\cdot|s_{\ph+1})\Big].
\end{align*}
Similarly, we will denote $\EE_{\tpi,\dM(\pi)}[\cdot]$ to be the expectation taken over trajectories sampled by executing $\tpi$ in the model $\dM$, such that the transition and reward functions are fixed by $\pi$. Again, we will call $\pi$ as the ``reference policy''.

By definition, once the reference policy $\pi$ is determined, the transition/reward functions reduced to single-agent transition/reward functions, and the value functions are defined in the same way as single-agent RL setting.
Besides, we define the total return of $\tpi$ conditioning on the reference policy $\pi$ as:
\begin{align*}
    \dJ_\dM(\tpi;\pi) := \EE_{s_1\sim\mu_1}[\dV^\tpi_{\dM,h}(s_1,\pi)],
\end{align*}
and define $\Delta_{\dM}(\tpi,\pi) := \dJ_{\dM}(\tpi,\pi) - \dJ_{\dM}(\pi,\pi)$.
Similar to MF-MDP, we define the NE in $\dM$.
Intuitively, the NE in $\dM$ is the policy $\pi_\dM^\NE$ that agents do not tend to deviate when $\pi_\dM^\NE$ is chosen to be the reference policy.
\begin{definition}[Nash Equilibrium in $\dM$]\label{def:NE_in_PAM}
    Given a model $\dM$, we call $\pi_\dM^\NE$ is a Nash Equilibrium (NE) of $\dM$, if 
    $$
    \forall \tpi \in \Pi, \quad \dJ_\dM(\tpi;\pi^{\NE}_\dM) \leq \dJ_\dM(\pi^{\NE}_\dM;\pi^{\NE}_\dM).
    $$
    Besides, we call $\hpi_\dM^\NE$ is an $\epsilon$-approximate NE of $\dM$, if 
    $$
    \forall \tpi \in \Pi, \quad \dJ_\dM(\tpi;\pi^{\NE}_\dM) \leq \dJ_\dM(\pi^{\NE}_\dM;\pi^{\NE}_\dM) + \epsilon.
    $$
\end{definition}
Similar to the conditional distance $d(M,\tM|\pi)$ defined in Sec.~\ref{sec:learning_MFG}, we can define the conditional distance for PAM.
\begin{align*}
    d(\dM,\dtM|\pi) := \max_\tpi~\max\{\EE_{\tpi,\dM(\pi)} & [\sum_{h=1}^H \|\dmP_{\dM,h}(\cdot|\cdot,\cdot,\pi) - \dmP_{\dtM,h}(\cdot|\cdot,\cdot,\pi)\|_1],\\
    & \EE_{\tpi,\dtM(\pi)}[\sum_{h=1}^H \|\dmP_{\dM,h}(\cdot|\cdot,\cdot,\pi) - \dmP_{\dtM,h}(\cdot|\cdot,\cdot,\pi)\|_1]\}
\end{align*}
Given a PAM model class $\dcM$ and a model $\dM \in \dcM$, for any reference policy $\pi$, we define the $\epsilon_0$-neighborhood of $\dM$ in $\dcM$ to be 
$$
\cB_{\pi}^{\epsilon_0}(\dM;\dcM) := \{\dM'\in\dcM | d(\dM,\dM'|\pi)\leq \epsilon_0\}.
$$
Besides, we define the Central Model $\dM_{\Central}^{\epsilon_0}(\pi;\dcM)$ in $\dcM$ regarding $\pi$ to be the model with the largest neighborhood set:
\begin{align*}
    \dM_{\Central}^{\epsilon_0}(\pi;\dcM) \gets \arg\max_{\dM\in\dcM} |\cB_{\pi}^{\epsilon_0}(\dM;\dcM)|.
\end{align*}
\subsubsection{Existence of Nash Equilibrium in $\dM$}
Next, we investigate the existence of NE in $\dM$. Recall the definition 
\begin{equation}
d_{\infty,1}(\pi,\pi') := \max_{h\in[H],s_h\in\cS_h}\|\pi_h(\cdot|s_h)-\pi'_h(\cdot|s_h)\|_1.
\end{equation}
\begin{restatable}{theorem}{ThmPAMNE}[Existence of Nash Equilibrium in PAM]\label{thm:exist_NE_PAM}
    Given a PAM $\dM$ with discrete state and action spaces, such that, for any $h\in[H],s_{h+1}\in\cS_{h+1},s_h\in\cS_h,a_h\in\cA_h$, both $\dmP_{\dM,h}(s_{h+1}|s_h,a_h,\pi)$ and $\dr_{\dM,h}(s_h,a_h,\pi)$ are continuous at $\pi$ w.r.t. distance $d_{\infty,1}$, then $\dM$ has at least one NE satisfying Def.~\ref{def:NE_in_PAM}.
\end{restatable}
\begin{proof}
    In Prop.~\ref{thm:exist_NE_PAM}, we establish the existence of NE in Multi-Type PAM, and the proof for this theorem is a special case when $W = 1$.
\end{proof}
As a direct result of Thm.~\ref{thm:exist_NE_PAM} and Lem.~\ref{lem:Lipschitz_PA_model}, we have the following corollary.
\begin{corollary}
    Given a MF-MDP model $M$ satisfying Assump.~\ref{assump:Lipschitz}, the PAM model $\dM$ converted from $M$ according to the rules in Eq.~\eqref{eq:conversion_to_dM} has at least one NE.
\end{corollary}

\subsubsection{Useful Lemma Related to the PAM converted from MF-MDP}

\begin{restatable}{lemma}{LemPACLip}[Lipschitz Continuity of PAM]\label{lem:Lipschitz_PA_model}
    Given an MF-MDP $M$ satisfying the Lipschitz continuity condition in Assump.~\ref{assump:Lipschitz}, consider the PAM $\dM$ converted from $M$ according to Eq.~\eqref{eq:conversion_to_dM}, we have $\dM$ is also Lipschitz continuous that, $\forall h\in[H]$ and any $s_h\in\cS,a_h\in\cA$,
    \begin{align*}
        \|\dmP_{\dM,h}(\cdot|s_h,a_h,\pi) - \dmP_{\dM,h}(\cdot|s_h,a_h,\pi')\|_1 \leq d_{\infty,1}(\pi,\pi') L_T \sum_{\ph=1}^h (1+L_T)^{h-\ph}\\
        |\dr_{\dM,h}(s_h,a_h,\pi) - \dr_{\dM,h}(s_h,a_h,\pi')| \leq d_{\infty,1}(\pi,\pi') L_r \sum_{\ph=1}^h (1+L_T)^{h-\ph}.
    \end{align*}
\end{restatable}
\begin{proof}
    This lemma is a special case of Lem.~\ref{lem:Lipschitz_PAM_MT} when $W = 1$.
\end{proof}

\subsection{Multi-Type Policy-Aware Model}
In this section, we introduce Multi-Type Policy-Aware Model (MT-PAM) extended from PAM.
To distinguish with MT-MFG, we use $\dvecM$ as notation.

MT-PAM is specified by $\dvecM:=\{(\mu_1^w,\cS^w,\cA^w,H,\dmP^w_\dvecM,\dr^w_{\dvecM})_{w\in[W]}\}$\footnote{Here we specify the model in the subscription of the reward function, which will avoid confusion when we consider the PAMs converted from (Multi-Type) MF-MDPs}, where $\cS^w,\cA^w,H,\mu_1^w$ are defined the same as the Multi-Type MF-MDP setting; $\dmP^w_\dvecM := \{\dmP^w_{\dvecM,h}\}_{h\in[H]}$ is the transition function with $\dmP^w_{\dvecM,h} : \cS_h^w\times\cA_h^w\times\vecPi \rightarrow \Delta(\cS_{h+1})$ and $\dr_{\dvecM,h}: \cS^w_h\times\cA^w_h\times\vecPi \rightarrow [0,\frac{1}{H}]$, where recall $\vecPi$ denotes the set of all Markov policies $\vecpi:=\{\pi^w\}_{w\in[W]}$ with $\pi^w\in\Pi^w$.

Given a reference policy $\vecpi:=\{\pi^w\}_{w\in[W]}\in\vecPi$, for any $\tvecpi:=\{\tpi^w\}_{w\in[W]}\in\vecPi$, we define the value function for type $w$ $\dQ^{w,(\cdot)}_{\dvecM,h}:\cS^w_h\times\cA^w_h\times\vecPi\rightarrow\mR$ and $\dV^{w,(\cdot)}_{\dvecM,h}:\cS_h^w\times\vecPi\rightarrow\mR$ in the following way:
\begin{align*}
    \dQ^{w,\tvecpi}_{\dvecM,h}(\cdot,\cdot;\vecpi) := \EE\Big[\sum_{\ph=h}^H \dr_{\dvecM,\ph}^w(s_\ph^w,a_\ph^w,\vecpi)\Big|s_h^w=\cdot,~a_h^w=\cdot,~\forall \tldh \geq h,~~{s_{\tldh+1}^w\sim \dmP^w_{\dvecM,\tldh}(\cdot|s_\tldh^w,a_\tldh^w,\vecpi),~a_{\tldh+1}^w \sim \tpi_{\tldh+1}^w}\Big].\\
    V^{w,\tvecpi}_{\dvecM,h}(\cdot;\vecpi) := \EE\Big[\sum_{\ph=h}^H \dr_{\dvecM,\ph}^w(s_\ph^w,a_\ph^w,\vecpi)\Big|s_h^w=\cdot,~{a_{h}^w \sim \tpi_{h}^w,~~\forall \tldh \geq h,~~s_{\tldh+1}^w\sim \dmP^w_{\dvecM,\tldh}(\cdot|s_\tldh^w,a_\tldh^w,\vecpi),~a_{\tldh+1}^w \sim \tpi_{\tldh+1}^w}\Big].
\end{align*}
Similarly, we will denote $\EE_{\tvecpi,\dvecM(\vecpi)}[\cdot]$ to be the expectation taken over trajectories sampled by executing $\tvecpi$ in the model $\dvecM$, such that the transition and reward functions are fixed by $\vecpi$. Again, we will call $\vecpi$ as the ``reference policy''.

We denote $\dJ_\dvecM^w(\tvecpi;\vecpi) := \EE_{s_1^w\sim\mu_1^w}[V^{w,\tvecpi}_{\vecM,1}(s_1^w;\vecpi)]$ to be the expected return of agents in type $w$ in model $\dvecM$ by executing $\tvecpi$ given $\vecpi$ as the reference policy.
\begin{definition}[Nash Equilibrium in Multi-Type PAM]\label{def:NE_in_MTPAM}
    The Nash Equilibrium policy in Multi-Type PAM is defined to be the policy $\vecpi^\NE:=\{\pi^{w,\NE}\}_{w\in[W]}$ satisfying:
    \begin{align}
        \forall w \in [W],~\forall\tvecpi\in\vecPi,\quad \dJ^w_\dvecM(\tvecpi;\vecpi^\NE) \leq \dJ^w_\dvecM(\vecpi^\NE;\vecpi^\NE).\label{eq:NE_MTPAM}
    \end{align}
\end{definition}
Note that $\dJ^w_\dvecM(\tvecpi;\vecpi)$ actually only depends on $\vecpi$ and $\tpi^w$.

\subsubsection{Existence of Nash Equilibrium in MT-PAM}
We first investigate a stronger notion of NE, which we call the ``strict NE''.
\begin{definition}[Strict NE]\label{def:strict_NE}
    Given a MT-PAM $\dvecM$ with transitions and rewards $\{(\dmP^w_{\dvecM}, \dr^w_\dvecM)\}_{w\in[W]}$, the policy $\vecpi^\NE$ is a strict NE of $\dvecM$ if and only if the following holds:
    \begin{align}
        \forall w\in[W],\quad \pi_h^w(\cdot|s_h^w) \in \argmax_{u\in\Delta(\cA^w)} \dQ^{w,\vecpi}_{\dvecM,h}(s_h^w,\cdot,\vecpi)\trans u.
    \end{align}
\end{definition}
Note that this is a stronger notion than the NE defined in Def.~\ref{def:NE_in_MTPAM}, i.e. a strict NE is always a NE. In the following, we will focus on the existence of strict NE.
\begin{lemma}[Strict NE as Fixed Point]\label{lem:NE_FixedPoint_MT}
    Given a MT-PAM $\dvecM$ with transitions and rewards $\{(\dmP^w_{\dvecM}, \dr^w_\dvecM)\}_{w\in[W]}$, the policy $\vecpi^\SNE$ is a strict NE of $\dvecM$ if and only if the following holds:
    \begin{align*}
        \Gamma_{\dvecM}^{\SNE}(\vecpi^\SNE) = \vecpi^\SNE,
    \end{align*}
    where
    $$
    \Gamma_{\dvecM}^{\SNE}(\vecpi) := \{\tvecpi:=\{\tpi^w_h\}_{w\in[W],h\in[H]} |\forall w,s_h^w,~\tpi^w_h(\cdot|s_h^w) := \argmax_{u\in\Delta(\cA^w)} \dQ_{\dvecM,h}^{w,\vecpi}(s_h^w,\cdot, \vecpi)^\top u - \| \pi^w_h(\cdot|s_h^w) - u\|_2^2\}.
    $$
\end{lemma}
\begin{proof}
    First of all, suppose $\vecpi$ is the NE of $\dvecM$ according to Def.~\ref{def:NE_in_MTPAM}, by the policy improvement theorem in single-agent RL, we have:
    \begin{align*}
        \forall w\in[W],\quad \pi_h^w(\cdot|s_h^w) \in \argmax_{u\in\Delta(\cA^w)} \dQ^{w,\vecpi}_{\dvecM,h}(s_h^w,\cdot,\vecpi)\trans u,
    \end{align*}
    which also implies 
    \begin{align*}
        \forall w\in[W],~\pi_h^w(\cdot|s^w_h) \in \arg\max_{u\in\Delta(\cA^w)} \dQ^{w,\vecpi}_{\dvecM,h}(s_h^w,\cdot,\vecpi)\trans u - \|\pi^w_h(\cdot|s_h^w) - u\|_2^2.
    \end{align*}
    Therefore, if $\vecpi$ is the strict NE of $\dvecM$, we have $\Gamma_{\dvecM}^{\SNE}(\vecpi) = \vecpi$.

    On the other hand, if $\Gamma_{\dM}^{\SNE}(\vecpi) = \vecpi$, it implies:
    \begin{align*}
        \forall w\in[W],~\pi_h^w(\cdot|s^w_h) \in \arg\max_{u\in\Delta(\cA^w)} \dQ^{w,\vecpi}_{\dvecM,h}(s_h^w,\cdot,\vecpi)\trans u - \|\pi^w_h(\cdot|s_h^w) - u\|_2^2.
    \end{align*}
    By the first order optimality condition of the RHS, we should have:
    \begin{align*}
        \forall w\in[W],\quad \pi_h^w(\cdot|s_h^w) \in \argmax_{u\in\Delta(\cA^w)} \dQ^{w,\vecpi}_{\dvecM,h}(s_h^w,\cdot,\vecpi)\trans u,
    \end{align*}
    Therefore, $\vecpi$ is the strict NE of $\dvecM$.
\end{proof}
\begin{definition}[Distance measure between policies]\label{def:d1_distance_MT}
    Given two policies $\vecpi:=\{\pi^w_h\}_{h\in[H],w\in[W]}$ and $\tvecpi:=\{\tpi^w_h\}_{h\in[H],w\in[W]}$, we define:
    \begin{align*}
        d_{\infty,1}(\vecpi,\tvecpi) := \max_{w\in[W],h\in[H],s_h^w\in\cS_h^w}\|\pi_h^w(\cdot|s_h^w)-\tpi_h^w(\cdot|s_h^w)\|_1.
    \end{align*}
\end{definition}
\begin{condition}\label{cond:continuity_MT}
    Given a model function $\dvecM:=\{(\dmP^w_{\dvecM,h}, \dr^w_{\dvecM,h})\}_{h\in[H]}$, for any fixed $w\in[W],h\in[H],s^w_{h+1}\in\cS^w_{h+1},s_h^w\in\cS_h^w,a_h^w\in\cA_h^w$ and any $\vecpi$, $\dmP^w_{\dvecM,h}(s_{h+1}^w|s_h^w,a_h^w,\vecpi)$ and $\dr^w_{\dvecM,h}(s_h^w,a_h^w,\vecpi)$ are continuous at $\vecpi$ w.r.t. distance $d_{\infty,1}$.
\end{condition}
\begin{lemma}[Continuity of $\dQ$]\label{lem:continuity_ddotQ_MT}
    Under Cond.~\ref{cond:continuity_MT}, for any $w\in[W]$, any $s_h^w,a_h^w$ and $\vecpi$, $\dQ^{w,\vecpi}_{\dvecM,h}(s_h^w,a_h^w,\vecpi)$ is continuous at $\vecpi$ w.r.t. the distance $d_{\infty,1}$ in Def.~\ref{def:d1_distance_MT}.
\end{lemma}
\begin{proof}
    The proof is obvious by noting that $\dQ_{\dvecM,h}^{w,\vecpi}(s_h^w,a_h^w,\vecpi)$ is a function resulting from finite multiplication and addition among $\dmP^w_{\dvecM,h}(\cdot|\cdot,\cdot,\vecpi)$, $\dr^w_{\dvecM,h}(\cdot,\cdot,\vecpi)$ and $\vecpi$.
\end{proof}
In the following proposition, we will establish the existence of NE based on the existence of strict NE.
\begin{proposition}\label{prop:existence_MT_PAM}
    Under Cond.~\ref{cond:continuity_MT}, the MT-PAM has at least one NE policy $\vecpi^\NE$ satisfying Def.~\ref{def:NE_in_MTPAM}.
\end{proposition}
\begin{proof}
    We first show the mapping $\Gamma_{\dvecM}^{\SNE}:\vecPi \rightarrow \vecPi$ is continuous under Cond.~\ref{cond:continuity_MT}.
    Based on a similar discussion as Lem. E.6 in \citep{huang2023statistical},
    \begin{align*}
        u \rightarrow \argmax_{u'\in\Delta(\cA^w)} q^\top u' - \|u - u' \|_2^2,
    \end{align*}
    is continuous for any fixed $q \in \mR^{|\cA^w|}$, and
    \begin{align*}
        q \rightarrow \argmax_{u'\in\Delta(\cA^w)} q^\top u' - \|u - u' \|_2^2
    \end{align*}    
    is also continuous for any fixed $u\in\Delta(\cA^w)$.
    
    By Lem.~\ref{lem:continuity_ddotQ_MT}, and the rule of composition of continuous functions, $\Gamma_{\dvecM}^{\SNE}$ is a continuous mapping.
    Therefore, $\Gamma_{\dvecM}^{\SNE}$ maps from the closed and convex polytope $\vecPi$ to a subset of itself. By Brouwers fixed point theorem it has a fixed point.
    By Lem.~\ref{lem:NE_FixedPoint_MT}, such fixed point is a strict NE of $\dvecM$.

    Comparing with Def.~\ref{def:NE_in_MTPAM} and Def.~\ref{def:strict_NE}, we know NE is a super-set of strict NE, which implies the existence of NE in the MT-PAM.
\end{proof}

\subsubsection{Existence of Nash Equilibrium in MT-MFG as Corollary}
\paragraph{Conversion from Multi-Type MF-MDP to MT-PAM}
Given a Multi-Type MF-MDP $\vecM$, we can convert it to a MT-PAM sharing the same $\{\mu_1^w,\cS^w,\cA^w,H\}_{w\in[W]}$ with $\vecM$, while the transition and reward functions of $\dvecM$ are defined by:
\begin{align}
    \forall w\in[W],\forall h\in[H],\quad \dmP^w_{\dvecM,h}(\cdot|\cdot,\cdot,\vecpi) := \mP^w_{\vecM,h}(\cdot|\cdot,\cdot,\vecmu^\vecpi_{\vecM,h}),~ \dr_{\dvecM,h}(\cdot,\cdot,\vecpi) := r_h(\cdot,\cdot,\vecmu^\vecpi_{\vecM,h}),\label{eq:conversion_MT}
\end{align}
where $\vecmu^\vecpi_{\vecM,h}$ is the density of agents in all types induced by policy $\vecpi$ in model $\vecM$ starting from $\vecmu^\vecpi_{\vecM,1}=\vecmu_1$.

\begin{restatable}{proposition}{PropMTNE}[Existence of NE in MT-MFG]\label{prop:existence_MT_MFG}
    Under Assump.~\ref{assump:Lipschitz_MT}, the Multi-Type MF-MDP has at least one NE policy $\vecpi^\NE$ satisfying Eq.~\eqref{eq:NE_MT}.
\end{restatable}
\begin{proof}
    By Lem.~\ref{lem:Lipschitz_PAM_MT}, we know the MT-PAM converted from such Multi-Type MF-MDP satisfying Cond.~\ref{cond:continuity_MT}, and by Prop.~\ref{prop:existence_MT_PAM}, the MT-PAM has at least one NE.
    Easy to check that such NE is also a NE for the Multi-Type MF-MDP satisfying Eq.~\eqref{eq:NE_MT}.
\end{proof}

\subsubsection{Proofs Related to the MT-PAM converted from Multi-Type MF-MDP}
\begin{restatable}{lemma}{LemPACLip}[Lipschitz Continuity of MT-PAM]\label{lem:Lipschitz_PAM_MT}
    Given a Multi-Type MF-MDP $\vecM$ satisfying the Lipschitz continuity condition in Assump.~\ref{assump:Lipschitz_MT}, consider the MT-PAM $\dvecM$ converted from $\vecM$ according to Eq.~\eqref{eq:conversion_MT}, we have $\dvecM$ is also Lipschitz continuous that, $\forall w\in[W],h\in[H]$ and any $s_h^w\in\cS^w,a_h^w\in\cA^w$,
    \begin{align*}
        \quad &\|\dmP^w_{\dvecM,h}(\cdot|s_h^w,a_h^w,\vecpi) - \dmP^w_{\dvecM,h}(\cdot|s_h^w,a_h^w,\vecpi')\|_1 \leq d_{\infty,1}(\vecpi,\vecpi') W\vecL_T \sum_{\ph=1}^h (1+\vecL_T)^{h-\ph}\\
        &|\dr_{\dvecM,h}^w(s_h^w,a_h^w,\vecpi) - \dr_{\dvecM,h}^w(s_h^w,a_h^w,\vecpi')| \leq d_{\infty,1}(\vecpi,\vecpi') W\vecL_r \sum_{\ph=1}^h (1+\vecL_T)^{h-\ph}.
    \end{align*}
\end{restatable}
\begin{proof}
    Based on Lem.~\ref{lem:density_differences}, as a special case, when $\vecM = \vecM'$, we have:
    \begin{align*}
        \|\vecmu^{\vecpi}_{\vecM,h} - \vecmu^{\vecpi'}_{\vecM,h}\|_1 \leq & (1 + \vecL_T)\|\vecmu^{\vecpi}_{\vecM,h-1} - \vecmu^{\vecpi'}_{\vecM,h-1}\|_1 + W\cdot d_{\infty,1}(\vecpi,\vecpi') \\
        = & d_{\infty,1}(\vecpi,\vecpi') W\sum_{\ph=1}^h (1+\vecL_T)^{h-\ph}.
    \end{align*}
    Therefore, for any $w\in[W]$,
    \begin{align*}
        \|\ddot{\mP}^w_{\dvecM,h}(\cdot|s_h^w,a_h^w,\vecpi) - \ddot{\mP}^w_{\dvecM,h}(\cdot|s_h^w,a_h^w,\vecpi')\|_1 =& \|\mP^w_{\vecM,h}(\cdot|s_h^w,a_h^w,\vecmu^{\vecpi}_{\vecM,h}) - \mP^w_{\vecM,h}(\cdot|s_h^w,a_h^w,\vecmu^{\vecpi'}_{\vecM,h})\|_1 \\
        \leq & \vecL_T\|\vecmu^{\vecpi}_{\vecM,h} - \vecmu^{\vecpi'}_{\vecM,h}\|_1\leq d_{\infty,1}(\pi,\pi')W\vecL_T\sum_{\ph=1}^h (1+\vecL_T)^{h-\ph}.
    \end{align*}
    and
    \begin{align*}
        |\dr^w_{\dvecM,h}(s_h^w,a_h^w,\vecpi) - \dr^w_{\dvecM,h}(s_h^w,a_h^w,\vecpi')| =& |r^w_{h}(s_h^w,a_h^w,\mu^{w,\vecpi}_{\vecM,h}) - r^w_{h}(s_h^w,a_h^w,\mu^{w,\vecpi'}_{\vecM,h})| \\
        \leq & \vecL_r\|\vecmu^{\vecpi}_{\vecM,h} - \vecmu^{\vecpi'}_{\vecM,h}\|_1\leq d_{\infty,1}(\vecpi,\vecpi')W\vecL_r\sum_{\ph=1}^h (1+\vecL_T)^{h-\ph}.
    \end{align*}
\end{proof}
\newpage

\section{Missing Details and Proofs for Results in Sec.~\ref{sec:learning_MFG}}\label{appx:learning_with_DCP}

\subsection{Proofs for Lemma and Theorems used for Insights}
\LemLocalAlign*
\begin{proof}
    For any policy $\tpi$, we have:
    \begin{align*}
        & J_\tM(\tpi, \hpi^\NE_M) - J_\tM(\hpi^\NE_M, \hpi^\NE_M) \\
        \leq& J_\tM(\tpi, \hpi^\NE_M) - J_M(\tpi, \hpi^\NE_M) - \Big(J_\tM(\hpi^\NE_M, \hpi^\NE_M) - J_M(\hpi^\NE_M, \hpi^\NE_M)\Big) + J_M(\tpi, \hpi^\NE_M) - J_M(\hpi^\NE_M, \hpi^\NE_M) \\
        \leq & \epsilon_1 + J_\tM(\tpi, \hpi^\NE_M) - J_M(\tpi, \hpi^\NE_M) - \Big(J_\tM(\hpi^\NE_M, \hpi^\NE_M) - J_M(\hpi^\NE_M, \hpi^\NE_M)\Big) \tag{$\hpi^\NE_M$ is an $\epsilon_1$-NE of $M$} \\
        \leq & \epsilon_1 + 2d(M,\tM|\hpi^\NE_M)\\
        \leq & \epsilon_1 + 2\epsilon_2.
    \end{align*}
\end{proof}
\ThmExistBrPi*
\begin{proof}
    \textbf{Thm.~\ref{lem:exists_bridge_policy} is just a helper theorem to make it easy for the reader to understand our proofs. It will not be used in the proof of our main results Thm.~\ref{thm:sample_complexity_MFG}, so here we only show an informal proof.}

    Combining with Lem.~\ref{lem:Lipschitz_PA_model} and Thm.~\ref{thm:exist_NE_PAM}, we show the bridge model $\dM_\Bridge$ has at least one NE $\pi^\NE_\Bridge$.
    In Thm.~\ref{thm:ConstructionErr}, we provide upper bound for the distance between the central model of $\pi^\NE_\Bridge$ with $\dM_\Bridge$, which implies $\pi^\NE_\Bridge$ is an approximate NE of its central model.
\end{proof}

\subsection{Definition of $\epsilon$-cover of Policy Space}
\begin{proposition}[$\epsilon$-cover of $\Pi$]\label{prop:eps_cover_Pi}
    Consider the set 
    $$\Pi_\epsilon := \{\pi:=\{\pi_1,...,\pi_H\}|\forall h\in[H],s_h\in\cS_h,~\pi_h(\cdot|s_h)\in\cN_{\epsilon}\},
    $$
    where 
    $$
    \cN_\epsilon := \{(\frac{N_1}{N},...,\frac{N_A}{N})|N=\lceil\frac{2A}{\epsilon}\rceil;N_1,...,N_A\in\mN;\sum_{i=1}^A N_i = N\}.
    $$
    Then, $\Pi_\epsilon$ is an $\epsilon$-cover of the policy space $\Pi$ w.r.t. $d_{\infty,1}$ distance.
\end{proposition}
\begin{proof}
    For any $u\in\Delta(\cA)$, there exists a $v\in\cN_\epsilon$, such that, $\|u-v\|_1 \leq \frac{1}{N} \cdot (A-1) + \frac{A-1}{N} \leq \epsilon$, which implies $\cN_\epsilon$ is an $\epsilon$-cover of simplex $\Delta(\cA)$. By definition of $\Pi_\epsilon$, we finish the proof.
\end{proof}

\subsection{Proofs for Algorithm~\ref{alg:elimination_DCP_formal}}

\begin{theorem}[Adapted from Thm.~4.2 in \citep{huang2023statistical}]\label{thm:MLE_guarantees}
    For any $\delta \in (0,1)$, during the running of Alg.~\ref{alg:elimination_DCP_formal}, suppose $M^* \in \bcM$, then w.p. $1-\delta$, $\forall t\in[T]$, we have $M^* \in \bcM^t$. Besides, denote $\mH$ as the hellinger distance, for each $M\in\bcM^{t}$ with transition $\mP_M$ and any $h\in[H]$:
    \begin{align*}
        \sum_{i=1}^{t-1} \EE_{\tpi^i,M^*(\pi)}[\mH^2(\mP_{M,h}(\cdot|s_h^i,a_h^i,\mu^{\pi}_{M,h}),~\mP_{M^*,h}(\cdot|s_h^i,a_h^i,\mu^{\pi}_{M^*,h}))] \leq 2 \log (\frac{2|\cM|TH}{\delta})\\
        (t-1)\cdot \EE_{\pi,M^*(\pi)}[\mH^2(\mP_{M,h}(\cdot|s_h^i,a_h^i,\mu^{\pi}_{M,h}),~\mP_{M^*,h}(\cdot|s_h^i,a_h^i,\mu^{\pi}_{M^*,h}))]\leq 2 \log (\frac{2|\cM|TH}{\delta}).
    \end{align*}
\end{theorem}

\begin{restatable}{theorem}{ThmElimAlg}\label{thm:Elimination_Alg}
    Given any reference policy $\pi$, $\teps,\delta \in (0,1)$, $M^*\in\bcM$, if $T = \tilde{O}(\frac{H^4}{\teps^2}(\dimPEI(\cM,\epsilon')\wedge (1+L_T)^{2H}(1+L_T H)^2\dimPEII(\cM,\epsilon'))\log^2\frac{2|\cM|TH}{\delta})$
    with $\epsilon' = O(\frac{\teps}{H^{2}(1+L_T)^{H}})$, 
    w.p. $1-\delta$,
    Alg.~\ref{alg:elimination_DCP_formal} terminates at some $T_0 \leq T$, and return $\bcM^{T_0}$ s.t. (i) $M^* \in \bcM^{T_0}$ (ii) $\forall M \in \bcM^{T_0}$, $d(M^*,M|\pi)\leq \teps$.
\end{restatable}
\begin{proof}
    Suppose Alg.~\ref{alg:elimination_DCP_formal} proceeds to iteration $T_0 \leq T$, and does not terminate at Line~\ref{line:elimination_termination_formal}.
    On the good events in Thm.~\ref{thm:MLE_guarantees}, we have $M^* \in \bcM^t$ for all $t \leq T_0$.

    In our first step, we discuss how to provide upper bounds for accumulative model difference depending on two types of P-MBED.
    \paragraph{Step 1-(a): Upper Bound Model Difference with Type $\RII$ P-MBED}
    For any $t \leq T_0$, given the fact that $\|P - Q\|_1 \leq \sqrt{2}\mH(P,Q)$, for any fixed $h\in[H]$, we have:
    \begin{align*}
        &\sum_{i=1}^{t-1} \EE_{\tpi^i,M^*(\pi)}[\|\mP_{M^*,h}(\cdot|s_h,a_h,\mu^\pi_{M^*,h}) - \mP_{M^t,h}(\cdot|s_h,a_h,\mu^\pi_{M^*,h})\|_1^2]\\
        \leq &2\sum_{i=1}^{t-1} \EE_{\tpi^i,M^*(\pi)}[\|\mP_{M^*,h}(\cdot|s_h,a_h,\mu^\pi_{M^*,h}) - \mP_{M^t,h}(\cdot|s_h,a_h,\mu^\pi_{M^t,h})\|_1^2] + 2\sum_{i=1}^{t-1} L_T^2\|\mu^\pi_{M^*,h} - \mu^\pi_{M^t,h}\|_1^2\\
        \leq& 8(1 + L_T^2 H^2) \log\frac{2|\cM|TH}{\delta}.\numberthis\label{eq:eq_1}
    \end{align*}
    where in the last step is because, as a result of Lem.~\ref{lem:density_est_err}, Cauchy's inequality, and $\EE^2[X] \leq \EE[X^2]$, we have:
    \begin{align*}
        \sum_{i=1}^{t-1} \|\mu^\pi_{M^*,h} - \mu^\pi_{M^t,h}\|_1^2 \leq & (t-1) \cdot H\cdot \EE_{\pi,M^*}[\sum_{\ph=1}^h\|\mP_{M^*,h}(\cdot|s_h,a_h,\mu^\pi_{M^*,h}) - \mP_{M^t,h}(\cdot|s_h,a_h,\mu^\pi_{M^t,h})\|_1^2]\\
        \leq & 4 H^2 \log\frac{2|\cM|TH}{\delta}.
    \end{align*}
    By Lem.~\ref{lem:concentration}, w.p. $1-\delta/2TH$, for any $t\in[T_0]$ and any $h\in[H]$, we have:
    \begin{align*}
        &\sum_{i=1}^{t-1} \|\mP_{M^*,h}(\cdot|\ts_h^i,\ta_h^i,\mu^\pi_{M^*,h}) - \mP_{M^t,h}(\cdot|\ts_h^i,\ta_h^i,\mu^\pi_{M^*,h})\|_1^2\\
        \leq & 96 (1 + L_T^2H^2) \log\frac{2|\cM|TH}{\delta} + C\cdot \log\frac{2TH}{\delta} \leq c_1 (1 + L_T^2H^2) \log\frac{2|\cM|TH}{\delta}.
    \end{align*}
    for some constant $C$ and $c_1$. By Lem.~\ref{lem:Partial_Eluder_Bound}, we further have:
    \begin{align*}
        &\sum_{t=1}^{T_0} \|\mP_{M^*,h}(\cdot|\ts_h^t,\ta_h^t,\mu^\pi_{M^*,h}) - \mP_{M^t,h}(\cdot|\ts_h^t,\ta_h^t,\mu^\pi_{M^*,h})\|_1 \leq c_2((1 + L_T H)\sqrt{\dimPEII(\cM,\epsilon') T_0 \log\frac{2|\cM|TH}{\delta}} + T_0 \epsilon').
    \end{align*}
    for some constant $c_2$. By Lem.~\ref{lem:concentration} again, w.p. $1-\delta/2TH$, for any $T_0\in[T]$ and any $h\in[H]$,
    \begin{align*}
        &\sum_{t=1}^{T_0} \EE_{\tpi^t,M^*(\pi)}[\|\mP_{M^*,h}(\cdot|s_h,a_h,\mu^\pi_{M^*,h}) - \mP_{M^t,h}(\cdot|s_h,a_h,\mu^\pi_{M^*,h})\|_1] \\
        \leq & 3c_2((1 + L_T H)\sqrt{\dimPEII(\cM,\epsilon') T_0 \log\frac{2|\cM|TH}{\delta}} + T_0 \epsilon') + C\cdot \log\frac{2TH}{\delta} \\
        \leq & c_3\cdot ((1 + L_T H)\sqrt{\dimPEII(\cM,\epsilon')T_0 }\log\frac{2|\cM|TH}{\delta} + T_0 \epsilon').
    \end{align*}
    for some constant $c_3$.
    Similarly, we can guarantee by analyzing data collected by $(\pi,\pi)$:
    \begin{align*}
        &\sum_{t=1}^{T_0} \EE_{\pi,M^*(\pi)}[\|\mP_{M^*,h}(\cdot|s_h,a_h,\mu^\pi_{M^*,h}) - \mP_{M^t,h}(\cdot|s_h,a_h,\mu^\pi_{M^*,h})\|_1] \\
        \leq & c_3\cdot ((1 + L_T H)\sqrt{\dimPEII(\cM,\epsilon')T_0 }\log\frac{2|\cM|TH}{\delta} + T_0 \epsilon').
    \end{align*}
    Therefore, by Lem.~\ref{lem:density_est_err} again,
    \begin{align*}
        & \sum_{t=1}^{T_0}  \EE_{\tpi^t,M^*(\pi)}[\sum_{h=1}^H\|\mP_{M^*,h}(\cdot|s_h,a_h,\mu^\pi_{M^*,h}) - \mP_{M^t,h}(\cdot|s_h,a_h,\mu^\pi_{M^t,h})\|_1]\\
        \leq &\sum_{t=1}^{T_0}  \EE_{\tpi^t,M^*(\pi)}[\sum_{h=1}^H\|\mP_{M^*,h}(\cdot|s_h,a_h,\mu^\pi_{M^*,h}) - \mP_{M^t,h}(\cdot|s_h,a_h,\mu^\pi_{M^*,h})\|_1] + \sum_{t=1}^{T_0} \sum_{h=1}^H L_T\|\mu^\pi_{M^*,h} - \mu^\pi_{M^t,h}\|_1 \\
        \leq & c_3 \cdot H(1 + L_T)^H ((1 + L_T H)\sqrt{\dimPEII(\cM,\epsilon')T_0 }\log\frac{2|\cM|TH}{\delta} + T_0 \epsilon').
    \end{align*}
    where the last step we use:
    \begin{align*}
        \sum_{t=1}^{T_0}\sum_{h=1}^H\|\mu^\pi_{M^*,h} - \mu^\pi_{M^t,h}\|_1 \leq & \sum_{h=1}^H \sum_{\ph=1}^h (1+L_T)^{h-\ph} \sum_{t=1}^{T_0} \EE_{\pi,M^*(\pi)}[\|\mP_{M^*,h}(\cdot|s_h,a_h,\mu^\pi_{M^*,h}) - \mP_{M^t,h}(\cdot|s_h,a_h,\mu^\pi_{M^*,h})\|_1] \\
        \leq & c_4 \cdot H\cdot \frac{(1 + L_T)^h - 1}{L_T} ((1 + L_T H)\sqrt{\dimPEII(\cM,\epsilon')T_0 }\log\frac{2|\cM|TH}{\delta} + T_0 \epsilon').
    \end{align*}
    Similarly, for model $M'^t$, we also have:
    \begin{align*}        
        &\sum_{t=1}^{T_0} \sum_{h=1}^H \EE_{\tpi^t,M^*(\pi)}[\|\mP_{M^*,h}(\cdot|s_h,a_h,\mu^\pi_{M^*,h}) - \mP_{M^t,h}(\cdot|s_h,a_h,\mu^\pi_{M^t,h})\|_1]\\
        \leq & c_4 \cdot H(1 + L_T)^H ((1 + L_T H)\sqrt{\dimPEII(\cM,\epsilon')T_0 }\log\frac{2|\cM|TH}{\delta} + T_0 \epsilon').
    \end{align*}
    \paragraph{Step 1-(b): Upper Bound Model Difference with Type $\RI$ P-MBED}
    By Thm.~\ref{thm:MLE_guarantees} and Lem.~\ref{lem:concentration}, w.p. $1-\delta/2TH$, for any $T_0\in[T]$ and any $h\in[H]$, we have:
    \begin{align*}
        \sum_{i=1}^{t-1} \|\mP_{M^*,h}(\cdot|\ts_h^i,\ta_h^i,\mu^\pi_{M^*,h}) - \mP_{M^t,h}(\cdot|\ts_h^i,\ta_h^i,\mu^\pi_{M^t,h})\|_1^2\leq  c_5 \log\frac{2|\cM|TH}{\delta} + C\cdot \log\frac{2TH}{\delta} \leq c_6 \log\frac{2|\cM|TH}{\delta}.
    \end{align*}
    As a result of Lem.~\ref{lem:Partial_Eluder_Bound}, we have:
    \begin{align*}
        \sum_{t=1}^{T_0}\sum_{h=1}^H \|\mP_{M^*,h}(\cdot|\ts_h^t,\ta_h^t,\mu^\pi_{M^*,h}) - \mP_{M^t,h}(\cdot|\ts_h^t,\ta_h^t,\mu^\pi_{M^t,h})\|_1 \leq c_7\cdot H (\sqrt{\dimPEI(\cM,\epsilon')T_0 \log\frac{2|\cM|TH}{\delta}} + T_0 \epsilon').
    \end{align*}
    By Lem.~\ref{lem:concentration}, we have:
    \begin{align*}
        \sum_{t=1}^{T_0}\sum_{h=1}^H \EE_{\tpi^t,M^*(\pi)}[\|\mP_{M^*,h}(\cdot|\ts_h^t,\ta_h^t,\mu^\pi_{M^*,h}) -& \mP_{M^t,h}(\cdot|\ts_h^t,\ta_h^t,\mu^\pi_{M^t,h})\|_1] \\
        \leq & c_8\cdot H (\sqrt{\dimPEI(\cM,\epsilon')T_0} \log\frac{2|\cM|TH}{\delta} + T_0 \epsilon').
    \end{align*}
    Similarly, for model $M'^t$, we also have:
    \begin{align*}
        \sum_{t=1}^{T_0}\sum_{h=1}^H \EE_{\tpi^t,M^*(\pi)}[\|\mP_{M^*,h}(\cdot|\ts_h^t,\ta_h^t,\mu^\pi_{M^*,h}) -& \mP_{M'^t,h}(\cdot|\ts_h^t,\ta_h^t,\mu^\pi_{M^t,h})\|_1] \\
        \leq& c_8\cdot H (\sqrt{\dimPEI(\cM,\epsilon')T_0} \log\frac{2|\cM|TH}{\delta} + T_0 \epsilon').
    \end{align*}

    \paragraph{Step 2: Lower Bound on Model Difference}
    On the other hand, since the algorithm does not terminate at step $T_0$, we have:
    \begin{align*}
        T_0 \teps < & \sum_{t=1}^{T_0} \EE_{\tpi^t,M^t(\pi)}[\sum_{h=1}^H \|\mP_{M^t,h}(\cdot|\cdot,\cdot,\mu^\pi_{M^t,h}) - \mP_{M'^t,h}(\cdot|\cdot,\cdot,\mu^\pi_{M'^t,h})\|_1] \\
        \leq & \sum_{i=1}^{T_0} \EE_{\tpi^t, M^*(\pi)}[\sum_{h=1}^H \|\mP_{M^*,h}(\cdot|s_h,a_h,\mu^\pi_{M^*,h}) - \mP_{M'^t,h}(\cdot|s_h,a_h,\mu^\pi_{M'^t,h})\|_1] \\
        & + (H+1) \cdot \EE_{\tpi^t,M^*(\pi)}[\sum_{h=1}^H\|\mP_{M^*,h}(\cdot|s_h,a_h,\mu^\pi_{M^*,h}) - \mP_{M^t,h}(\cdot|s_h,a_h,\mu^\pi_{M^t,h})\|_1] \numberthis\label{eq:model_diff_lower_bound}.
    \end{align*}
    where in the last step we apply Lem.~\ref{lem:change_of_measure}.
    On the one hand, for Type $\RII$ P-MBED, we have:
    \begin{align*}
        T_0 \teps \leq c_4\cdot H(H+2)(1 + L_T)^H ((1 + L_T H)\sqrt{\dimPEII(\cM,\epsilon')T_0 }\log\frac{2|\cM|TH}{\delta} + T_0 \epsilon'),
    \end{align*}
    by choosing $\epsilon' \leq \teps(2c_4H(H+2)(1+L_T)^H)^{-1}$, it implies, for some constant $c_8$,
    \begin{align*}
        T_0 \leq c_8\cdot \frac{H^4(1+L_T)^{2H}(1+L_T H)^2\dimPEII(\cM,\epsilon')}{\teps^2}\log^2\frac{2|\cM|TH}{\delta}.
    \end{align*}
    On the other hand, for Type $\RI$ P-MBED, we have:
    \begin{align*}
        T_0 \teps \leq c_7\cdot H (H+2)(\sqrt{\dimPEI(\cM,\epsilon')T_0 }\log\frac{2|\cM|TH}{\delta} + T_0 \epsilon'),
    \end{align*}
    by choosing $\epsilon' \geq \teps\cdot (c_7H (H+2))^{-1}$, we have:
    \begin{align*}
        T_0 \leq c_9 \cdot \frac{H^4\dimPEI(\cM,\epsilon')}{\teps^2}\log^2\frac{2|\cM|TH}{\delta}.
    \end{align*}
    As a summary, by choosing
    \begin{align*}
        T = O(\frac{H^4}{\teps^2}\min\{\dimPEI(\cM,\epsilon'),~(1+L_T)^{2H}(1+L_T H)^2\dimPEI(\cM,\epsilon')\}\log^2\frac{2|\cM|TH}{\delta})
    \end{align*}
    with $\epsilon' = O(\teps H^{-2}(1+L_T)^{-H})$,
    we can guarantee the algorithm will terminates for some $T_0 \leq T$ and return us a model class $\bcM^{T_0}$ satisfying $\max_{M,M'\in\bcM^{T_0}} d(M,M'|\pi) \leq \teps$, which implies
    \begin{align*}
        d(M^*,M|\pi) \leq \teps,\quad \forall M \in \bcM^{T_0}.
    \end{align*}
\end{proof}
\begin{remark}[Why $(1+L_T)^H$ Disappears if Considering Type-$\RI$ P-MBED?]\label{remark:exp_disappears}
    From the proof above, especially the proof in Step 1-(a) and Step 1-(b), we can see that during the model elimination, what matters is the model distance conditioning on the density induced by the corresponding models, i.e. $\|\mP_{M,h}(\cdot|\cdot,\cdot,\mu^\pi_{M,h}) - \mP_{M^*,h}(\cdot|\cdot,\cdot,\mu^\pi_{M^*,h})\|_1$.
    Therefore, if we consider the Type-$\RI$ P-MBED, we do not need additional conversion between $\|\mP_{M,h}(\cdot|\cdot,\cdot,\mu^\pi_{M,h}) - \mP_{M^*,h}(\cdot|\cdot,\cdot,\mu^\pi_{M^*,h})\|_1$ and $\|\mP_{M,h}(\cdot|\cdot,\cdot,\mu^\pi_{M^*,h}) - \mP_{M',h}(\cdot|\cdot,\cdot,\mu^\pi_{M^*,h})\|_1$, which is the origin of the exponential term $(1+L_T)^H$ in the upper bound regarding Type-$\RII$ P-MBED.
\end{remark}

\subsection{Proofs for Algorithm~\ref{alg:BridgePolicy}}\label{appx:proofs_BridgePolicy}
Recall the notations for central models in Appx.~\ref{appx:details_PAM}.
\begin{restatable}{theorem}{ThmConstructionErr}\label{thm:ConstructionErr}
    Suppose we feed Alg.~\ref{alg:BridgePolicy} with a model class $\dcM$, the bridge model $\dM_\Bridge$ it computes is a valid model, and by choosing $\beps = \epsilon_0 / \min\{2HL_r \frac{(1+L_T)^H - 1}{L_T}, 2H(H+1)((1 + L_T)^H - 1)\}$, for any reference policy $\pi$ and its associated central model $\dM_\Central^{\epsilon_0}(\pi;\dcM)$, we have:
    \begin{align*}
        \max_{\tpi}\EE_{\tpi,\dM_\Central^{\epsilon_0}(\pi;\dcM)(\pi)}[\sum_{h=1}^H \|\dmP_{\dM_\Central^{\epsilon_0}(\pi;\dcM),h}(\cdot|s_h,a_h,\pi) - \dmP_{\Bridge,h}(\cdot|s_h,a_h,\pi)\|_1] \leq& (H+3)\epsilon_0, \\
        \max_{\tpi}\EE_{\tpi,\dM_\Central^{\epsilon_0}(\pi;\dcM)(\pi)}[\sum_{h=1}^H |\dr_{\dM_\Central^{\epsilon_0}(\pi;\dcM),h}(s_h,a_h,\pi) - \dr_{\Bridge,h}(s_h,a_h,\pi)|] \leq& L_r H(H+4)\epsilon_0.
    \end{align*}
\end{restatable}
\begin{proof}
    In the proof, for notation simplicity, given the refernce policy $\pi$, we use $\dM_\Central^\pi$ as a short note of $\dM_\Central^{\epsilon_0}(\pi;\dcM)$, i.e. the central model regarding $\pi$.
    \paragraph{Validity of Construction}
    First of all, note that $\Pi_\beps$ is an $\beps$-cover of the policy space. Therefore, for any $\pi$, there must exist at least one $\tpi \in \Pi_\beps$ satisfying $d_{\infty,1}(\pi,\tpi) \leq \beps$ which ensures $\sum_{\tpi \in \Pi_\beps}[2\beps - d_{\infty,1}(\pi,\tpi)]^+ > 0$.
    So the transition and reward functions in the bridge model is well-defined, and also continuous in $\pi$ w.r.t. distance $d_{\infty,1}$.
    \paragraph{Upper Bound on Transition Difference}
    By definition, 
    \begin{align*}
        \|\dmP_{\dM_\Central^\pi,h}(\cdot|s_h,a_h,\pi) - \dmP_{\Bridge,h}(\cdot|s_h,a_h,\pi)\|_1 =&\|\dmP_{\dM_\Central^\pi,h}(\cdot|s_h,a_h,\pi) - \frac{\sum_{\tpi \in \Pi_\beps}[2\beps - d_{\infty,1}(\pi,\tpi)]^+\dmP_{\dM_\Central^\tpi,h}(\cdot|s_h,a_h,\tpi)}{\sum_{\tpi \in \Pi_\beps}[2\beps - d_{\infty,1}(\pi,\tpi)]^+}\|_1\\
        \leq & \frac{\sum_{\tpi \in \Pi_\beps}[2\beps - d_{\infty,1}(\pi,\tpi)]^+\|\dmP_{\dM_\Central^\pi,h}(\cdot|s_h,a_h,\pi) - \dmP_{\dM_\Central^\tpi,h}(\cdot|s_h,a_h,\tpi)\|_1}{\sum_{\tpi \in \Pi_\beps}[2\beps - d_{\infty,1}(\pi,\tpi)]^+}
    \end{align*}
    We only need to care about those $\tpi \in \Pi_\beps$ with $[2\beps - d_{\infty,1}(\pi,\tpi)]^+ > 0$, i.e. $d_{\infty,1}(\pi,\tpi) < 2\beps$.
    Given the condition when Alg.~\ref{alg:learning_with_DCP} call Alg.~\ref{alg:BridgePolicy}, we have $\cB^{\epsilon_0}_{\pi}(\dM_\Central^\pi;\dcM) > \frac{|\dcM|}{2}$ for any $\pi$. Therefore, for any $\pi$ and $\tpi$ with $d_{\infty,1}(\pi,\tpi) \leq 2\beps$, there exists a model $\dM_{\share}$ such that $\dM_\share \in  \cB^{\epsilon_0}_{\pi}(\dM_\Central^\pi;\dcM) \cap \cB^{\epsilon_0}_{\tpi}(\dM_\Central^\tpi;\dcM)$, which implies for any $\pi'$
    \begin{align*}
        &\EE_{\pi',\dM_\Central^\pi(\pi)}[\sum_{h=1}^H\|\dmP_{\dM_\Central^\pi,h}(\cdot|s_h,a_h,\pi) - \dmP_{\dM_\Central^\tpi,h}(\cdot|s_h,a_h,\tpi)\|_1]\\
        \leq & \EE_{\pi',\dM_\Central^\pi(\pi)}[\sum_{h=1}^H\|\dmP_{\dM_\Central^\pi,h}(\cdot|s_h,a_h,\pi) - \dmP_{\dM_\share,h}(\cdot|s_h,a_h,\pi)\| + \sum_{h=1}^H\|\dmP_{\dM_\share,h}(\cdot|s_h,a_h,\pi) - \dmP_{\dM_\share,h}(\cdot|s_h,a_h,\tpi)\|_1\\
        & + \sum_{h=1}^H\|\dmP_{\dM_\share,h}(\cdot|s_h,a_h,\tpi) - \dmP_{\dM_\Central^\tpi,h}(\cdot|s_h,a_h,\tpi)\|_1] \\
        \leq & \epsilon_0 + 2H((1 + L_T)^H - 1) \beps + \EE_{\pi',\dM_\Central^\pi(\pi)}[\sum_{h=1}^H \|\dmP_{\dM_\share,h}(\cdot|s_h,a_h,\tpi) - \dmP_{\dM_\Central^\tpi,h}(\cdot|s_h,a_h,\tpi)\|_1] \tag{Lem.~\ref{lem:Lipschitz_PA_model}}.
    \end{align*}
    where by applying Lem.~\ref{lem:model_diff_conversion}, we have:
    \begin{align*}
        &\EE_{\pi',\dM_\Central^\pi(\pi)}[\sum_{h=1}^H \|\dmP_{\dM_\share,h}(\cdot|s_h,a_h,\tpi) - \dmP_{\dM_\Central^\tpi,h}(\cdot|s_h,a_h,\tpi)\|_1]\\
        \leq & \EE_{\pi',\dM_\share(\pi)}[\sum_{h=1}^H \|\dmP_{\dM_\share,h}(\cdot|s_h,a_h,\tpi) - \dmP_{\dM_\Central^\tpi,h}(\cdot|s_h,a_h,\tpi)\|_1] \\
        &+H\cdot \EE_{\pi',\dM_\Central^\pi(\pi)}[\sum_{h=1}^H \|\dmP_{\dM_\Central^\pi,h}(\cdot|s_h,a_h,\pi) - \dmP_{\dM_\share,h}(\cdot|s_h,a_h,\pi)\|_1]\\
        \leq & \EE_{\pi',\dM_\share(\pi)}[\sum_{h=1}^H \|\dmP_{\dM_\share,h}(\cdot|s_h,a_h,\tpi) - \dmP_{\dM_\Central^\tpi,h}(\cdot|s_h,a_h,\tpi)\|_1] + H\epsilon_0 \tag{$\dM_\share\in \cB^{\epsilon_0}_{\pi}(\dM_\Central^\pi;\dcM)$}\\
        \leq &  \EE_{\pi',\dM_\share(\tpi)}[\sum_{h=1}^H \|\dmP_{\dM_\share,h}(\cdot|s_h,a_h,\tpi) - \dmP_{\dM_\Central^\tpi,h}(\cdot|s_h,a_h,\tpi)\|_1]\\
        & + H\cdot \EE_{\pi',\dM_\share(\tpi)}[\sum_{h=1}^H\|\dmP_{\dM_\share,h}(\cdot|s_{h},a_{h},\pi)-\dmP_{\dM_\share,h}(\cdot|s_{h},a_{h},\tpi)\|_1] + H\epsilon_0\\
        \leq &  (H+1)\epsilon_0 + 2H^2((1 + L_T)^H - 1) \beps \tag{$\dM_\share\in \cB^{\epsilon_0}_{\tpi}(\dM_\Central^\tpi;\dcM)$; Lem.~\ref{lem:Lipschitz_PA_model}}.
    \end{align*}
    which implies,
    \begin{align*}
        &\EE_{\pi',\dM_\Central^\pi(\pi)}[\sum_{h=1}^H\|\dmP_{\dM_\Central^\pi,h}(\cdot|s_h,a_h,\pi) - \dmP_{\dM_\Central^\tpi,h}(\cdot|s_h,a_h,\tpi)\|_1]\\
        \leq & \epsilon_0 + 2H((1 + L_T)^H - 1) \beps + (H+1)\epsilon_0 + 2H^2((1 + L_T)^H - 1) \beps\\
        \leq & (H+2)\epsilon_0 + 2H(H+1)((1 + L_T)^H - 1)\beps \\
        \leq & (H+3)\epsilon_0 \tag{$2H(H+1)((1 + L_T)^H - 1)\beps \leq \epsilon_0$}.
    \end{align*}
    Therefore,
    \begin{align*}
        \forall \pi,\pi',\quad \EE_{\pi',\dM_\Central^\pi(\pi)}[\sum_{h=1}^H\|\dmP_{\dM_\Central^\pi,h}(\cdot|s_h,a_h,\pi) - \dmP_{\Bridge,h}(\cdot|s_h,a_h,\pi)\|_1] \leq & \frac{\sum_{\tpi \in \Pi_\beps}[2\beps - d_{\infty,1}(\pi,\tpi)]^+(H+3)\epsilon_0}{\sum_{\tpi \in \Pi_\beps}[2\beps - d_{\infty,1}(\pi,\tpi)]^+} \\
        \leq & (H+3)\epsilon_0.
    \end{align*}
    \paragraph{Upper Bound on Reward Difference}
    By definition, for each $h,s_h,a_h$, we have:
    \begin{align*}
        &|\dr_{\dM_\Central^{\epsilon_0}(\pi;\dcM),h}(s_h,a_h,\pi) - \dr_{\Bridge,h}(s_h,a_h,\pi)| \\
        =& |\dr_{\dM_\Central^{\epsilon_0}(\pi;\dcM),h}(s_h,a_h,\pi) - \frac{\sum_{\tpi \in \Pi_\beps}[2\beps - d_{\infty,1}(\pi,\tpi)]^+\dr_{\dM_\Central^{\epsilon_0}(\tpi;\dcM),h}(s_h,a_h,\tpi)}{\sum_{\tpi \in \Pi_\beps}[2\beps - d_{\infty,1}(\pi,\tpi)]^+}|\\
        =&\frac{\sum_{\tpi \in \Pi_\beps}[2\beps - d_{\infty,1}(\pi,\tpi)]^+|\dr_{\dM_\Central^{\epsilon_0}(\pi;\dcM),h}(s_h,a_h,\pi) - \dr_{\dM_\Central^{\epsilon_0}(\tpi;\dcM),h}(s_h,a_h,\tpi)|}{\sum_{\tpi \in \Pi_\beps}[2\beps - d_{\infty,1}(\pi,\tpi)]^+}.
    \end{align*}
    Similarly, for those $\tpi\in \Pi_\beps$ with $[2\beps - d_{\infty,1}(\pi,\tpi)]^+ > 0$, we have:
    \begin{align*}
        &|\dr_{\dM_\Central^\pi,h}(s_h,a_h,\pi) - \dr_{\dM_\Central^\tpi,h}(s_h,a_h,\tpi)|\\
        =&|r_{h}(s_h,a_h,\mu^\pi_{M_\Central^\pi,h}) - r_{h}(s_h,a_h,\mu^\tpi_{M_\Central^\tpi,h})| \\
        \leq & L_r \|\mu^\pi_{M_\Central^\pi,h} - \mu^\tpi_{M_\Central^\tpi,h}\|_1\\
        \leq & L_r  H d_{\infty,1}(\pi,\pi') + L_r \EE_{\pi,M_\Central^\pi(\pi)}[\sum_{\ph=1}^h \|\mP_{M_\Central^\pi,\ph}(\cdot|s_{\ph-1},a_{\ph-1},\mu^{\pi}_{M_\Central^\pi,\ph-1})-\mP_{M_\Central^\tpi,\ph}(\cdot|s_{\ph-1},a_{\ph-1},\mu^{\tpi}_{M_\Central^\tpi,\ph-1})\|_1]\\
        \leq & 2L_r H\beps + L_r \EE_{\pi,\dM_\Central^\pi(\pi)}[\sum_{h=1}^H\|\dmP_{\dM_\Central^\pi,h}(\cdot|s_h,a_h,\pi) - \dmP_{\dM_\Central^\tpi,h}(\cdot|s_h,a_h,\tpi)\|_1]\\
        \leq & L_r (H+4)\epsilon_0.
    \end{align*}
    Therefore,
    \begin{align*}
        \forall \pi,\pi',\quad &\EE_{\pi',\dM_\Central^\pi(\pi)}[\sum_{h=1}^H|\dr_{\dM_\Central^\pi,h}(s_h,a_h,\pi) - \dr_{\dM_\Central^\tpi,h}(s_h,a_h,\tpi)|] \leq & \frac{\sum_{\tpi \in \Pi_\beps}[2\beps - d_{\infty,1}(\pi,\tpi)]^+H(H+4)\epsilon_0}{\sum_{\tpi \in \Pi_\beps}[2\beps - d_{\infty,1}(\pi,\tpi)]^+} \\
        \leq & L_r H(H+4)\epsilon_0.
    \end{align*}
\end{proof}

Next we prove an important Lemma based on results in theorem above, which indicates that the bridge policy constructed in Alg.~\ref{alg:BridgePolicy} is close to the NE of its central model.
\begin{lemma}\label{lem:BP_close_NE_CM}
    Suppose the \texttt{Else}-branch in Line~\ref{line:else_branch} if activated in Alg.~\ref{alg:elimination_DCP_formal}, for policy $\pi^{\NE,k}_\Bridge$ and its corresponding central model $M^k_\Central := \argmax_{M\in\cM^k}|\cB^{\epsilon_0}_{\pi^{\NE,k}_\Bridge}(M;\cM^k)|$, we have:
    \begin{align*}
        \cE^{\NE}_{M^k_\Central}(\pi^{\NE,k}_\Bridge) := \max_\pi \Delta_{M^k_\Central}(\pi,\pi^{\NE,k}_\Bridge) \leq 2(1+L_r )(H+4)\epsilon_0.
    \end{align*}
\end{lemma}
\begin{proof}
    For any policy $\pi$, we have
    \begin{align*}
        &\Delta_{M^k_\Central}(\pi, \pi^{\NE,k}_\Bridge) \\
        \leq & \Delta_{\dM^k_\Central}(\pi, \pi^{\NE,k}_\Bridge) - \Delta_{\dM_\Bridge}(\pi, \pi^{\NE,k}_\Bridge) \tag{$\Delta_{\dM_\Bridge}(\pi, \pi^{\NE,k}_\Bridge) \leq 0$}\\
        \leq & |\dJ_{\dM^k_\Central}(\pi, \pi^{\NE,k}_\Bridge) - \dJ_{\dM_\Bridge}(\pi, \pi^{\NE,k}_\Bridge)| + |\dJ_{\dM^k_\Central}(\pi^{\NE,k}_\Bridge, \pi^{\NE,k}_\Bridge) - \dJ_{\dM_\Bridge}(\pi^{\NE,k}_\Bridge, \pi^{\NE,k}_\Bridge)| \\
        \leq & \EE_{\pi^{\NE,k}_\Bridge,\dM^k_\Central(\pi^{\NE,k}_\Bridge)}[\sum_{h=1}^H |\dr_{\dM^k_\Central,h}(s_h,a_h,\pi^{\NE,k}_\Bridge) - \dr_{\dM_\Bridge,h}(s_h,a_h,\pi^{\NE,k}_\Bridge)| \\
        & \qquad\qquad\qquad\qquad + \|\dmP_{\dM^k_\Central,h}(\cdot|s_h,a_h,\pi^{\NE,k}_\Bridge), \dmP_{\dM_\Bridge,h}(\cdot|s_h,a_h,\pi^{\NE,k}_\Bridge)\|_1]\\
        & + \EE_{\pi,\dM^k_\Central(\pi^{\NE,k}_\Bridge)}[\sum_{h=1}^H |\dr_{\dM^k_\Central,h}(s_h,a_h,\pi^{\NE,k}_\Bridge) - \dr_{\dM_\Bridge,h}(s_h,a_h,\pi^{\NE,k}_\Bridge)| \\
        & \qquad\qquad\qquad\qquad + \|\dmP_{\dM^k_\Central,h}(\cdot|s_h,a_h,\pi^{\NE,k}_\Bridge), \dmP_{\dM_\Bridge,h}(\cdot|s_h,a_h,\pi^{\NE,k}_\Bridge)\|_1]\\
        \leq & 2\max_\pi \EE_{\pi,\dM^k_\Central(\pi^{\NE,k}_\Bridge)}[\sum_{h=1}^H |\dr_{\dM^k_\Central,h}(s_h,a_h,\pi^{\NE,k}_\Bridge) - \dr_{\dM_\Bridge,h}(s_h,a_h,\pi^{\NE,k}_\Bridge)|]\\
        & + 2\max_\pi \EE_{\pi,\dM^k_\Central(\pi^{\NE,k}_\Bridge)}[\sum_{h=1}^H \|\dmP_{\dM^k_\Central,h}(\cdot|s_h,a_h,\pi^{\NE,k}_\Bridge), \dmP_{\dM_\Bridge,h}(\cdot|s_h,a_h,\pi^{\NE,k}_\Bridge)\|_1]\\
        \leq & 2(1+L_r H)(H+4)\epsilon_0.\tag{Thm.~\ref{thm:ConstructionErr}}
    \end{align*}
    which finishes the proof.
\end{proof}

\subsection{Proofs for Algorithm~\ref{alg:learning_with_DCP}}

\ThmIfElseBranches*
\begin{proof}
    We separately discuss the if and else branches in the algorithm.
    \paragraph{Proof for \texttt{If-Branch} in Line~\ref{line:if_branch}}
    On the events in Thm.~\ref{thm:Elimination_Alg}, for any $\tM \not\in \cB^{\epsilon_0}_{\pi^k}(M^*;\cM^k)$, we have $d(M^*,\tM|\pi^k) \geq \epsilon_0 > \teps$, which implies $\tM \not\in \cM^{k+1}$.
    Combining the condition of \texttt{If-Branch}, we have:
    \begin{align*}
        |\cM^{k+1}| \leq |\cB^{\epsilon_0}_{\pi^k}(M^*; \cM^k)| \leq \frac{|\cM^k|}{2}.
    \end{align*}
    \paragraph{Proof for \texttt{Else-Branch} in Line~\ref{line:else_branch}}
    First of all, on the events in Thm.~\ref{thm:Elimination_Alg}, we have $d(M^*,\tM^k|\pi^{\NE,k}_\Bridge) \leq \teps$. By applying Lem.~\ref{lem:exploitability_diff}, it implies:
    \begin{align*}
        & |\Delta_{M^*}(\pi, \pi^{\NE,k}_\Bridge) - \Delta_{\tM^k}(\pi, \pi^{\NE,k}_\Bridge)| \\
        \leq & \EE_{\pi,M^*(\pi^{\NE,k}_\Bridge)}[\sum_{h=1}^H \|\mP_{M^*,h}(\cdot|s_h,a_h,\mu^{\pi^{\NE,k}_\Bridge}_{M^*,h}) - \mP_{\tM^k,h}(\cdot|s_h,a_h,\mu^{\pi^{\NE,k}_\Bridge}_{\tM^k,h})\|_1]\\
        & + (2L_rH + 1) \EE_{\pi^{\NE,k}_\Bridge,M^*(\pi^{\NE,k}_\Bridge)}[\sum_{h=1}^H  \|\mP_{M^*,h}(\cdot|s_h,a_h,\mu^{\pi^{\NE,k}_\Bridge}_{M^*,h}) - \mP_{\tM^k,h}(\cdot|s_h,a_h,\mu^{\pi^{\NE,k}_\Bridge}_{\tM^k,h})\|_1] \\
        \leq & 2(L_rH + 1) \teps.
    \end{align*}
    Also note that:
    \begin{align*}
        \Delta_{M^*}(\pi, \pi^{\NE,k}_\Bridge) =& \Delta_{M^*}(\pi, \pi^{\NE,k}_\Bridge) - \Delta_{\tM^k}(\pi, \pi^{\NE,k}_\Bridge) + \Delta_{\tM^k}(\pi, \pi^{\NE,k}_\Bridge) .
    \end{align*}
    In the following, we separately discuss two cases.
    \paragraph{Case 1: $\cE^\NE_{\tM^k}(\pi^{\NE,k}_\Bridge) \leq \frac{3\epsilon}{4}$ and Line~\ref{line:else_if_branch} is activated}
    Given that $\teps \leq \frac{\epsilon}{16(1+L_rH)}$:
    \begin{align*}
        \Delta_{M^*}(\pi, \pi^{\NE,k}_\Bridge)\leq & |\Delta_{M^*}(\pi, \pi^{\NE,k}_\Bridge) - \Delta_{\tM^k}(\pi, \pi^{\NE,k}_\Bridge)| + \cE^\NE_{\tM^k}(\pi^{\NE,k}_\Bridge) \tag{$\cE^\NE_{\tM^k}(\pi^{\NE,k}_\Bridge) = \max_\pi \Delta_{\tM^k}(\pi, \pi^{\NE,k}_\Bridge)$}\\
        \leq & 2(L_rH + 1)\teps + \frac{3\epsilon}{4} \leq \epsilon.
    \end{align*}
    which implies $\pi^{\NE,k}_\Bridge$ is an $\epsilon$-NE of $M^*$.
    \paragraph{Case 2: $\cE^\NE_{\tM^k}(\pi^{\NE,k}_\Bridge) > \frac{3\epsilon}{4}$ and Line~\ref{line:else_if_branch} is not activated}
    As a result, for any policy $\pi$,
    \begin{align*}
        \Delta_{M^*}(\pi, \pi^{\NE,k}_\Bridge) \geq - |\Delta_{M^*}(\pi, \pi^{\NE,k}_\Bridge) - \Delta_{\tM^k}(\pi, \pi^{\NE,k}_\Bridge)| + \Delta_{\tM^k}(\pi, \pi^{\NE,k}_\Bridge) \geq \Delta_{\tM^k}(\pi, \pi^{\NE,k}_\Bridge) - 2(L_rH + 1)\teps.
    \end{align*}
    Therefore, by our choice of $\teps$,
    \begin{align*}
        \max_\pi \Delta_{M^*}(\pi, \pi^{\NE,k}_\Bridge) \geq \cE^\NE_{\tM^k}(\pi^{\NE,k}_\Bridge) - 2(L_rH + 1)\teps \geq \frac{5\epsilon}{8}.
    \end{align*}
    On the other hand, by Lem.~\ref{lem:BP_close_NE_CM}, for any $\pi$, we have:
    \begin{align*}
        & \Delta_{M^*}(\pi, \pi^{\NE,k}_\Bridge) - 2(1+L_rH)(H+4)\epsilon_0 \\
        \leq &|\Delta_{M^*}(\pi, \pi^{\NE,k}_\Bridge)| - |\Delta_{M^k_\Central}(\pi, \pi^{\NE,k}_\Bridge)| \tag{Here we apply Lem.~\ref{lem:BP_close_NE_CM}}\\
        \leq & |\Delta_{M^*}(\pi, \pi^{\NE,k}_\Bridge) - \Delta_{M^k_\Central}(\pi, \pi^{\NE,k}_\Bridge)| \\
        \leq & \EE_{\pi,M^*(\pi^{\NE,k}_\Bridge)}[\sum_{h=1}^H \|\mP_{M^*,h}(\cdot|s_h,a_h,\mu^{\pi^{\NE,k}_\Bridge}_{M^*,h}) - \mP_{M^k_\Central,h}(\cdot|s_h,a_h,\mu^{\pi^{\NE,k}_\Bridge}_{M^k_\Central,h})\|_1]\\
        & + (2L_rH + 1) \EE_{\pi^{\NE,k}_\Bridge,M^*(\pi^{\NE,k}_\Bridge)}[\sum_{h=1}^H  \|\mP_{M^*,h}(\cdot|s_h,a_h,\mu^{\pi^{\NE,k}_\Bridge}_{M^*,h}) - \mP_{M^k_\Central,h}(\cdot|s_h,a_h,\mu^{\pi^{\NE,k}_\Bridge}_{M^k_\Central,h})\|_1] \\
        \leq & (2L_r H + 2) d(M^*,M^k_\Central|\pi^{\NE,k}_\Bridge).
    \end{align*}
    According to the choice of $\epsilon_0$, we have $2(1+L_rH)(H+4)\epsilon_0 \leq \frac{\epsilon}{4}$, therefore,
    \begin{align*}
        d(M^*,M^k_\Central|\pi^{\NE,k}_\Bridge) \geq \frac{1}{2L_r H + 2}\Big(\max_\pi \Delta_{M^*}(\pi, \pi^{\NE,k}_\Bridge) - 2(1+L_rH)(H+4)\epsilon_0\Big) \geq \frac{3\epsilon}{16(L_r H + 1)}.
    \end{align*}
    Next we try to show that models in $\cB^{\epsilon_0}_{\pi^{\NE,k}_\Bridge}(M^k_\Central, \cM^k)$ will be eliminated.
    For any $M \in \cB^{\epsilon_0}_{\pi^{\NE,k}_\Bridge}(M^k_\Central, \cM^k)$, because of $\teps < \frac{\epsilon}{48(L_r H + 1)}$ we have:
    \begin{align*}
        d(M,M^*|\pi^{\NE,k}_\Bridge) \geq d(M^k_\Central,M^*|\pi^{\NE,k}_\Bridge) - d(M,M^k_\Central|\pi^{\NE,k}_\Bridge) \geq \frac{3\epsilon}{16(L_r H + 1)} - \epsilon_0 \geq \frac{\epsilon}{16(L_r H + 1)} > \teps.
    \end{align*}
    On the event in Thm.~\ref{thm:Elimination_Alg} (which holds with probability $1-\delta$), $M \not\in \cM^{k+1}$, which implies,
    $$
        |\cM^{k+1}| \leq |\cM^k| - |\cB^{\epsilon_0}_{\pi^{\NE,k}_\Bridge}(M^k_\Central, \cM^k)| \leq |\cM^k| / 2.
    $$
\end{proof}

\begin{restatable}{theorem}{ThmMainResult}[Sample Complexity of Learning MFGs]\label{thm:sample_complexity_MFG}
    Under Assump.~\ref{assump:realizability} and~\ref{assump:Lipschitz},
    by running Alg.~\ref{alg:learning_with_DCP} with Alg.~\ref{alg:elimination_DCP_formal} as \texttt{ModelElim} and Alg.~\ref{alg:BridgePolicy} as \texttt{BridgePolicy}, and hyper-parameter choices according to Thm.~\ref{thm:Elimination_Alg},~\ref{thm:if_else_branches}, and~\ref{thm:ConstructionErr}, w.p. $1-\delta$, Alg.~\ref{alg:learning_with_DCP} will terminate at some $k \leq \log_2|\cM| + 1$ and return us an $\epsilon$-NE of $M^*$, and the number of trajectories consumed is at most $O(\frac{H^7}{\epsilon^2}(1+L_r)^2(\dimPEI(\cM,\epsilon') \wedge H^3(1+L_T H)^2(1+L_T)^{2H}\dimPEII(\cM,\epsilon'))\log^3\frac{|\cM|}{\delta})$ where $\epsilon'=O(\epsilon/H^3(1+L_r)(1+L_T)^H)$, and in $\tilde{O}$ we omit logarithmic terms of $\epsilon,H,\log|\cM|,\dimPE$, $1+L_T$ and $1 + L_r$.
\end{restatable}
\begin{proof}
    As a result of Thm.~\ref{thm:if_else_branches}, w.p. $1-\frac{\delta}{\log_2|\cM|+1}\cdot(\log_2|\cM|+1) = 1-\delta$, there exists a step $k \leq \log_2|\cM| + 1$ such that Alg.~\ref{alg:learning_with_DCP} will terminate the return us an $\epsilon$-approximate NE of $M^*$. The total number of trajectories required is:
    \begin{align*}
        &(\log_2|\cM| + 1) \cdot 2HT \\
        =& (\log_2|\cM| + 1) \cdot \tilde{O}(\frac{H^5}{\teps^2}(\dimPEI(\cM,\epsilon')\wedge (1+L_T)^{2H}(1+L_T H)^2\dimPEII(\cM,\epsilon'))\log^2\frac{2|\cM|TH}{\delta}) \\
        =& \tilde{O}(\frac{H^7}{\epsilon^2}(1+L_rH)^2\Big(\dimPEI(\cM,\epsilon') \wedge H^3(1+L_T H)^2(1+L_T)^{2H}\dimPEII(\cM,\epsilon')\Big)\log^3|\cM|).
    \end{align*}
    where we use the fact that by Thm.~\ref{thm:Elimination_Alg}, we choose $\teps = \frac{\epsilon_0}{6} = O(\frac{\epsilon}{(1+L_rH)H})$, and $\epsilon'=O(\teps/H^2(1+L_T)^H) = O(\epsilon/H^3(1+L_rH)(1+L_T)^H)$.
\end{proof}

\subsection{Sample Complexity Separation between Mean-Filed Control and Mean-Field Games}\label{appx:separation_MFC_MFG}
In this section, we establish the separation between of RL in MFC and MFGs from information theoretical perspective.
\paragraph{A Basic Recap of the MFC Setting}
In MFC, similar to single-agent RL, we are interested in finding a policy $\hat\pi^*_\Opt$ to approximately minimize the optimality gap $\cE_{\Opt}(\pi):= \max_\tpi J_{M^*}(\tpi;\vecmu^\tpi_{M^*}) - J_{M^*}(\pi;\vecmu^{\pi}_{M^*})$, i.e.,
\begin{equation}
    \cE_{\Opt}(\hat\pi^*_\Opt)\leq \epsilon.\label{eq:objective_MFC}
\end{equation}

\paragraph{Exponential Lower Bound in Tabular RL for Mean-Field Control}
Our results are based on a different query model from Def.~\ref{def:collection_process} defined below.
\begin{definition}[Strong Query Model]\label{def:SQM}
    The Strong Query Model (SQM) can take a policy $\pi$ and return a sequence of transition function $\{\mP^\pi_h(\cdot|\cdot,\cdot)\}_{h=1}^H$, such that $\mP^\pi_h(\cdot|s_h,a_h) := \mP_{M^*,h}(\cdot|s_h,a_h,\mu^\pi_{M^*,h})$ for any $h\in[H],s_h\in\cS_h,a_h\in\cA_h$.
\end{definition}
\noindent The SQM is strictly stronger than the sample query model in Def.~\ref{def:collection_process}, because given the conditional model $\{\mP^\pi_h(\cdot|\cdot,\cdot)\}_{h=1}^H$, one can sample arbitrary trajectories by arbitrary policies from it, and therefore, recover the data collection process in Def.~\ref{def:collection_process}. 
In the following, we investigate the number of SQM queries required to identify $\epsilon$-optimal policy in MFC setting.
We show that, under Assump.~\ref{assump:realizability} and~\ref{assump:Lipschitz}, even in the tabular setting, MFC requires queries exponential to the number of states and actions.
\begin{figure}[h]
    \begin{center}
        \includegraphics[scale=0.7]{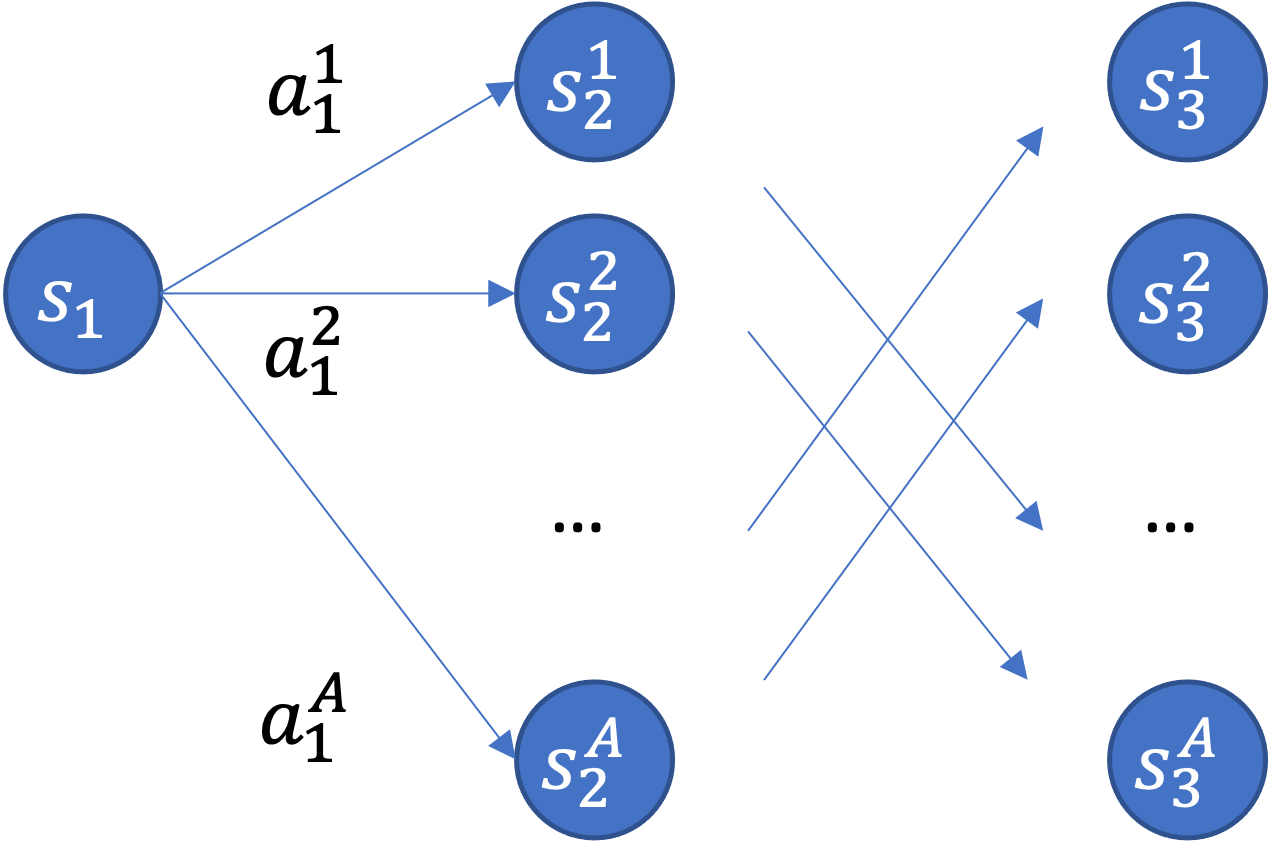}
        \caption{Construction of Lower Bound}\label{fig:lower_bound}
    \end{center}
\end{figure}
\begin{restatable}{theorem}{ThmLowerBound}[Exponential Lower Bound for MFC]\label{thm:LB_MFC}
    Given arbitrary $L_T > 0$ and $d \geq 2$, consider tabular MF-MDPs satisfying Assump.~\ref{assump:Lipschitz} with Lipschitz coefficient $L_T$, $|\cS| = |\cA| = d$ and $H=3$. For any algorithm $\Alg$, and any $\epsilon \leq \frac{L_T}{d + 1}$, there exists an MDP $M^*$ and a model class $\cM$ satisfying $M^* \in \cM$, and $|\cM| = \Omega((\frac{L_T}{d\epsilon})^{d-1})$, s.t., if $\Alg$ only queries GM or DCP for at most $K$ times with $K \leq |\cM|/2 - 1$, the probability that $\Alg$ produces an $\epsilon$-optimal policy is less than $1/2$.
\end{restatable}
\begin{proof}
    Our proof is divided into three parts: construction of hard MF-MDP instance, construction of model class $\cM$, and the proof of lower bound.
    \paragraph{Part 1: Construction of Hard Examples}
    We construct a three layer MDP as shown in Fig.~\ref{fig:lower_bound}. The initial state distribution is fixed to be $\mu_1(s_1) = 1$, and we have $S$ states and $A$ actions available at each layer with $S=A=d$.
    The transition at initial state is deterministic, i.e., $\mP(s_2^i|s_1,a_1^i,\mu_1)=1$.
    At the second layer, given $L_T \leq 1$, there exists an optimal state density $\mu_2^*$, such that, $\forall i\in[S],j\in[A]$ and $\forall \mu_2 \in \Delta(\cS)$:
    \begin{align*}
        \mP(s_3^1|s_2^i,a_2^j,\mu_2) = \frac{1}{2}+2\epsilon\cdot \Big[1-\frac{L_T}{4\epsilon}\|\mu_2 - \mu_2^*\|_1\Big]^+,\quad \mP(s_3^2|s_2^i,a_2^j,\mu_2) = \frac{1}{2}-2\epsilon\cdot\Big[1-\frac{L_T}{4\epsilon}\|\mu_2 - \mu_2^*\|_1\Big]^+.
    \end{align*}
    where $[x]^+ = \max\{x,0\}$.
    As for the reward function, we have zero reward at each state action in the previous two layers, and for the third layer, we have only have non-zero reward at $r_3(s_3^1,\cdot,\cdot) = 1$ and $r_3(s_3^i,\cdot,\cdot) = 0$ for all $i \neq 1$.
    
    As we can see, for arbitrary policy $\pi$, we have $\mu_2^\pi(s_2^i) = \pi(a_1^i|s_1)$.
    Besides, the optimal policy should be taking action to make sure $\mu_2 = \mu_2^*$, which can be achieved by setting $\pi^*(a_1^i|s_1) = \mu_2^*(s_2^i)$, and then take arbitrary policy at the second layer. 
    Even if the agent just wants to achieve $\epsilon$-near-optimal policy, it at least has to determine the position of set $\{\mu:\|\mu-\mu_2^*\|_1 \leq \frac{4\epsilon}{L_T}\}$.
    The key difficulty here is to explore and gather information which can be used to infer $\mu_2^*$.
    
    We further reduce the difficulty of the exploration by providing for the learner with the transition at initial state and the third layer (or equivalently, the available representation function for the first and third layers is unique) and all the information of reward function. All the learner need to do is to identify the correct feature for the second layer and use it to obtain the optimal policy (at the initial state) to maximize the return. 
    
    Next, we verify the above model belongs to the low-rank Mean-Field MDP.
    For $h=1$, it's easy to see $\mP(s_2^i|s_1,a_1^j,\mu_1) = \phi_1(s_1,a_1^j,\mu_1)\trans\psi_1(s_2^i)$, where $\phi_1(s_1,a_1^j,\mu_1) = \textbf{e}_j$ and $\psi_1(s_2^i)=\textbf{e}_i$, and  $\textbf{e}_{(\cdot)}$ is the one-hot vector with the $(\cdot)$-th element equal 1.
    For the second layer, given a density $\mu \in \Delta(\cS)$, we use $\phi_{\mu,L_T}$ to denote the following feature function class that, $\forall i\in[S],j\in[A],\mu'\in\Delta(\cS)$,
    \begin{align*}
        \phi_{\mu,L_T}(s_2^i,a_2^j,\mu') := (\frac{1}{2} + 2\epsilon\cdot\Big[1-\frac{L_T}{4\epsilon}\|\mu' - \mu\|_1\Big]^+, \frac{1}{2}-2\epsilon \cdot\Big[1-\frac{L_T}{4\epsilon}\|\mu' - \mu\|_1\Big]^+, 0, ..,0)\trans \in \mR^{d}.
    \end{align*}
    and the next state feature function is $\psi(s_3^i) = \textbf{e}_i\trans,\quad \forall i \in [d]$.
    It's easy to verify that the transition can be decomposed to $\phi_{\mu^*_2,L_T}(\cdot,\cdot,\mu_2)\trans\psi(s_3^i)$, and the above feature satisfies the normalization property:
    \begin{align*}
        \|\sum_{i\in[d]} \psi(s_3^i) g(s_3^i)\| \leq \sqrt{2}d,\quad \forall g:\cS\rightarrow\{-1,1\}.
    \end{align*}
    Besides, we verify that for any choice of $\mu$, the induced transition function is $L_T$-Lipschitz:
    \begin{align*}
        & \|\mP_{\mu,L_T}(\cdot|s_2^i,a_2^j,\mu') - \mP_{\mu,L_T}(\cdot|s_2^i,a_2^j,\mu'')\|_1 \\
        =&\sum_{l\in[S]}|\phi_{\mu,L_T}(s_2^i,a_2^j,\mu')\trans\psi(s_3^l) - \phi_{\mu,L_T}(s_2^i,a_2^j,\mu'')\psi(s_3^l)|\\
        =&2\cdot 2\epsilon |\Big[1-\frac{L_T}{4\epsilon}\|\mu - \mu''\|_1\Big]^+ - \Big[1-\frac{L_T}{4\epsilon}\|\mu - \mu'\|_1\Big]^+|\\
        \leq & L_T |\|\mu-\mu'\|_1 - \|\mu-\mu''\|_1|\leq L_T \|\mu' - \mu''\|_1
    \end{align*}
    \paragraph{Part 2: Construction of Model Class}
    Given an integer $\zeta$, we denote $\cN_{\zeta} := \{\mu|\mu(s^i_2) = N(s^i_2) / \zeta,~ N(s^i_2) \in \textbf{N},~\sum_{i\in[S]} N(s^i_2) = \zeta\}$. In another word, $\cN_{\zeta}$ includes all state density with resolution $1/\zeta$. Now, consider $\cN_{\lfloor\frac{L_T}{5\epsilon}\rfloor}$. For each $\mu,\mu'\in \cN_{\lfloor\frac{L_T}{5\epsilon}\rfloor}$, we should have:
    \begin{align*}
            \|\mu - \mu'\|_1 \geq 2 / \lfloor\frac{L_T}{5\epsilon}\rfloor \geq \frac{10\epsilon}{L_T} > \frac{8\epsilon}{L_T}.
    \end{align*}
    Therefore, if we consider the set $\cB(\mu, \frac{4\epsilon}{L_T}) := \{\mu'\in\Delta(\cS)|\|\mu-\mu'\|_1 \leq \frac{4\epsilon}{L_T}\}$, we can expect $\cB(\mu, \frac{4\epsilon}{L_T}) \cap \cB(\mu', \frac{4\epsilon}{L_T}) = \emptyset$ for any $\mu, \mu' \in \cN_{\lfloor\frac{L_T}{5\epsilon}\rfloor}$.
    Given arbitrary $N \leq |\cN_{\lfloor\frac{L_T}{5\epsilon}\rfloor}| = \frac{(\lfloor\frac{L_T}{5\epsilon}\rfloor + d - 1)!}{(\lfloor\frac{L_T}{5\epsilon}\rfloor)! (d-1)!} = \Omega((\frac{L_T}{d\epsilon})^{d-1})$, we can find $N - 1$ different elments $\{\mu_2^1,...,\mu_2^N\}\subset\cN_{\lfloor\frac{L_T}{5\epsilon}\rfloor}$ and construct (here we only specify the representation at the second layer, since we assume the other layers are known)
    \begin{align*}
        \cM^{[N]} := \{M^n := (\phi_{\mu^n_2, L_T}, \psi) | n\in[N]\}.
    \end{align*}
    For analysis, we introduce another model $\bar{M}$ which shares the transition and reward function as $M^n$s but for the transition of second layer, it has:
    \begin{align*}
        \mP(s^1_3|s_2^i,a_2^j,\mu_2) = \mP(s^2_3|s_2^i,a_2^j,\mu_2) = \frac{1}{2},\quad \forall i\in[S],j\in[A],\mu_2\in\Delta(\cS).
    \end{align*}
    We define:
    \begin{align*}
        \bar{\phi}(\cdot,\cdot,\cdot) = (\frac{1}{2},...,\frac{1}{2}) \in \mR^d.
    \end{align*}
    and define:
    \begin{align*}
        \cM := \cM^{[N]} \cup \{(\bar{\phi}, \psi)\}.
    \end{align*}
    Note that $\bar{M} = (\bar{\phi}, \psi) \in \cM$.
    
    \paragraph{Part 3: Establishing Lower Bound}
    Now, we consider the following learning setting: the environment randomly select one model $M$ from $\cM$ and provide the entire representation feature class $\cM$ (which is also the entire model class) to the learner; then, the learner can repeatedly use gathered information to compute a policy $\pi^k$ and query it with SQM for each iteration, and output a final policy after $K$ steps.
    We want to show that, for arbitrary algorithm, there exists at least one model in $\cM$ which cost number of queries linear w.r.t. $N$ before identifying the optimal policy.
    
    In the following, we use $\cE_{k,M^n}$ to denote the event that in the first $k$ trajectories, there is at least one policy (or equivalently, density $\mu^\pi_2$) used to query SQM resulting in $\|\mu^\pi_2 - \mu^n\|_1 \leq \frac{4\epsilon}{L_T}$. The key observation is that, given arbitrary algorithm $\Alg$, for arbitrary fixed $n\in[N]$, if $\Alg$ never deploy a policy $\pi$ (or equivalently, query an density $\mu^\pi_2$) satisfying $\|\mu^\pi_2 - \mu^n\|_1 \leq \frac{4\epsilon}{L_T}$, the algorithm can not distinguish between $M^n$ and $\bar{M}$, and should behave similar in both $M^n$ and $\bar{M}$. Therefore,
    \begin{align*}
        \textPr_{M^n,\Alg}(\cE_{k,M^n}^\complement) = \textPr_{\bar{M},\Alg}(\cE_{k,M^n}^\complement),\quad \forall k \in [K].
    \end{align*}
    which also implies:
    \begin{align*}
        \textPr_{M^n,\Alg}(\cE_{k,M^n}) = \textPr_{\bar{M},\Alg}(\cE_{k,M^n}),\quad \forall k \in [K].
    \end{align*}
    We use $\Alg(K)$ to denote the policy output by the algorithm in the final. 
    Besides, we use $\Pi( \mu, b_0) := \{\pi|\|\mu_2^\pi - \mu\|_1 \leq b_0\}$ to denote the set of policies, which can lead to a density $\mu_2^\pi$ close to $\mu$.
    Then, we have:
    \begin{align*}
        &\sum_{n\in[N]} \textPr_{M^n,\Alg}(\Alg(K) \in \Pi(\mu^n, \frac{4\epsilon}{L_T})) - \textPr_{\bar{M},\Alg}(\Alg(K) \in \Pi(\mu^n, \frac{4\epsilon}{L_T})) \\
        = & \sum_{n\in[N]} \textPr_{M^n,\Alg}(\{\Alg(K) \in \Pi(\mu^n, \frac{4\epsilon}{L_T})\} \cap \{\cE_{K,M^n}\}) - \textPr_{\bar{M},\Alg}(\{\Alg(K) \in \Pi(\mu^n, \frac{4\epsilon}{L_T})\}\cap \{\cE_{K,M^n}\}) \\
        & + \sum_{n\in[N]} \textPr_{M^n,\Alg}(\{\Alg(K) \in \Pi(\mu^n, \frac{4\epsilon}{L_T})\} \cap \{\cE_{K,M^n}^\complement\}) - \textPr_{\bar{M},\Alg}(\{\Alg(K) \in \Pi(\mu^n, \frac{4\epsilon}{L_T})\}\cap \{\cE_{K,M^n}^\complement\})\\
        = & \sum_{n\in[N]} \textPr_{M^n,\Alg}(\{\Alg(K) \in \Pi(\mu^n, \frac{4\epsilon}{L_T})\} \cap \{\cE_{K,M^n}\}) - \textPr_{\bar{M},\Alg}(\{\Alg(K) \in \Pi(\mu^n, \frac{4\epsilon}{L_T})\}\cap \{\cE_{K,M^n}\}) \\
        \leq & \sum_{n\in[N]} \textPr_{M^n,\Alg}(\cE_{k,M^n})\Big(\textPr_{M^n,\Alg}(\Alg(K) \in \Pi(\mu^n, \frac{4\epsilon}{L_T})|\cE_{K,M^n}) - \textPr_{\bar{M},\Alg}(\Alg(K) \in \Pi(\mu^n, \frac{4\epsilon}{L_T})|\cE_{K,M^n})\Big)\\
        \leq & \sum_{n\in[N]} \textPr_{M^n,\Alg}(\cE_{k,M^n}) = \sum_{n\in[N]} \textPr_{\bar{M},\Alg}(\cE_{k,M^n}) \leq K.
    \end{align*}
    where the last step is because,
    \begin{align*}
        \sum_{n\in[N]} \textPr_{\bar{M},\Alg}(\cE_{k,M^n}) \leq & \sum_{n\in[N]} \sum_{k=1}^K \textPr_{\bar{M},\Alg}(\|\mu^{\pi^k}_2 - \mu^n\|_1 \leq \frac{4\epsilon}{L_T})= \sum_{k=1}^K \sum_{n\in[N]} \textPr_{\bar{M},\Alg}(\|\mu^{\pi^k}_2 - \mu^n\|_1 \leq \frac{4\epsilon}{L_T}) \leq \sum_{k=1}^K 1 = K \tag{$\cB(\mu^i, \frac{4\epsilon}{L_T}) \cap \cB(\mu^j, \frac{4\epsilon}{L_T}) = \emptyset$ for all $i\neq j$}.
    \end{align*}
    Therefore, the average success probability would be:
    \begin{align*}
        &\textPr(M=\bar{M}) + \sum_{n\in[N]}\textPr(\{M=M^n\}\cap\{\Alg(K)\in \Pi(\mu^n, \frac{4\epsilon}{L_T})\})\tag{Each policy is optimal in $\bar{M}$.}\\
        =& \frac{1}{|\cM|} + \frac{1}{|\cM|} \sum_{n\in[N]} \textPr_{M^n,\Alg}(\Alg(K)\in \Pi(\mu^n, \frac{4\epsilon}{L_T})) \leq \frac{K+1}{|\cM|}.
    \end{align*}
    As a result, even if $K = \frac{|\cM|}{2} - 1 = O(N)$, there exists $n\in[N]$, such that, the failure rate 
    $$
    \textPr_{M^n,\Alg}(\Alg(K) \not\in \cB(\pi^*_{M^n}, \frac{4\epsilon}{L_T})) \geq \frac{1}{2}.
    $$
\end{proof}

\newpage
\section{Proofs for Multi-Type MFGs}\label{appx:proof_MT_PAM}

\subsection{More Details about the Setting}\label{appx:more_details_setting_MT}
In Multi-Type MF-MDP, we will denote $\mu_h^w \in \Delta(\cS^w)$ to be a state density at step $h$ in type $w$, and define $\mu^w := \{\mu_h^w\}_{h\in[H]}$ to be the collection over all $h$.
For the policies, we define $\Pi^w := \{\pi^w:=\{\pi^w_h\}|\forall h\in[H], \pi^w_h:\cS^w_h\rightarrow\Delta(\cA^w_h)\}$, and $\vecPi := \{\vecpi:=\{\pi^w\}_{w\in[W]}|\forall w\in[W],\pi^w\in\Pi^w\}$.
In this paper, we only consider policies in $\vecPi$, i.e. the set of non-stationary Markovian policies.

In order to distinguish with (single-type) MF-MDP setting, for notations regarding the collection of densities or policies over all groups, we use the bold font, i.e. $\vecmu_h := \{\mu_h^w\}_{w\in[W]}$ and $\vecmu:=\{\vecmu_h\}_{h\in[W]}$, $\vecpi := \{\pi^w\}_{w\in[W]}$ and $\vecpi_h := \{\pi^w_h\}_{w\in[W]}$.
When a policy $\vecpi$ and a model $\vecM$ is speicified, we use $\vecmu^{\vecpi}_\vecM := \{\mu^{w,\vecpi}_{\vecM}\}_{w\in[W]} = \{\vecmu^{\vecpi}_{\vecM,h}\}_{h\in[H]}$ to denote the collection of densities of $W$ groups induced by the policy $\pi$ in model $\vecM$, where $\mu^{w,\vecpi}_{\vecM} := \{\mu^{w,\vecpi}_{\vecM,h}\}_{h\in[H]}$ and $\vecmu^{\vecpi}_{\vecM,h}:=\{\mu^{w,\vecpi}_{\vecM,h}\}_{w\in[W]}$. 
When a policy $\vecpi \in \vecPi$ is specified, the evolution of the densities in all groups can be described by:
\begin{align*}
    \forall h\in[H],~\forall w\in[W], \quad& \mu^{w,\vecpi}_{\vecM,h+1} = \Gamma^{w,\pi^w}_{\vecM,h}(\mu^{w,\vecpi}_{\vecM,h}),\\
    &\text{with}~\Gamma^{w,\pi^w}_{\vecM,h}(\mu_h^w)(\cdot) := \sum_{s_h^w,a_h^w} \mu_h^w(s_h^w)\pi^w_h(a_h^w|s_h^w) \mP^w_{\vecM,h}(\cdot|s^w_h,a^w_h,\vecmu^{\vecpi}_{\vecM,h})
\end{align*}
Similarly to MF-MDP setting, given two policies $\tvecpi,\vecpi \in \vecPi$, we can define the value functions for each group following $\tvecpi$ while conditioning on $\vecpi$:
\begin{align*}
    Q^{w,\tvecpi}_{\vecM,h}(\cdot,\cdot;\vecmu^\vecpi_\vecM) :=& \EE_{\tvecpi,\vecM(\vecpi)}[\sum_{\ph=h}^H r_\ph^w(s_\ph^w,a_\ph^w,\vecmu^\vecpi_{\vecM,\ph})|s_h^w=\cdot,a_h^w=\cdot],\\ 
    V^{w,\tvecpi}_{\vecM,h}(\cdot;\vecmu^\vecpi_{\vecM}) :=& \EE_{\tvecpi,\vecM(\vecpi)}[\sum_{\ph=h}^H r_\ph^w(s_\ph^w,a_\ph^w,\vecmu^\vecpi_{\vecM,\ph})|s_h^w=\cdot],\\
    J^w_\vecM(\tvecpi; \vecpi) :=& \EE_{s_1^w\sim\mu_1^w}[V^{w,\tvecpi}_{\vecM,1}(s_1^w)].
\end{align*}
where we use $\EE_{\tvecpi,\vecM(\vecpi)}$ to denote the expectation over trajectories generated by executing policy $\tvecpi$ in $\vecM$ conditioning on $\vecpi$, i.e. the transitions $\mP_{\vecM,h}^w(\cdot|\cdot,\cdot,\vecmu^\vecpi_{\vecM,h})$ and rewards $r_h^w(\cdot,\cdot,\vecmu^\vecpi_{\vecM,h})$ are fixed by $\vecpi$.
Besides, we denote $J_\vecM^w(\tvecpi;\vecpi) := \EE_{s_1^w\sim\mu_1^w}[V^{w,\vecpi}_{\vecM,1}(s_1^w;\vecmu^\vecpi_\vecM)]$ to be the expected return of type $w$ in model $\vecM$ by executing $\tvecpi$ conditioning on $\vecpi$.
The Nash Equilibrium policy in Multi-Type MFG is defined to be the policy $\vecpi^\NE:=\{\pi^{w,\NE}\}_{w\in[W]}$ satisfying:
\begin{align}
    \forall w \in [W],~\forall\tvecpi \in \vecPi,\quad J^w_\vecM(\tvecpi;\vecpi^\NE) \leq J^w_\vecM(\vecpi^\NE;\vecpi^\NE).\label{eq:NE_MT}
\end{align}
We define $\Delta^w_\vecM(\tvecpi,\vecpi^w) := J^w_\vecM(\tvecpi;\vecpi) - J^w_\vecM(\vecpi;\vecpi)$, and define $\cE^{w,\NE}_{\vecM}(\vecpi) := \max_{\tvecpi} \Delta^w_\vecM(\tvecpi,\vecpi)$. Our goal in this setting is to find an $\epsilon$-approximate NE policy $\hvecpi^\NE := \{\hvecpi^{w,\NE}\}_{w\in[W]}$ such that:
\begin{align}
    \forall w\in[W],\quad \cE^{w,\NE}_{\vecM^*}(\hvecpi) \leq \epsilon.\label{eq:approx_NE_MT}
\end{align}

\subsection{Conversion from MT-MFG to MFG with Constrained Policy Space}\label{appx:conversion}
Intuitively, the construction is made by integrating the state and action spaces, which will result in a MFG with transition and reward functions following some block diagnoal structure.

Given a MT-MFG $\vecM:=\{(\mu_1^w,H,\cS^w,\cA^w,\mP_\vecM^w,r^w)_{w\in[W]}\}$, we denote the converted MF-MDP by $M_\MFG := \{\mu_1,H,\cS_\MFG,\cA_\MFG,\mP_\MFG, r_\MFG\}$, where we have the extended state space $\cS_\MFG := \bigcup_{w\in[W]}(\cS^w\times\{w\})$ and action space $\cA_\MFG := \bigcup_{w\in[W]}(\cA^w\times\{w\})$. 
As we can see, the new state/action space is the collection of all states/actions agumented by the group index $w\in[W]$.
In this way, states and actions in different groups can be distinguished by the group index $w$.
Next, we construct a new initial distribution $\mu_1 := [\frac{\mu_1^1}{W},\frac{\mu_1^2}{W},...,\frac{\mu_1^W}{W}]$ by concatenating all the initial distributions with normalization.
For the policy, we define 
\begin{align}
\Pi^\dagger := \{\pi|\forall w\in[W],\pi(a^w\circ w|s^w\circ w) = \pi^w(a^w\circ w|s^w\circ w),~\text{for some }\pi^w\in\Pi^w\}, \label{def:constrained_policy_space}
\end{align}
with $\Pi^w:=\{\pi^w:\cS^w\rightarrow\Delta(\cA^w)\}$.
In another word, $\Pi^\dagger$ includes and only includes policies taking actions sharing the same group index with states, and we only consider the policies $\pi \in \Pi^\dagger$.

\subsubsection{Definition of Transition/Reward Functions in the Lifted MF-MDP}\label{appx:def_tran_rew_lifted_MFMDP}
Next, given a density $\mu_h := [\frac{\mu_h^1}{W},...,\frac{\mu_h^W}{W}] \in \Delta(\cS_\MFG)$ with $\mu_h^w\in\Delta(\cS^w)$, the transition and reward functions in the converted MFG is defined by (note that by definition of $\Pi^\dagger$, we only need to consider the case when the state and action share the group index):
\begin{align*}
    \mP_{\MFG,h}(s_{h+1}^{\tw}\circ\tw|s_h^w\circ w,a_h^w\circ w,\mu_h) =& \begin{cases}
        \mP^w_{\vecM,h}(s_{h+1}^w|s_h^w,a_h^w,\mu_h^w),\quad &\text{if}~\tw = w\\        
        0,\quad &\text{otherwise}.
    \end{cases}\\
    r_{\MFG,h}(s_h^w\circ w,a_h^w\circ w,\mu_h) =& r^w_h(s_h^w,a_h^w,\mu_h^w).
\end{align*}
For the sake of rigor, we include the definition for the transition/reward functions on those $s_h^w\circ w$ and $a_h^\tw\circ \tw$ with $w \neq \tw$.
We define $\mP_{\MFG,h}(\cdot|s_h^w\circ w,a_h^\tw\circ \tw,\mu_h)$ to be a uniform distribution over $\cS_\MFG$, and $r_{\MFG,h}(s_h^w\circ w,a_h^\tw\circ \tw,\mu_h) = 0$, for any $\mu_h \in \Delta(\cS_\MFG)$.

After specifying a policy $\pi \in \Pi^\dagger$, denote $\vecpi:=\{\pi^w\}_{w\in[W]}$ to be the MT-MFG policy that $\pi$ corresponds to, we can verify that the state density $\mu^\pi_{M_\MFG,h}\in\Delta(\cS_\MFG)$ evolves according to:
\begin{align}
    \forall h \in [H]:\quad & \mu^\pi_{M_\MFG,h} \gets [\frac{\mu^{1,\vecpi}_{\vecM,h}}{W},...,\frac{\mu^{W,\vecpi}_{\vecM,h}}{W}].\label{eq:density_equivalence}
\end{align}
where recall $\mu^{w,\vecpi}_{\vecM,h}$ denotes the density of type $w$ induced by $\vecpi$ in model $\vecM$. To see this, by induction,
\begin{align*}
    \mu^\pi_{M_\MFG,h}(s^w\circ w) =& \sum_{s_{h-1}\in\cS_\MFG,a_{h-1}\in\cA_\MFG}\mu_{M_\MFG,h}^\pi(s_{h-1})\pi(a_{h-1}|s_{h-1})\mP_{\MFG,h}(s^w_h\circ w|s_{h-1},a_{h-1},\mu^\pi_{M_\MFG,h-1}) \\
    =& \sum_{s^w_{h-1}\in\cS^w,a^w_{h-1}\in\cA^w}\mu^{w,\pi}_{M_\MFG,h-1}(s^w_{h-1}\circ w)\cdot \pi(a^w_{h-1}\circ w|s^w_{h-1}\circ w)\mP^w_{\vecM,h}(s_h^w|s_{h-1}^w,a_{h-1}^w,\mu_{M_\MFG,h-1}^{w,\pi}),\\
    =& \sum_{s^w_{h-1}\in\cS^w,a^w_{h-1}\in\cA^w}\mu^{w,\vecpi}_{\vecM,h-1}(s^w_{h-1})/W\cdot \pi^w(a^w_{h-1}|s^w_{h-1})\mP^w_{\vecM,h}(s_h^w|s_{h-1}^w,a_{h-1}^w,\mu_{M_\MFG,h-1}^{w,\pi}),\\
    = & \mu^{w,\vecpi}_{\vecM,h}(s^w) / W.
\end{align*}
Intuitively, in the converted MFG, following a policy $\pi\in\Pi^\dagger$, if an agent starts from the initial state with index $w$, it will follow a trajectory as if it is generated in the original MT-MFG.
In the following, we will call $M_\MFG$ (or $\vecM$) the corresponding MFG (or MT-MFG) of $\vecM$ (or $M_\MFG$).
\subsection{Assumptions and Additional Definitions}\label{appx:MT_assumps_defs}
Recall the definition of $\{\cM^w\}_{w\in[W]}$ and $\veccM$ discussed in Sec.~\ref{sec:preliminary}, 
In the following, we use $\cM_\MFG$ to denote the model class including MFG models converted from models in $\veccM$ according to the method discussed in Appx.~\ref{appx:conversion}, and denote $M^*_\MFG$ to be the one converted from $\vecM^*$.

We have the following assumptions, which can be regarded as a generalization of Assump.~\ref{assump:realizability},~\ref{assump:Lipschitz} and Def.~\ref{def:collection_process}.
\begin{assumption}[Realizability]\label{assump:realizability_MT}
    The true model $\vecM^* \in \veccM$.
\end{assumption}
\begin{assumption}[Lipschitz Continuity for MT-MFG]\label{assump:Lipschitz_MT}
    For any $\vecM \in \veccM$, and for two arbitrary policies $\vecpi,\tvecpi$
    \begin{align*}
        \forall w\in[W],~\forall h,s^w,a^w,\quad &\|\mP_{\vecM,h}^w(\cdot|s_h^w,a_h^w,\vecmu_{\vecM,h}^\vecpi) - \mP_{\vecM,h}^w(\cdot|s_h^w,a_h^w,\vecmu_{\vecM,h}^\tvecpi)\|_1 \leq \frac{\vecL_T}{W}\|\vecmu_{\vecM,h}^\vecpi - \vecmu_{\vecM,h}^\tvecpi\|_1,\\
        &|r_h^w(s_h^w,a_h^w,\vecmu_{\vecM,h}^\vecpi) - r_h^w(s_h^w,a_h^w,\vecmu_{\vecM,h}^\tvecpi)| \leq \frac{\vecL_T}{W}\|\vecmu_{\vecM,h}^\vecpi - \vecmu_{\vecM,h}^\tvecpi\|_1.
    \end{align*}
\end{assumption}
Here we introduce a normalization factor $W$ given that $\|\vecmu_{\vecM,h}^\vecpi - \vecmu_{\vecM,h}^\tvecpi\|_1 = \sum_{w\in[W]}\|\mu_{\vecM,h}^{w,\vecpi} - \mu_{\vecM,h}^{w,\tvecpi}\|_1$.
\begin{definition}[Trajectory Sampling Model in MT-MFG]\label{def:collection_process_MT}
    The learner can query the sampling model with an arbitrary policy $\vecpi:=\{\pi^1,...,\pi^W\}$, a group index $w$ and another policy $\tvecpi := \{\tpi^1,...,\tpi^W\}$, and receive a trajectory by executing $\tpi^w$ while the transition and reward functions are fixed by $\vecpi$, i.e. $\mP^w_{\vecM^*,h}(\cdot|\cdot,\cdot,\vecmu^{\vecpi}_{\vecM^*,h})$ and $r^w_h(\cdot,\cdot,\vecmu^{\vecpi}_{\vecM^*,h})$.
\end{definition}
Similar to the sampling model in Def.~\ref{def:collection_process}, the model above can be implemented by utilizing the observation of an individually deviating agent with type $w$ following policy $\tpi^w$ while the other agents follows $\vecpi$ in a large Multi-Type MARL system.

Moreover, for learning in the lifted MFGs, note that a sampling model in $M^*_\MFG$ as described in Def.~\ref{def:collection_process} can be implemented by Def.~\ref{def:collection_process_MT}.
To see this, given two policies $\pi,\tpi\in\Pi^\dagger$, which correspond to $\vecpi:=\{\pi^1,...,\pi^W\}$ and $\tvecpi:=\{\tpi^1,...,\tpi^W\}$, respectively, the trajectory can be generated by first uniformly sample $w\in[W]$, and then sample a trajectory with Def.~\ref{def:collection_process_MT} with $\vecpi$ converted from $\pi$, type $w$ and policy $\tpi^w$.

\begin{proposition}
    Given a MT-MFG model class $\veccM$ satisfying Assump.~\ref{assump:Lipschitz_MT}, consider its converted MF-MDP model class $\cM_\MFG$, for any $M\in\cM_\MFG$, and any $\pi,\tpi\in\Pi^\dagger$, we have,
    \begin{align}
        \forall h\in[H],\quad  \|\mP_{M,h}(\cdot|s_h,a_h, \mu_{M,h}^\pi) - \mP_{M,h}(\cdot|s_h,a_h,\mu_{M,h}^{\tpi})\|_1 & \leq \vecL_T \cdot \|\mu_{M,h}^\pi-\mu_{M,h}^\tpi\|_1 \label{eq:Lip_model_H_pi_Cstr} .\\
            |r_h(s_h,a_h,\mu_{M,h}^\pi) - r_h(s_h,a_h,\mu_{M,h}^\tpi)| & \leq \vecL_r \cdot \|\mu_{M,h}^\pi-\mu_{M,h}^\tpi\|_1 \label{eq:Lip_rew_l1_pi_Cstr}
    \end{align}
\end{proposition}
\begin{proof}
    According to the definition in Appx.~\ref{appx:def_tran_rew_lifted_MFMDP}, for those $s_h,a_h$ with different group index, their transition or reward differences will be 0. Therefore, we only need to consider the case when $s_h,a_h$ share the group index.

    As we explained in Eq.~\eqref{eq:density_equivalence}, given $\pi,\tpi\in\Pi^\dagger$, which corresponds to $\vecpi,\tvecpi\in\vecPi$, respectively, we have:
    \begin{align*}
        \|\mu^\pi_{M,h} - \mu^\tpi_{M,h}\|_1 = \frac{1}{W}\sum_{w\in[W]}\|\mu^{w,\vecpi}_{\vecM,h}, \mu^{w,\tvecpi}_{\vecM,h}\|_1,
    \end{align*}
    where $\vecM$ is the corresponding MT-MFG model of $M$. Combining with Assump.~\ref{assump:Lipschitz_MT}, we finish the proof.
\end{proof}

\subsection{Constrained Nash Equilibrium}
\PropNEConversion*
\begin{proof}
    Given any $\pi \in \Pi^\dagger$, we denote its corresponding policy in MT-MFG by $\vecpi := \{\pi^1,...,\pi^W\}$ with $\pi^w:\cS^w\rightarrow\Delta(\cA^w)$ and $\pi(a^w_h\circ w|s^w_h\circ w) = \pi^w(a^w_h|s^w_h)$.
    Conversely, given any $\vecpi:=\{\pi^1,...,\pi^W\}$, we can convert it to a policy in $\Pi^\dagger$, which we denote by $\pi$.
    For $\hpi^\NE_\Cstr$, we denote its correspondence in MT-MFG by $\hvecpi^\NE:=\{\hpi^{\NE,1},...\hpi^{\NE,W}\}$.

    Note that, given any $\pi,\tpi\in\Pi^\dagger$ and their correspondence $\vecpi := \{\pi^1,...\pi^W\}$ and $\tvecpi := \{\tpi^1,...\tpi^W\}$, we have:
    \begin{align*}
        J_{M_\MFG}(\tpi; \pi) = \EE_{\tpi;M_\MFG(\pi)}[\sum_{h=1}^H r_{M_\MFG,h}(s_h,a_h,\mu_{M_\MFG,h}^\pi)] = \frac{1}{W}\sum_{w=1}^W \EE_{\tpi^w;\vecM(\vecpi)}[\sum_{h=1}^H r^w_h(s^w_h,a^w_h,\vecmu^\vecpi_{\vecM,h})].
    \end{align*}
    where recall $r^w$ is the reward in type $w$ in MT-MFG and $\vecmu^\vecpi_{\vecM,h}:=\{\mu^{w,\pi}_{\vecM,h}\}_{w\in[W]}$ is the collection of densities for all groups.
    Consider the case when $\vecpi = \hvecpi^\NE$ and $\tpi^\tw \gets \hpi^{\NE,\tw}$ for all $\tw$ except $\tw = w$, we have:
    \begin{align*}
        \epsilon \geq& J_{M_\MFG}(\tpi;\hpi^\NE_\Cstr) - J_{M_\MFG}(\hpi^\NE_\Cstr;\hpi^\NE_\Cstr)  \\
        =& \frac{1}{W}\Big(\EE_{\tpi^w;\vecM(\hvecpi^\NE)}[\sum_{h=1}^H r^w_h(s_h^w,a_h^w,\vecmu^{\hvecpi^\NE}_{\vecM,h})] - \EE_{\hpi^{\NE,w};\vecM(\hvecpi^\NE)}[\sum_{h=1}^H r^w_h(s_h^w,a_h^w,\vecmu^{\hvecpi^\NE}_{\vecM,h})]\Big)\\
        =& \frac{1}{W} \Big(J^w_\vecM(\tvecpi,\hvecpi^\NE) - J^w_\vecM(\tvecpi,\hvecpi^\NE)\Big).
    \end{align*}
    By repeating such discussion for any $w\in[W]$ and any $\pi^w$, we complete the proof for argument (1).

    On the other hand, given an $\epsilon$-approximate NE $\hvecpi^\NE$ in $\vecM$ and its corresponding $\hpi^\NE_\Cstr$ in $M_\MFG$, for any $\pi \in \Pi^\dagger$ we have:
    \begin{align*}
        &J_{M_\MFG}(\pi; \hpi^\NE_\Cstr) - J_{M_\MFG}(\hpi^\NE_\Cstr; \hpi^\NE_\Cstr) \\
        = & \frac{1}{W}\sum_{w\in[W]} \Big(\EE_{\pi^w;\vecM(\hvecpi^\NE)}[\sum_{h=1}^H r^w_h(s_h^w,a_h^w,\vecmu^{\hvecpi^\NE}_{\vecM,h})] - \EE_{\hpi^{\NE,w};\vecM(\hvecpi^\NE)}[\sum_{h=1}^H r^w_h(s_h^w,a_h^w,\vecmu^{\hvecpi^\NE}_{\vecM,h})]\Big).
    \end{align*}
    To upper bound the RHS, for each $w\in[W]$, we consider an arbitrary policy $\tvecpi$ with with $\tpi^\tw = \hpi^{\NE,\tw}$ for all $\tw$ except $\tpi^w = \pi^w$, we should have:
    \begin{align*}
        &\EE_{\pi^w;\vecM(\hvecpi^\NE)}[\sum_{h=1}^H r^w_h(s_h^w,a_h^w,\vecmu^{\hvecpi^\NE}_{\vecM,h})] - \EE_{\hpi^{\NE,w};\vecM(\hvecpi^\NE)}[\sum_{h=1}^H r^w_h(s_h^w,a_h^w,\vecmu^{\hvecpi^\NE}_{\vecM,h})]= J^w_\vecM(\tvecpi,\hvecpi^\NE) - J^w_\vecM(\tvecpi,\hvecpi^\NE) \leq \epsilon,
    \end{align*}
    By repeating for all $w\in[W]$, we complete the proof for argument (2).
\end{proof}

\paragraph{Existence of Constrained NE Policy}
Before we introduce algorithms finding constrained NE(s) in MFG, we first investigate their existence, which is actually directly implied by Prop.~\ref{prop:NE_conversion}.
\begin{corollary}
    Given $\vecM$ satisfying Lipschitz continuity conditions in Assump.~\ref{assump:Lipschitz_MT}, the MFG $M_\MFG$ converted from $\vecM$ has at least one constrained NE satisfying $\forall \pi\in\Pi^\dagger,~J_{M_\MFG}(\pi,\hpi^\NE_\Cstr) \leq J_{M_\MFG}(\hpi^\NE_\Cstr,\hpi^\NE_\Cstr) + \epsilon$ with $\epsilon = 0$.
\end{corollary}
\begin{proof}
From Prop.~\ref{prop:existence_MT_MFG}, any MT-MFG $\vecM$ satisfying Assump.~\ref{assump:Lipschitz_MT} has at least one NE. 
As implied by Prop.~\ref{prop:NE_conversion} when $\epsilon \rightarrow 0$, any MFG $M_\MFG$ converted from an MT-MFG $\vecM$ with NE(s) should have at least one constrained NE.
Therefore, under Assump.~\ref{assump:Lipschitz_MT}, we can guarantee any model in the converted function class $\cM_\MFG$ has at least one constrained NE.
\end{proof}

\subsection{Algorithm Details}\label{appx:MT_algorithm}
We first generalize some notations in Sec.~\ref{sec:learning_MFG}.
We define the (constrained) conditional distance between models:
\begin{align*}
d^\dagger(M,\tM|\pi) := \max_{\tpi\in\Pi^\dagger}~\max\{&\EE_{\tpi,M(\pi)}[\sum_{h=1}^H\|\mP_{M,h}(\cdot|\cdot,\cdot,\mu^\pi_{M,h})-\mP_{\tM,h}(\cdot|\cdot,\cdot,\mu^\pi_{\tM,h})\|_1],\\
&\EE_{\tpi,\tM(\pi)}[\sum_{h=1}^H \|\mP_{M,h}(\cdot|\cdot,\cdot,\mu^\pi_{M,h}) - \mP_{\tM,h}(\cdot|\cdot,\cdot,\mu^\pi_{\tM,h})\|_1]\}.
\end{align*}
Besides, given a MF-MDP class $\cM$, a model $M\in\cM$, and any policy $\pi$, we define the $\epsilon_0$-neighborhood of $M$ in $\cM$ w.r.t. distance $d^\dagger(\cdot,\cdot|\pi)$ to be: $\cB_{\pi}^{\dagger,\epsilon_0}(M;\cM) := \{M'\in\cM | d^\dagger(M,M'|\pi)\leq \epsilon_0\}$.
The ``Central Model'' of $\cM$ w.r.t. policy $\pi$ and distance $d^\dagger$ is defined to be the model with the largest neighborhood set $M_{\Central}^{\dagger,\epsilon_0}(\pi;\cM) \gets \arg\max_{M\in\cM} |\cB_{\pi}^{\dagger,\epsilon_0}(M;\cM)|$.
When $\epsilon_0$ and $\cM$ is clear from context, we will use $M^{\dagger,\pi}_\Central$ as a short note.

Besides, we define $\cE^{\dagger,\NE}_{M}(\pi) := \max_{\tpi\in\Pi^\dagger}\Delta_{M}(\tpi,\pi) = \max_{\tpi\in\Pi^\dagger} J_M(\tpi,\pi) - J_M(\pi,\pi)$ to be the constrained NE gap.

\begin{algorithm}
    \textbf{Input}: Model Class $\veccM$; Policy Class $\Pi^\dagger$; Accuracy level $\epsilon_0,\teps,\beps$; Confidence level $\delta$\\
    Convert $\veccM$ to $\cM_\MFG$ as described in Appx.~\ref{appx:conversion}; $\cM_\MFG^1 \gets \cM_\MFG$, $\delta_0 \gets \frac{\delta}{\log_2|\cM_\MFG| + 1}$.\\
    \For{$k=1,2,...$}{
        $\pi^k \gets \argmin_{\pi\in\Pi^\dagger} |\cB_{\pi}^{\dagger,\epsilon_0}(M^{\dagger,\pi}_\Central;\cM_\MFG^k)|$;\\
        \lIf{$|\cB_{\pi^k}^{\dagger,\epsilon_0}(M^{\dagger,\pi^k}_\Central;\cM_\MFG^k)| \leq \frac{|\cM_\MFG^k|}{2}$}{\label{line:if_branch_Cstr}
            $\cM_\MFG^{k+1} \gets \texttt{ModelElimCstr}(\pi^k, \cM_\MFG^k, \teps, \delta_0)$.
        }
        \Else{\label{line:else_branch_Cstr}
            $\pi^{\dagger,\NE,k}_{\Bridge} \gets \texttt{BridgePolicyCstr}(\cM_\MFG^k,\beps)$; \\ 
            $\cM_\MFG^{k+1} \gets \texttt{ModelElimCstr}(\pi^{\dagger,\NE,k}_\Bridge, \cM_\MFG^k, \teps, \delta_0)$;\\
            Randomly pick $\tM^k$ from $\cM_\MFG^{k+1}$;\\
            $\cE^{\dagger,\NE}_{\tM^k}(\pi^{\dagger,\NE,k}_\Bridge) \gets \max_{\pi\in\Pi^\dagger} J_{\tM^k}(\pi,\pi^{\dagger,\NE,k}_\Bridge) - J_{\tM^k}(\pi^{\dagger,\NE,k}_\Bridge,\pi^{\dagger,\NE,k}_\Bridge)$; \\
            \lIf{$\cE^{\dagger,\NE}_{\tM^k}(\pi^{\dagger,\NE,k}_\Bridge) \leq \frac{3\epsilon}{4}$}{\label{line:else_if_branch_Cstr}
                \Return{$\pi^{\dagger,\NE,k}_\Bridge$}
            }
        }
        \lIf{$|\cM_\MFG| = 1$}{
            Return the NE of the model in $\cM_\MFG$.
        }
    }
    Return the constrained NE policy of the model in $\cM_\MFG^{k}$.
    \caption{Multi-Type MFG Learning with Constrained Policy Space}\label{alg:learning_with_DCP_Cstr}
\end{algorithm}

\begin{algorithm}
    \textbf{Input}: Reference Policy $\pi$; Policy Class $\Pi^\dagger$; Model Class $\bcM$; Accuracy level $\teps$; Confidence $\delta$\\
    $\bcM^1 \gets \cM$;~$\bar{\epsilon}$; Choosing $T$ according to Thm.~\ref{thm:Elimination_Alg_Cstr}\\
    \For{$t=1,2,...,T$}{
        $\tpi^t, M^t, M'^t \gets \arg\max_{\tpi\in\Pi^\dagger} \max_{M,M' \in \bcM^t}\EE_{\tpi,M(\pi)}[\sum_{h=1}^H \|\mP_{M,h}^w(\cdot|\cdot,\cdot,\mu^\pi_{M,h}) - \mP_{M',h}^w(\cdot|\cdot,\cdot,\mu^\pi_{M',h})\|_1]$. \label{line:argmax_Cstr}\\
        Denote the value taken at the above as $\Delta_{\max}^t$. \\
        \lIf{$\Delta_{\max}^t \leq \teps$}{\label{line:elimination_termination_Cstr}
            \Return $\bcM^t$
        }
        \Else{
            $\cZ^t\gets\{\}$\\
            \For{$h=1,2...,H$}{
                \For{$w\in[W]$}{
                    // \blue{Trajectory sampling in $M_\MFG^*$ can be implemented by Def.~\ref{def:collection_process_MT}.}\\
                    Sample a trajectory with $(\pi,\pi)$, and collect the data at step $h$: $\{(s_h^{w,t},a_h^{w,t},s_{h+1}'^{w,t})\}$.\\
                    Sample a trajectory with $(\tpi^t,\pi)$, and collect the data at step $h$: $\{(\ts_h^{w,t},\ta_h^{w,t},\ts_{h+1}'^{w,t})\}$.\\ 
                    $\cZ^t \gets \cZ^t \cup \{(s_h^{w,t},a_h^{w,t},s_{h+1}'^{w,t})\} \{(\ts_h^{w,t},\ta_h^{w,t},\ts_{h+1}'^{w,t})\}$.
                }
            }
            $\forall M \in \bcM^t$, define 
            $$
            l^\pi_\MLE(M;\cZ^t):=\sum_{k=1}^K\sum_{w=1}^W\sum_{h=1}^H \log \mP^w_{M,h}(s_{h+1}'^{w,t}|s_h^{w,t},a_h^{w,t},\mu^\pi_{M,h}) + \log \mP^w_{M,h}(\ts_{h+1}'^{w,t}|\ts_h^{w,t},\ta_h^{w,t},\mu^\pi_{M,h}).
            $$
            $\bcM^{t+1} \gets \{M\in\bcM^t|~l^\pi_{\MLE}(M;\cZ^t) \geq \max_{\tM}l^\pi_{\MLE}(\tM;\cZ^t) - \log\frac{WHT|\cM|}{\delta}\}$.
        }
    }
    \caption{ModelElimCstr}\label{alg:elimination_DCP_Cstr}
\end{algorithm}

\begin{algorithm}
    \textbf{Input}: MF-MDP model class $\cM$; Policy Space $\Pi^\dagger$; Accuracy Level $\beps$, $\epsilon_0$\\
    Convert Policy-Aware MDP Model Class $\dcM$ from $\cM$ by Eq.~\eqref{eq:conversion_to_dM}.\\
    Construct $\beps$-cover of the policy space $\Pi^\dagger$ w.r.t. $d_{\infty,1}$ distance, denoted as $\Pi_\beps^\dagger$.\\
    \lFor{$\tpi \in \Pi_\beps^\dagger$}{
        Find the central model $\dM_\Central^{\dagger,\epsilon_0}(\tpi;\dcM) \gets \arg\max_{\dM \in \dcM}|\cB^{\dagger,\epsilon_0}_\pi(\dM;\dcM)|$
    }
    Construct the new PAM $\dM_\Bridge$ with transition and reward functions $\forall w\in[W],h\in[H]$:
    \begin{align*}
        \dmP_{\Bridge,h}^w(\cdot|s_h,a_h,\pi) :=& \frac{\sum_{\tpi \in \Pi_\beps^\dagger}[2\beps - d_{\infty,1}(\pi,\tpi)]^+\dmP_{\dM^\tpi,h}^w(\cdot|s_h,a_h,\tpi)}{\sum_{\tpi \in \Pi_\beps^\dagger}[2\beps - d_{\infty,1}(\pi,\tpi)]^+}\\
        \dr_{\Bridge,h}^w(s_h,a_h,\pi) := & \frac{\sum_{\tpi \in \Pi_\beps^\dagger}[2\beps - d_{\infty,1}(\pi,\tpi)]^+\dr_{\dM^\tpi,h}^w(s_h,a_h,\tpi)}{\sum_{\tpi \in \Pi_\beps^\dagger}[2\beps - d_{\infty,1}(\pi,\tpi)]^+}.
    \end{align*}\\
    Find the NE of bridge model: $\pi^{\dagger,\NE}_\Bridge \gets \arg\min_{\pi\in\Pi^\dagger} \max_{\tpi\in\Pi^\dagger}\dJ_{\dM_\Bridge}(\tpi;\pi) - \dJ_{\dM_\Bridge}(\pi;\pi)$.\\
    \Return $\pi^{\dagger,\NE}_\Bridge$.
    \caption{BridgePolicyCstr}\label{alg:BridgePolicy_Cstr}
\end{algorithm}

\newpage
\subsection{Proofs for Algorithm~\ref{alg:learning_with_DCP_Cstr}}\label{appx:MT_proofs}
\begin{theorem}\label{thm:Elimination_Alg_Cstr}
    Under Assump.~\ref{assump:realizability_MT} and~\ref{assump:Lipschitz_MT}, in Alg.~\ref{alg:elimination_DCP_Cstr}, given any $\teps$, reference policy $\pi$, $\delta \in (0,1)$, and $M^*\in\bcM$, by choosing $T = \tilde{O}(\frac{H^4}{\teps^2}(\dimCPE(\cM,\epsilon')\wedge (1+\vecL_T)^{2H}(1+\vecL_T H)^2\dimCPEII(\cM,\epsilon')))$
    with $\epsilon' = O(\frac{\teps}{H^{2}(1+\vecL_T)^{H}})$, , w.p. $1-\delta$,
    the algorithm terminates at some $T_0 \leq T$, and return $\bcM^{T_0}$ satisfying (i) $M^* \in \bcM^{T_0}$ (ii) $\forall M \in \bcM^{T_0}$, $d^\dagger(M^*,M|\pi)\leq \teps$.
\end{theorem}
\begin{proof}
    The proof is the same as Thm.~\ref{thm:Elimination_Alg}, except that we consider the constrained policy space, and need to replace P-MBED with constrained P-MBED.
\end{proof}

\begin{theorem}\label{thm:ConstructionErr_Cstr}
    Suppose we feed Alg.~\ref{alg:BridgePolicy_Cstr} with a model class $\dcM$ and policy space $\Pi^\dagger$, then for the bridge model $\dM_\Bridge$ it computes, by choosing $\beps = \epsilon_0 / \min\{2HL_r \frac{(1+\vecL_T)^H - 1}{\vecL_T}, 2H(H+1)((1 + \vecL_T)^H - 1)\}$, for any reference policy $\pi\in\Pi^\dagger$ and its associated central model $\dM_\Central^{\dagger,\epsilon_0}(\pi;\dcM)$, we have:
    \begin{align*}
        \max_{\tpi\in\Pi^\dagger}\EE_{\tpi,\dM_\Central^{\dagger,\epsilon_0}(\pi;\dcM)(\pi)}[\sum_{h=1}^H \|\dmP_{\dM_\Central^{\dagger,\epsilon_0}(\pi;\dcM),h}(\cdot|s_h,a_h,\pi) - \dmP_{\Bridge,h}(\cdot|s_h,a_h,\pi)\|_1] \leq& (H+3)\epsilon_0, \\
        \max_{\tpi\Pi^\dagger}\EE_{\tpi,\dM_\Central^{\dagger,\epsilon_0}(\pi;\dcM)(\pi)}[\sum_{h=1}^H |r_{\dM_\Central^{\dagger,\epsilon_0}(\pi;\dcM),h}(s_h,a_h,\pi) - \dr_{\Bridge,h}(s_h,a_h,\pi)|] \leq& L_rH(H+4)\epsilon_0.
    \end{align*}
\end{theorem}
\begin{proof}
    The proof is the same as Thm.~\ref{thm:ConstructionErr} except that we constrain the policies in $\Pi^\dagger$.
\end{proof}

\begin{lemma}\label{lem:BP_close_NE_CM_Cstr}
    Suppose the \texttt{Else}-branch in Line~\ref{line:else_branch} if activated in Alg.~\ref{alg:elimination_DCP_formal}, for policy $\pi^{\dagger,\NE,k}_\Bridge$ and its corresponding central model $M^{\dagger,k}_\Central := \argmax_{M\in\cM^k}|\cB^{\dagger,\epsilon_0}_{\pi^{\dagger,\NE,k}_\Bridge}(M;\cM^k)|$, we have:
    \begin{align*}
        \cE^{\dagger,\NE}_{M^{\dagger,k}_\Central}(\pi^{\dagger,\NE,k}_\Bridge) := \max_{\pi\in\Pi^\dagger} \Delta_{M^{\dagger,k}_\Central}(\pi,\pi^{\dagger,\NE,k}_\Bridge) \leq 2(1+L_r)(H+4)\epsilon_0.
    \end{align*}
\end{lemma}
\begin{proof}
    The proof is the almost the same as Lem.~\ref{lem:BP_close_NE_CM}, except that we consider the constrained policy space.

    For any policy $\pi\in\Pi^\dagger$, we have
    \begin{align*}
        &\Delta_{M^{\dagger,k}_\Central}(\pi, \pi^{\dagger,\NE,k}_\Bridge) \\
        \leq & \Delta_{\dM^{\dagger,k}_\Central}(\pi, \pi^{\dagger,\NE,k}_\Bridge) - \Delta_{\dM_\Bridge}(\pi, \pi^{\dagger,\NE,k}_\Bridge) \tag{$\Delta_{\dM_\Bridge}(\pi, \pi^{\dagger,\NE,k}_\Bridge) \leq 0$}\\
        \leq & |\dJ_{\dM^{\dagger,k}_\Central}(\pi, \pi^{\dagger,\NE,k}_\Bridge) - \dJ_{\dM_\Bridge}(\pi, \pi^{\dagger,\NE,k}_\Bridge)| + |\dJ_{\dM^{\dagger,k}_\Central}(\pi^{\dagger,\NE,k}_\Bridge, \pi^{\dagger,\NE,k}_\Bridge) - \dJ_{\dM_\Bridge}(\pi^{\dagger,\NE,k}_\Bridge, \pi^{\dagger,\NE,k}_\Bridge)| \\
        \leq & \EE_{\pi^{\dagger,\NE,k}_\Bridge,\dM^{\dagger,k}_\Central(\pi^{\dagger,\NE,k}_\Bridge)}[\sum_{h=1}^H |\dr_{\dM^{\dagger,k}_\Central,h}(s_h,a_h,\pi^{\dagger,\NE,k}_\Bridge) - \dr_{\dM_\Bridge,h}(s_h,a_h,\pi^{\dagger,\NE,k}_\Bridge)| \\
        & \qquad\qquad\qquad\qquad\qquad\qquad + \|\dmP_{\dM^{\dagger,k}_\Central,h}(\cdot|s_h,a_h,\pi^{\dagger,\NE,k}_\Bridge), \dmP_{\dM_\Bridge,h}(\cdot|s_h,a_h,\pi^{\dagger,\NE,k}_\Bridge)\|_1]\\
        & + \EE_{\pi,\dM^{\dagger,k}_\Central(\pi^{\dagger,\NE,k}_\Bridge)}[\sum_{h=1}^H |\dr_{\dM^{\dagger,k}_\Central,h}(s_h,a_h,\pi^{\dagger,\NE,k}_\Bridge) - \dr_{\dM_\Bridge,h}(s_h,a_h,\pi^{\dagger,\NE,k}_\Bridge)| \\
        & \qquad\qquad\qquad\qquad\qquad\qquad + \|\dmP_{\dM^{\dagger,k}_\Central,h}(\cdot|s_h,a_h,\pi^{\dagger,\NE,k}_\Bridge), \dmP_{\dM_\Bridge,h}(\cdot|s_h,a_h,\pi^{\dagger,\NE,k}_\Bridge)\|_1]\\
        \leq & 2\max_{\pi\in\Pi^\dagger} \EE_{\pi,\dM^{\dagger,k}_\Central(\pi^{\dagger,\NE,k}_\Bridge)}[\sum_{h=1}^H |\dr_{\dM^{\dagger,k}_\Central,h}(s_h,a_h,\pi^{\dagger,\NE,k}_\Bridge) - \dr_{\dM_\Bridge,h}(s_h,a_h,\pi^{\dagger,\NE,k}_\Bridge)|]\\
        & + 2\max_{\pi\in\Pi^\dagger} \EE_{\pi,\dM^{\dagger,k}_\Central(\pi^{\dagger,\NE,k}_\Bridge)}[\sum_{h=1}^H \|\dmP_{\dM^{\dagger,k}_\Central,h}(\cdot|s_h,a_h,\pi^{\dagger,\NE,k}_\Bridge), \dmP_{\dM_\Bridge,h}(\cdot|s_h,a_h,\pi^{\dagger,\NE,k}_\Bridge)\|_1]\\
        \leq & 2(1+L_r H)(H+4)\epsilon_0.\tag{Thm.~\ref{thm:ConstructionErr_Cstr}}
    \end{align*}
    which finishes the proof.
\end{proof}

\begin{theorem}\label{thm:if_else_branches_Cstr}
    In Alg.~\ref{alg:learning_with_DCP_Cstr}, by choosing $\epsilon_0 = \frac{\epsilon}{8(H+4)(1+L_rH)}$, $\teps = \frac{\epsilon_0}{6}$ and choosing $\beps$ according to Thm.~\ref{thm:ConstructionErr_Cstr}, on the good events in Thm.~\ref{thm:Elimination_Alg_Cstr}, 
        (1) if the \texttt{If-Branch} in Line~\ref{line:if_branch_Cstr} is activated: we have $|\cM^{k+1}| \leq |\cM^k|/2$;
        (2) otherwise, in the \texttt{Else-Branch} in Line~\ref{line:else_branch_Cstr}: either we return the $\pi^{\dagger,\NE,k}_\Bridge$ which is an $\epsilon$-approximate NE for $M^*$; or the algorithm continues with $|\cM^{k+1}| \leq |\cM^k|/2$.
\end{theorem}
\begin{proof}
    We separately discuss the if and else branches in the algorithm.
    \paragraph{Proof for \texttt{If-Branch} in Line~\ref{line:if_branch}}
    On the events in Thm.~\ref{thm:Elimination_Alg_Cstr}, for any $\tM \not\in \cB^{\dagger,\epsilon_0}_{\pi^k}(M^*;\cM^k)$, we have $d^\dagger(M^*,\tM) \geq \epsilon_0 > \teps$, which implies $\tM \not\in \cM^{k+1}$.
    Combining the condition of \texttt{If-Branch}, we have:
    \begin{align*}
        |\cM^{k+1}| \leq |\cB^{\dagger,\epsilon_0}_{\pi^k}(M^*; \cM^k)| \leq \frac{|\cM^k|}{2}.
    \end{align*}
    \paragraph{Proof for \texttt{Else-Branch} in Line~\ref{line:else_branch}}
    First of all, on the events in Thm.~\ref{thm:Elimination_Alg_Cstr}, we have $d^\dagger(M^*,\tM^k|\pi^{\dagger,\NE,k}_\Bridge) \leq \teps$. By applying Lem.~\ref{lem:exploitability_diff}, it implies:
    \begin{align*}
        & |\Delta_{M^*}(\pi, \pi^{\dagger,\NE,k}_\Bridge) - \Delta_{\tM^k}(\pi, \pi^{\dagger,\NE,k}_\Bridge)| \\
        \leq & \EE_{\pi,M^*(\pi^{\dagger,\NE,k}_\Bridge)}[\sum_{h=1}^H \|\mP_{M^*,h}(\cdot|s_h,a_h,\mu^{\pi^{\dagger,\NE,k}_\Bridge}_{M^*,h}) - \mP_{\tM^k,h}(\cdot|s_h,a_h,\mu^{\pi^{\dagger,\NE,k}_\Bridge}_{\tM^k,h})\|_1]\\
        & + (2L_rH + 1) \EE_{\pi^{\dagger,\NE,k}_\Bridge,M^*(\pi^{\dagger,\NE,k}_\Bridge)}[\sum_{h=1}^H  \|\mP_{M^*,h}(\cdot|s_h,a_h,\mu^{\pi^{\dagger,\NE,k}_\Bridge}_{M^*,h}) - \mP_{\tM^k,h}(\cdot|s_h,a_h,\mu^{\pi^{\dagger,\NE,k}_\Bridge}_{\tM^k,h})\|_1] \\
        \leq & 2(L_rH + 1) \teps. \numberthis\label{eq:ub_delta_diff}
    \end{align*}
    Also note that:
    \begin{align*}
        \Delta_{M^*}(\pi, \pi^{\dagger,\NE,k}_\Bridge) =& \Delta_{M^*}(\pi, \pi^{\dagger,\NE,k}_\Bridge) - \Delta_{\tM^k}(\pi, \pi^{\dagger,\NE,k}_\Bridge) + \Delta_{\tM^k}(\pi, \pi^{\dagger,\NE,k}_\Bridge) .
    \end{align*}
    In the following, we separately discuss two cases.
    \paragraph{Case 1: $\cE^{\dagger,\NE}_{\tM^k}(\pi^{\dagger,\NE,k}_\Bridge) \leq \frac{3\epsilon}{4}$ and Line~\ref{line:else_if_branch} is activated}
    Given that $\teps \leq \frac{\epsilon}{16(1+L_rH)}$:
    \begin{align*}
        \forall \pi\in\Pi^\dagger,\quad \Delta_{M^*}(\pi, \pi^{\dagger,\NE,k}_\Bridge)\leq & |\Delta_{M^*}(\pi, \pi^{\dagger,\NE,k}_\Bridge) - \Delta_{\tM^k}(\pi, \pi^{\dagger,\NE,k}_\Bridge)| + \cE^{\dagger,\NE}_{\tM^k}(\pi^{\dagger,\NE,k}_\Bridge) \tag{$\cE^{\dagger,\NE}_{\tM^k}(\pi^{\dagger,\NE,k}_\Bridge) = \max_{\pi\in\Pi^\dagger} \Delta_{\tM^k}(\pi, \pi^{\dagger,\NE,k}_\Bridge)$}\\
        \leq & 2(L_rH + 1)\teps + \frac{3\epsilon}{4} \leq \epsilon.
    \end{align*}
    which implies $\pi^{\dagger,\NE,k}_\Bridge$ is an $\epsilon$-NE of $M^*$.
    \paragraph{Case 2: $\cE^{\dagger,\NE}_{\tM^k}(\pi^{\dagger,\NE,k}_\Bridge) > \frac{3\epsilon}{4}$ and Line~\ref{line:else_if_branch} is not activated}
    As a result, for any policy $\pi\in\Pi^\dagger$, by Eq.~\eqref{eq:ub_delta_diff}, we have:
    \begin{align*}
        \Delta_{M^*}(\pi, \pi^{\dagger,\NE,k}_\Bridge) \geq - |\Delta_{M^*}(\pi, \pi^{\dagger,\NE,k}_\Bridge) - \Delta_{\tM^k}(\pi, \pi^{\dagger,\NE,k}_\Bridge)| + \Delta_{\tM^k}(\pi, \pi^{\dagger,\NE,k}_\Bridge) \geq \Delta_{\tM^k}(\pi, \pi^{\dagger,\NE,k}_\Bridge) - 2(L_rH + 1)\teps.
    \end{align*}
    Therefore, by our choice of $\teps$,
    \begin{align*}
        \max_{\pi\in\Pi^\dagger} \Delta_{M^*}(\pi, \pi^{\dagger,\NE,k}_\Bridge) \geq \cE^{\dagger,\NE}_{\tM^k}(\pi^{\dagger,\NE,k}_\Bridge) - 2(L_rH + 1)\teps \geq \frac{5\epsilon}{8}.
    \end{align*}
    On the other hand, by Lem.~\ref{lem:BP_close_NE_CM}, for any $\pi\in\Pi^\dagger$, we have:
    \begin{align*}
        & \Delta_{M^*}(\pi, \pi^{\dagger,\NE,k}_\Bridge) - 2(1+L_rH)(H+4)\epsilon_0 \\
        \leq &|\Delta_{M^*}(\pi, \pi^{\dagger,\NE,k}_\Bridge)| - |\Delta_{M^{\dagger,k}_\Central}(\pi, \pi^{\dagger,\NE,k}_\Bridge)| \tag{Here we apply Lem.~\ref{lem:BP_close_NE_CM_Cstr}}\\
        \leq & |\Delta_{M^*}(\pi, \pi^{\dagger,\NE,k}_\Bridge) - \Delta_{M^{\dagger,k}_\Central}(\pi, \pi^{\dagger,\NE,k}_\Bridge)| \\
        \leq & \EE_{\pi,M^*(\pi^{\dagger,\NE,k}_\Bridge)}[\sum_{h=1}^H \|\mP_{M^*,h}(\cdot|s_h,a_h,\mu^{\pi^{\dagger,\NE,k}_\Bridge}_{M^*,h}) - \mP_{M^{\dagger,k}_\Central,h}(\cdot|s_h,a_h,\mu^{\pi^{\dagger,\NE,k}_\Bridge}_{\tM,h})\|_1]\\
        & + (2L_rH + 1) \EE_{\pi^{\dagger,\NE,k}_\Bridge,M^*(\pi^{\dagger,\NE,k}_\Bridge)}[\sum_{h=1}^H  \|\mP_{M^*,h}(\cdot|s_h,a_h,\mu^{\pi^{\dagger,\NE,k}_\Bridge}_{M^*,h}) - \mP_{M^{\dagger,k}_\Central,h}(\cdot|s_h,a_h,\mu^{\pi^{\dagger,\NE,k}_\Bridge}_{M^{\dagger,k}_\Central,h})\|_1] \\
        \leq & (2L_r H + 2) d^\dagger(M^*,M^{\dagger,k}_\Central|\pi^{\dagger,\NE,k}_\Bridge).
    \end{align*}
    According to the choice of $\epsilon_0$, we have $2(1+L_rH)(H+4)\epsilon_0 \leq \frac{\epsilon}{4}$, therefore,
    \begin{align*}
        d^\dagger(M^*,M^{\dagger,k}_\Central|\pi^{\dagger,\NE,k}_\Bridge) \geq \frac{1}{2(L_r H + 1)}\Big(\max_{\pi\in\Pi^\dagger} |\Delta_{M^*}(\pi, \pi^{\dagger,\NE,k}_\Bridge)| - 2(1+L_rH)(H+4)\epsilon_0\Big) \geq \frac{3\epsilon}{16(L_r H + 1)}.
    \end{align*}
    Next we try to show that models in $\cB^{\epsilon_0}_{\pi^{\dagger,\NE,k}_\Bridge}(M^{\dagger,k}_\Central, \cM^k)$ will be eliminated.
    For any $M \in \cB^{\epsilon_0}_{\pi^{\dagger,\NE,k}_\Bridge}(M^{\dagger,k}_\Central, \cM^k)$, we have:
    \begin{align*}
        d^\dagger(M,M^*|\pi^{\dagger,\NE,k}_\Bridge) \geq d^\dagger(M^{\dagger,k}_\Central,M^*|\pi^{\dagger,\NE,k}_\Bridge) - d^\dagger(M,M^{\dagger,k}_\Central|\pi^{\dagger,\NE,k}_\Bridge) \geq \frac{3\epsilon}{16(L_r H + 1)} - \epsilon > \teps.
    \end{align*}
    On the good events in Thm.~\ref{thm:Elimination_Alg_Cstr}, we have $M \not\in \cM^{k+1}$, which implies
    $$
        |\cM^{k+1}| \leq |\cM^k| - |\cB^{\epsilon_0}_{\pi^{\dagger,\NE,k}_\Bridge}(M^{\dagger,k}_\Central, \cM^k)| \leq |\cM^k| / 2.
    $$
\end{proof}

\begin{restatable}{theorem}{ThmMFGCstr}[Sample Complexity in MT-MFG]\label{thm:sample_complexity_MT_MFG}
    Under Assump.~\ref{assump:realizability_MT} and~\ref{assump:Lipschitz_MT}, by running Alg.~\ref{alg:learning_with_DCP_Cstr} with Alg.~\ref{alg:elimination_DCP_Cstr} as \texttt{ModelElimCstr} and Alg.~\ref{alg:BridgePolicy_Cstr} as \texttt{BridgePolicyCstr}, and hyper-parameter choices according to Thm.~\ref{thm:Elimination_Alg_Cstr},~\ref{thm:ConstructionErr_Cstr}, and~\ref{thm:if_else_branches_Cstr}, w.p. $1-\delta$, Alg.~\ref{alg:learning_with_DCP_Cstr} will terminate at some $k \leq \log_2|\veccM| + 1$ and return an $\epsilon$-NE of $\vecM^*$. The number of trajectories consumed is $\tilde{O}(\frac{W^2H^7}{\epsilon^2}(1+\vecL_rH)^2\sum_{w\in[W]}(\dimMTPE(\cM^w,\epsilon') \wedge H^3(1+\vecL_T)^{2H}(1+\vecL_T H)^2\dimMTPE^\RII(\cM^w,\epsilon'))\log^2\frac{|\veccM|}{\delta})$, where $\epsilon'=O(\epsilon/WH^3(1+\vecL_rH)(1+\vecL_T)^H)$, $\dimMTPE(\cM^w,\epsilon')$ and $\dimMTPE^\RII(\cM^w,\epsilon')$ are the Multi-Type P-MBED defined in Def.~\ref{def:P_MBED_MT}, and we omit the logarithmic terms of $H,\epsilon,\log|\veccM|,\dimMTPE,1+\vecL_T$ and $1 + \vecL_r$.
\end{restatable}
\begin{proof}
    As a result of Thm.~\ref{thm:if_else_branches_Cstr}, w.p. $1-\frac{\delta}{\log_2|\cM_\MFG|+1}\cdot(\log_2|\cM_\MFG|+1) = 1-\delta$, there exists a step $k \leq \log_2|\cM_\MFG| + 1 = \log_2|\veccM| + 1$ such that Alg.~\ref{alg:learning_with_DCP_Cstr} will terminate the return us an $\frac{\epsilon}{W}$-approximate NE of $M^*_\MFG$. The total number of trajectories required is:
    \begin{align*}
        (\log_2|\veccM| + 1) \cdot T \cdot 2H = \tilde{O}(\frac{H^5}{\teps^2}\Big(\dimCPE(\cM_\MFG,\epsilon') \wedge H^3(1+\vecL_T)^{2H}(1+\vecL_T H)^2\dimCPEII(\cM_\MFG,\epsilon')\Big))
    \end{align*}
    Note that in Thm.~\ref{thm:Elimination_Alg_Cstr}, we choose $\teps = \frac{\epsilon_0}{6} = O(\frac{\epsilon}{WH(1+L_rH)})$, and $\epsilon'=O(\teps/H^2(1+L_T)^H) = O(\epsilon/WH^3(1+L_rH)(1+L_T)^H)$.
    Combining with the above discussion and Prop.~\ref{prop:conversion_CP2P} and Prop.~\ref{prop:NE_conversion}, we finish the proof.
\end{proof}

\newpage
\section{Approximation Ability of Multi-Type MFGs}\label{appx:approx_MT_MFG}
\subsection{Multi-Type Symmetric Anonymous Games}\label{appx:MT_SAG}
\paragraph{Notations}
Given a multi-agent system where agents are divided into $W$ groups, where for each type $w$ the agents share the state-action spaces $\cS^w,\cA^w$ and initial distribution $\mu^w_1$, we use $N^w$ to denote the number of agents in group $w\in[W]$, and $s^{w,n}_h,a^{w,n}_h$ and $\pi^{w,n}$ to denote the state, action, and policy for the $n$-th agent in type $w$, respectively.
Besides, we define $\vecs_h := \{s^{w,n}_h\}_{w\in[W],n\in[N^w]}, \veca_h := \{a^{w,n}_h\}_{w\in[W],n\in[N^w]}$ to be the collection of states and actions of all agents in the system at step $h$, and denote $p_{\vecs_h} := \{p^1_{\vecs_h},...,p^W_{\vecs_h}\}$ to be the empirical distribution of the agents' states with:
\begin{align*}
    p^w_{\vecs_h}\in\mR^{|\cS^w|}:\quad p^w_{\vecs_h}(\cdot) = \frac{1}{N^w}\sum_{n=1}^{N^w}\delta(s_h^{w,n} = \cdot),
\end{align*}
To distinguish the policy in MFG setting, we use $\tvecnu := \{\pi^{w,n}\}_{w\in[W],n\in[N^w]}$ to denote the collection of policies.
We will denote $\vecnu(\veca_h|\vecs_h) := \prod_{w\in[W]}\prod_{n\in[N^w]}\pi^{w,n}(a^{w,n}_h|s^{w,n}_h)$.
Besides, we use $\vecnu^{-(w,n)}\circ\tpi^{w,n}$ to denote the policy replacing $\pi^{w,n}$ to $\tpi^{w,n}$ while keeping the others fixed.

\begin{definition}[Multi-Type Symmetric Anonymous Game]\label{def:MT_SAG}
    The Multi-Type Symmetric Anonymous Game (MT-SAG) $\bvecM:=\{(\mu_1^w,\cS^w,\cA^w,H,\mP^w,r^w)\}_{w\in[H]}$ is a Multi-Agent system consists of $W$ groups. Given a policy $\vecpi$, the system evolves as:
    \begin{align}
        s^{w,n}_1\sim\mu^w_1;~\forall h,w,n:&\quad a^{w,n}_h\sim\pi^{w,n}_h(\cdot|s^{w,n}_h),~ r^{w,n}_h \gets r^w_h(s_h^{w,n},a_h^{w,n},p_{\vecs_h}),~s^{w,n}_{h+1} \sim \mP^w_h(\cdot|s_h^w,a_h^w,p_{\vecs_h}).\label{eq:evol_MT_SAG}
    \end{align}
    Given a policy $\vecnu$, we define the value functions $V:\cS\rightarrow[0,1]$ and $Q:\cS\times\cA\rightarrow[0,1]$ of the $(w,n)$-th agent conditioning on the system state $\vecs_h$ to be:
    \begin{align}
        V^{w,n,\vecnu}_{\bvecM,h}(s^{w,n}_h;\vecs_h) := \EE_{\bvecM,\vecnu}[\sum_{\ph=h}^H r^w_\ph(s_\ph^{w,n},a_\ph^{w,n},p_{\vecs_\ph})|\vecs_h];
        \label{eq:value_MT_SAG}
    \end{align}
    where the expectation is taken over the evolution process in Eq.~\eqref{eq:evol_MT_SAG}. Besides, we define the total value starting from the initial states $J^{w,n}_\bvecM(\vecnu) := \EE_{\vecmu_1}[V_1^{w,n,\vecnu}(s_1^{w,n};\vecs_1)]$.

    A policy $\vecnu$ is called to be the NE policy in MT-SAG if any agent can not improve its value by deviating from its current policy while the others' are fixed,
    \begin{align*}
        \forall w,n,\quad \max_{\tpi^{w,n}}J^{w,n}_\bvecM(\vecnu^{-(w,n)}\circ\tpi^{w,n}) \leq J^{w,n}_\bvecM(\vecnu).
    \end{align*}
    and a policy $\vecnu'$ is called to be an $\epsilon$-approximate NE in MT-SAG if 
    \begin{align*}
        \forall w,n,\quad \max_{\tpi^{w,n}}J^{w,n}_\bvecM(\vecnu^{-(w,n)}\circ\tpi^{w,n}) \leq J^{w,n}_\bvecM(\vecnu) + \epsilon.
    \end{align*}
    \begin{assumption}[Lipschitz Continuity in MT-SAG]\label{assump:Lip_MT_SAG}
        We assume the transition and reward functions of MT-SAG are Lipschitz continuous w.r.t. the density, s.t. $\forall w\in[W],h\in[H],~\quad \forall \hvecmu_h, \hvecmu_h'\in\Delta(\cS^1)\times...\Delta(\cS^W)$
        \begin{align*}
            &\|\mP^w_h(\cdot|s_h^w,a_h^w,\hvecmu_h) - \mP^w_h(\cdot|s_h^w,a_h^w,\hvecmu_h')\|_1 \leq \frac{\vecL_T}{W}\|\hvecmu_h - \hvecmu_h'\|_1\\
            & |r^w_h(s_h^w,a_h^w,\hvecmu_h) - r^w_h(s_h^w,a_h^w,\hvecmu_h')\|_1 \leq \frac{\vecL_r}{W}\|\hvecmu_h - \hvecmu_h'\|_1.
        \end{align*}
    \end{assumption}
\end{definition}
\subsection{Approximating MT-SAGs via MT-MFGs}
\begin{definition}[Multi-Type Mean-Field Game Approximation of MT-SAG]\label{def:MT_MFA}
    Given an MT-SAG $\bvecM$, its Multi-Type Mean-Field (MT-MFG) Approximation is a model Multi-Type MF-MDP model $\vecM:=\{(\mu_1^w,\cS^w,\cA^w,H,\mP^w_\vecM,r^w_\vecM)\}_{w\in[W]}$, sharing the group, initial distribution, state-action spaces and transition $\mP^w$ and reward function $r^w$ as MT-SAG (i.e. $\mP^w_\vecM(\cdot|\cdot,\cdot,\cdot) = \mP^w_\bvecM(\cdot|\cdot,\cdot,\cdot), r^w_\vecM(\cdot,\cdot,\cdot) = r^w_\bvecM(\cdot,\cdot,\cdot)$), by have different transition rules.

    Next, we describe ``the different transition rules'' in MT-MFG.
    For simplicity of notation, in the following, we omit $\vecM$ or $\bvecM$ in the sub-scription of transition and reward functions.
    Given a reference policy $\vecpi:=\{\pi^1,...,\pi^W\}$ consisting of $W$ policies shared by each group, the density $\vecmu^\vecpi_h:=\{\mu_{h}^{1,\vecpi},...,\mu_{h}^{W,\vecpi}\}$ at step $h$ is defined by:
    \begin{align*}
        \mu^{w,\vecpi}_1 = \mu^w_1,\quad \forall h\geq 1,\quad \mu^{w,\vecpi}_{h+1} \gets \Gamma^{w,\vecpi}_h(\vecmu^\vecpi_{h}),~\text{with}~\Gamma^{w,\vecpi}_h(\vecmu_{h})(\cdot):=\sum_{s_h^w,a_h^w}\mu^{w}_h(s_h^w)\pi^w_h(a_h^w|s_h^w)\mP_h^w(\cdot|s_h^w,a_h^w,\vecmu_h).
    \end{align*}
    where $\Gamma^{w,\vecpi}_h$ is a mapping from $\Delta(\cS^1)\times...\Delta(\cS^W)$ to $\Delta(\cS^w)$.
    The evolution process of the $(w,n)$ agent in type $w$ following a deviation policy $\tpi^{w}$ conditioning on reference policy $\vecpi$ is specified by:
    \begin{align}
        s^{w,n}_1\sim\mu^w_1;~\forall h:&\quad a^{w,n}_h\sim\tpi^{w}_h(\cdot|s^{w,n}_h),~ r^{w,n}_h \gets r^w_h(s_h^{w,n},a_h^{w,n},\vecmu^\vecpi_h),~s^{w,n}_{h+1} \sim \mP^w_h(\cdot|s_h^w,a_h^w,\vecmu^\vecpi_h).\label{eq:evol_MT_MFG}
    \end{align}
    Comparing with Eq.~\eqref{eq:evol_MT_SAG}, the evolution of agents' states is depend on the density when $\forall w\in[W],~N^w\rightarrow +\infty$, instead of the empirical one in practice.

    Recall that given a reference policy $\vecpi$ and a deviation policy $\tpi^{w}$, the value functions $V:\cS\rightarrow[0,1]$ of the $(w,n)$-th agent conditioning on the density $\vecmu_{\vecM,h}^\vecpi$ is defined to be:
    \begin{align*}
        V^{w,n,\vecpi^{-w}\circ\tpi^{w}}_{\vecM,h}(s_h^{w,n};\vecmu^\vecpi_{\vecM,h}) :=& \EE_{\tpi^{w},\vecM(\vecpi)}[\sum_{\ph=h}^H r^w_\ph(s_\ph^{w,n},a_\ph^{w,n},\vecmu^\vecpi_{\vecM,\ph})];
    \end{align*}
    where the expectation is taken over the process in Eq.~\eqref{eq:evol_MT_MFG}. Then, we define the total value of $\vecpi^{-w}\circ \tpi^w$ given the reference policy $\vecpi$ to be:
    \begin{align*}
        J^{w,n}_\vecM(\vecpi^{-w}\circ \tpi^w;\vecpi) := \EE_{s_1^w\sim\mu_1^w}[V^{w,n,\tpi^{w}}_1(s_1^w,\vecmu_1)].
    \end{align*}
    A $\vecpi$ is called NE policy if:
    \begin{align*}
        \forall w,n,\forall \tpi^{w},\quad J^{w,n}_\vecM(\vecpi^{-w}\circ \tpi^w;\vecpi) \leq J^{w,n}_\vecM(\vecpi;\vecpi).
    \end{align*}
\end{definition}
\begin{proposition}[Approximation Error of MT-MFG]\label{prop:formal_MT_MFG_approx}
    Given a Multi-Type Symmetric Anonymous Game (MT-SAG) $\bvecM$, as defined in Def.~\ref{def:MT_SAG}, and its Multi-Type MFG approximation (MT-MFG) $\vecM$, as defined in Def.~\ref{def:MT_MFA}, suppose $\vecpi := \{\pi^w\}_{w\in[W]}$ is the NE policy of MT-MFG, then for any $\epsilon_0 > 0$, the lifted policy $\vecnu:=\{\pi^{w,n}\}_{w\in[W],n\in[N^w]}$ with $\pi^{w,n} = \pi^w,~\forall n\in[N^w]$ is an $\epsilon_0$-approximate NE of MT-SAG if
    \begin{align*}
        \forall w\in[W],\quad N^w \geq O((\vecL_r + \vecL_T)^2W^2H^3\Big(\frac{(1 + \vecL_T)^H - 1}{\vecL_T}\Big)^2\frac{\sum_{w\in[W]}S^w}{\epsilon^2_0}\log\frac{2(1+\vecL_r+\vecL_T)HWS_{\max}}{\epsilon_0}).
    \end{align*}
    where $S_{\max} := \max_{w}S^w$.
\end{proposition}
\begin{proof}
Given $\vecpi := \{\pi^1,...,\pi^W\}$, we denote $\vecnu:=\{\pi^{w,n}\}_{w\in[W],n\in[N^w]}$ to be the lifted policy such that $\nu^{w,n}\gets \pi^w$ for all $w$ and $n\in[N^w]$.
Given a deviation policy $\tpi^w$ for some $w$ ($\tpi^w$ may equal $\pi^w$), we define $\tvecnu:=\{\pi^{w,n}\}_{w\in[W],n\in[N^w]}$ to be a policy in MT-SAG, such that $\tpi^{\tw,n} \gets \pi^\tw$ for all agent except the $(w,1)$-th agent (i.e. the first agent in type $w$), we set $\tpi^{(w,1)}\gets\tpi^w$.

\paragraph{Concentration Events}
We first provide a high-probability bound for the distance between state density $\vecmu^\vecpi_h$ in MT-MFG and the empirical distribution $p_{\vecs_h}^\tvecnu$ in MT-SAG w.r.t. the lifted policy $\tvecnu$.

We use $\Gamma^{w,\vecpi}_{\vecM,h}(\cdot)$ to denote the operator $\Gamma^{w,\vecpi}_h(\cdot)$ in Eq.~\eqref{eq:evol_MT_MFG} specified in model $\vecM$. 
We extend its definition to $h=0$ by $\forall \vecmu_0,\Gamma^{w,\tvecnu}_0(\vecmu_0) \gets \mu^w_1$, and define $\Gamma^{\vecpi}_{\vecM,h}(\cdot) := \{\Gamma^{1,\vecpi}_{\vecM,h}(\cdot),...,\Gamma^{W,\vecpi}_{\vecM,h}(\cdot)\}$. Then, conditioning on $p_{\vecs_{h-1}}^\tvecnu$, we have:
\begin{align*}
    &\|\EE_{\bvecM,\tvecnu}[p_{\vecs_h}^\tvecnu|p_{\vecs_{h-1}}^\tvecnu] - \Gamma^{\vecpi}_{\vecM,h-1}(p_{\vecs_{h-1}}^\tvecnu)\|_1 \\
    =& \|\EE_{\bvecM,\tvecnu}[p_{\vecs_h}^{w,\tvecnu}|p_{\vecs_{h-1}}^\tvecnu] - \Gamma^{w,\vecpi}_{\vecM,h-1}(p_{\vecs_{h-1}}^\tvecnu)\|_1 \\
    =&\frac{N^w - 1}{N^w}\|\EE_{\bvecM,\tvecnu}[\frac{1}{N^w}\sum_{n=2}^{N^w-1}\bm{\delta}_{s^{w,n}_h = (\cdot)}|p_{\vecs_{h-1}}^\tvecnu] - \Gamma^{w,\vecpi}_{\vecM,h-1}(p_{\vecs_{h-1}}^\tvecnu)\|_1 \\
    & + \frac{1}{N^w}\|\EE_{\bvecM,\tvecnu}[\bm{\delta}_{s^{w,1}_h = (\cdot)}|p_{\vecs_{h-1}}^\tvecnu] - \Gamma^{w,\vecpi}_{\vecM,h-1}(p_{\vecs_{h-1}}^\tvecnu)\|_1\\
    = & \frac{1}{N^w}\|\EE_{\bvecM,\tvecnu}[\bm{\delta}_{s^{w,1}_h = (\cdot)}|p_{\vecs_{h-1}}^\tvecnu] - \Gamma^{w,\vecpi}_{\vecM,h-1}(p_{\vecs_{h-1}}^\tvecnu)\|_1\\
    \leq & \frac{2}{N^w}.\numberthis\label{eq:expected_error}
\end{align*}
where $\bm{\delta}_{s^{w,n}_h = (\cdot)} \in \mR^{\sum_{w}|\cS^w|}$ denotes a vector with 1 at the $(w,n)$-th value and 0 at the others. 
In the equalities, we use the fact that $\bvecM$ and $\vecM$ share the transition function.

Besides, conditioning on $p^\tvecnu_{\vecs_{h-1}}$, for any $\tw\in[W]$ we can treat $\{s^{\tw,n}_h\}_{n\in[N^\tw]}$ as i.i.d. samples according to distribution $\EE[p_{\vecs_h}^{\tw,\tvecnu}|p_{\vecs_{h-1}}^\tvecnu]$.
By applying Lem.~\ref{lem:con_l1_distance} for a fixed $p^\tvecnu_{\vecs_{h-1}}$, $\forall w\in[W],h\in[H]$, for any $\epsilon \in (0,1)$ and $\delta\in(0,1)$, as long as $N^\tw \geq \min\{\frac{8W^2\cS^\tw}{\epsilon^2}, \frac{8W^2}{\epsilon^2}\log\frac{2W}{\delta}\}$ holds for any $\tw\in[W]$, we have:
\begin{align*}
    &\Pr(\|p^\tvecnu_{\vecs_h} - \Gamma^{\vecpi}_{\vecM,h-1}(p^\tvecnu_{\vecs_{h-1}})\|_1 \geq \epsilon) \\
    \leq & \Pr(\|p^\tvecnu_{\vecs_h} - \Gamma^{\vecpi}_{\vecM,h-1}(p^\tvecnu_{\vecs_{h-1}})\| + \|\EE[p^\tvecnu_{\vecs_h}|p^\tvecnu_{\vecs_{h-1}}] - \Gamma^{\vecpi}_{\vecM,h-1}(p^\tvecnu_{\vecs_{h-1}})\|_1 \geq \epsilon)\\
    \leq & \Pr(\|p^\tvecnu_{\vecs_h} - \EE[p^\tvecnu_{\vecs_h}|p^\tvecnu_{\vecs_{h-1}}]\|_1 \geq \frac{\epsilon}{2}) \tag{Eq.~\eqref{eq:expected_error}}\\
    \leq & \sum_{\tw\in[W]}\Pr(\|p^{\tw,\tvecnu}_{\vecs_h} - \EE[p^{\tw,\tvecnu}_{\vecs_h}|p^\tvecnu_{\vecs_{h-1}}]\|_1 \geq \epsilon/2W) \\
    \leq & \delta. \numberthis \label{eq:l1_deviation}
\end{align*}
Note that the number of possible values of $p^\tvecnu_{\vecs_{h-1}}$ can be upper bounded by $\prod_{w\in[W]}(N^w)^{S^w}$.
We define event $\cE := \{\forall h\in[H],w\in[W],~\|p^\tvecnu_{\vecs_h} - \Gamma^{\vecpi}_{\vecM,h-1}(p^\tvecnu_{\vecs_{h-1}})\|_1 < \epsilon\}$. By applying a union bound over $h,w$ and all possible $p^\tvecnu_{\vecs_{h-1}}$, we have:
\begin{align}
    \Pr(\cE) \geq 1 - \delta,\quad \text{as long as }\forall w\in[W],~N^w \geq O(\frac{W^2\sum_{w\in[W]}S^w}{\epsilon^2}\log\frac{2HWS_{\max}}{\delta\epsilon}),\label{eq:choice_of_N}
\end{align}
\paragraph{Density Error Decomposition}
The following discussion are based on the event $\cE$. Recall we use $\vecmu^{\vecpi}_h := \{\mu^{1,\vecpi}_h,...,\mu^{W,\vecpi}_h\}$ to denote the density induced by $\vecpi$ in MT-MFG. Then we have:
\begin{align*}
    &\|p^\tvecnu_{\vecs_h} - \vecmu^{\vecpi}_h\|_1 \\
    =& \|p^\tvecnu_{\vecs_h} - \Gamma^{\vecpi}_{\vecM,h-1}(p^\tvecnu_{\vecs_{h-1}})\|_1 + \|\Gamma^{\vecpi}_{\vecM,h-1}(p^\tvecnu_{\vecs_{h-1}}) - \vecmu_h^{\vecpi}\|_1 \\
    \leq & \epsilon + \sum_{\tw\in[W]}\sum_{s_h^\tw}|\sum_{s_{h-1}^\tw,a_{h-1}^\tw}p^{\tw,\tvecnu}_{\vecs_{h-1}}(s_{h-1}^\tw)\pi^\tw_{h-1}(a_{h-1}^\tw|s_{h-1}^\tw)\mP^\tw_{h-1}(s_h^\tw|s_{h-1}^\tw,a_{h-1}^\tw,p^\tvecnu_{\vecs_{h-1}}) \\
    &\qquad\qquad - \sum_{s_{h-1}^\tw,a_{h-1}^\tw}\mu_{h-1}^{\tw,\vecpi}(s_{h-1}^\tw)\pi^\tw_{h-1}(a_{h-1}^\tw|s_{h-1}^\tw)\mP^\tw_{h-1}(s_h^\tw|s_{h-1}^\tw,a_{h-1}^\tw,\vecmu_{h-1}^{\vecpi})| \\
    \leq & \epsilon + \sum_{\tw\in[W]}\sum_{s_h^\tw}\sum_{s_{h-1}^\tw,a_{h-1}^\tw}|p^{\tw,\tvecnu}_{\vecs_{h-1}}(s_{h-1}^\tw)-\mu_{h-1}^{\tw,\vecpi}(s_{h-1}^\tw)|\pi^\tw_{h-1}(a_{h-1}^\tw|s_{h-1}^\tw)\mP^\tw_{h-1}(s_h^\tw|s_{h-1}^\tw,a_{h-1}^\tw,p^\tvecnu_{\vecs_{h-1}}) \\
    & + \sum_{\tw\in[W]}\sum_{s_{h-1}^\tw,a_{h-1}^\tw}\mu_{h-1}^{\tw,\vecpi}(s_{h-1}^\tw)\pi^\tw_{h-1}(a_{h-1}^\tw|s_{h-1}^\tw)\sum_{s_h^\tw}|\mP^\tw_{h-1}(s_h^\tw|s_{h-1}^\tw,a_{h-1}^\tw,p^\tvecnu_{\vecs_{h-1}}) - \mP^\tw_{h-1}(s_h^\tw|s_{h-1}^\tw,a_{h-1}^\tw,\vecmu_{h-1}^{\vecpi})| \\
    \leq & \epsilon + (1+\vecL_T)\|p^\tvecnu_{\vecs_{h-1}} - \vecmu_{h-1}^{\vecpi}\|_1 \\
    \leq & \frac{(1 + \vecL_T)^h - 1}{\vecL_T} \epsilon.
\end{align*}
\paragraph{Upper Bound of Approximation Error}
Recall the definition of value functions in Def.~\ref{def:MT_SAG} and Def.~\ref{def:MT_MFA}. We focus on the $(w,1)$-agent which takes a potentially deviated policy $\tpi^w$ while the others do not, and we are interested in provide an upper bound for the value difference $J^{(w,1)}_{\bvecM}(\tvecnu) - J^{(w,1)}_\vecM(\vecpi^{-w}\circ\tpi^w;\vecpi)$, which will be useful to characterize the sub-optimality of lifted policy $\vecnu$.

We start from step $h=H$, following the choice of $N^w$ in Eq.~\eqref{eq:choice_of_N},
\begin{align*}
    &\EE_{\bvecM,\tvecnu}[V_{\bvecM,H}^{(w,1),\tvecnu}(s_H^{(w,1)};\vecs_H)-V_{\vecM,H}^{(w,1),\vecpi^{-w}\circ\tpi^w}(s_H^{(w,1)};\vecmu^\vecpi_H)]\\
    =&\EE_{\bvecM,\tvecnu}[r_H^w(s_H^{(w,1)},a_H^{(w,1)};p_{\vecs_H}^\tvecnu)-r_H^w(s_H^{(w,1)},a_H^{(w,1)};\vecmu^\vecpi_H)]\\
    \leq & \vecL_r\EE_{\bvecM,\tvecnu}[\|p_{\vecs_H}^\tvecnu - \vecmu^\vecpi_H\|_1]\leq  \vecL_r(2\delta + \frac{(1+\vecL_T)^H-1}{\vecL_T}\epsilon) =: \epsilon_H.
\end{align*}
For $h < H$, we have:
\begin{align*}
    &\EE_{\bvecM,\tvecnu}[V_{\bvecM,h}^{(w,1),\tvecnu}(s_h^{(w,1)};\vecs_h)-V_{\vecM,h}^{(w,1),\vecpi^{-w}\circ\tpi^w}(s_h^{(w,1)};\vecmu^\vecpi_h)]\\
    =&\EE_{\bvecM,\tvecnu}[r_h^w(s_h^{(w,1)},a_h^{(w,1)};p_{\vecs_h}^\tvecnu)-r_h^w(s_h^{(w,1)},a_h^{(w,1)};\vecmu^\vecpi_h) + \sum_{\vecs_{h+1}}\mP_{\bvecM}(\vecs_{h+1}|\vecs_{h},\veca_{h})V_{\bvecM,h+1}^{(w,1),\tvecnu}(s_{h+1}^{(w,1)};\vecs_{h+1}) \\
    & - \sum_{s^{(w,1)}_{h+1}}\mP_h^w(s^{(w,1)}_{h+1}|s^{(w,1)}_{h},a^{(w,1)}_{h},\vecmu^\vecpi_h)V_{\vecM,h+1}^{(w,1),\vecpi^{-w}\circ\tpi^w}(s_{h+1}^{(w,1)};\vecmu^\vecpi_{h+1})]\\
    \leq & \EE_{\bvecM,\tvecnu}[\vecL_r\|p_{\vecs_h}^{\tvecnu}- \vecmu^\vecpi_h\|_1 + \sum_{\vecs_{h+1}}\mP_{\bvecM}(\vecs_{h+1}|\vecs_{h},\veca_{h})\Big(V_{\bvecM,h+1}^{(w,1),\tvecnu}(s_{h+1}^{(w,1)};\vecs_{h+1}) - V_{\vecM,h+1}^{(w,1),\vecpi^{-w}\circ\tpi^w}(s_{h+1}^{(w,1)};\vecmu^\vecpi_{h+1})\Big)\\
    & + \sum_{\vecs_{h+1}}\mP_{\bvecM}(\vecs_{h+1}|\vecs_{h},\veca_{h})V_{\vecM,h+1}^{(w,1),\vecpi^{-w}\circ\tpi^w}(s_{h+1}^{(w,1)};\vecmu^\vecpi_{h+1}) \\
    & - \sum_{s^{(w,1)}_{h+1}}\mP_h^w(s^{(w,1)}_{h+1}|s^{(w,1)}_{h},a^{(w,1)}_{h},\vecmu^\vecpi_h)V_{\vecM,h+1}^{(w,1),\vecpi^{-w}\circ\tpi^w}(s_{h+1}^{(w,1)};\vecmu^\vecpi_{h+1})]\\
    \leq & \epsilon_{h+1} + \EE_{\bvecM,\tvecnu}[\vecL_r\|p_{\vecs_h}^{\tvecnu}- \vecmu^\vecpi_h\|_1 +  \sum_{\vecs_{h+1}}\mP_{\bvecM}(\vecs_{h+1}|\vecs_{h},\veca_{h})V_{\vecM,h+1}^{(w,1),\vecpi^{-w}\circ\tpi^w}(s_{h+1}^{(w,1)};\vecmu^\vecpi_{h+1}) \\
    & - \sum_{s^{(w,1)}_{h+1}}\mP_h^w(s^{(w,1)}_{h+1}|s^{(w,1)}_{h},a^{(w,1)}_{h},\vecmu^\vecpi_h)V_{\vecM,h+1}^{(w,1),\vecpi^{-w}\circ\tpi^w}(s_{h+1}^{(w,1)};\vecmu^\vecpi_{h+1})]\tag{By induction from $h+1$}\\
    = & \epsilon_{h+1} + \EE_{\bvecM,\tvecnu}[\vecL_r\|p_{\vecs_h}^{\tvecnu}- \vecmu^\vecpi_h\|_1 + \\
    & + \sum_{s_{h+1}^{(w,1)}}\Big(\mP^w_h(s_{h+1}^{(w,1)}|s_{h}^{(w,1)},s_{h}^{(w,1)},p_{\vecs_h}^{\tvecnu}) - \mP_h^w(s^{(w,1)}_{h+1}|s^{(w,1)}_{h},a^{(w,1)}_{h},\vecmu^\vecpi_h)\Big)V_{\vecM,h+1}^{(w,1),\vecpi^{-w}\circ\tpi^w}(s_{h+1}^{(w,1)};\vecmu^\vecpi_{h+1})]\\
    \leq & \epsilon_{h+1} + (\vecL_r + \vecL_T) \EE_{\bvecM,\tvecnu}[\|p_{\vecs_h}^{\tvecnu}- \vecmu^\vecpi_h\|_1 ]\leq \epsilon_{h+1} + (\vecL_r + \vecL_T)(2\delta + \frac{(1 + \vecL_T)^h - 1}{\vecL_T}\epsilon) =: \epsilon_h.
\end{align*}
where we use $\mP_{\bvecM}(\vecs_{h+1}|\vecs_h,\veca_h) := \prod_{w\in[W]}\prod_{n\in[N]}\mP_h^w(s^{(w,1)}_{h+1}|s^{(w,1)}_h,a^{(w,1)}_h,p_{\vecs_h})$ to denote the dynamics in MF-SAG.
Therefore, for $h=1$, note that $\bvecM$ and $\vecM$ have the same initial distribution, and we have:
\begin{align*}
    J^{(w,1)}_{\bvecM}(\tvecnu) - J^{(w,1)}_\vecM(\vecpi^{-w}\circ\tpi^w;\vecpi) \leq 2\delta (\vecL_r+\vecL_T) H + 2(\vecL_r+\vecL_T)H\frac{(1 + \vecL_T)^H - 1}{\vecL_T}\epsilon.
\end{align*}
Given an $\teps$ NE policy in $\vecM$, denoted by $\vecpi$, consider the lifted policy $\vecnu$ and a deviation policy $\tvecnu$ agrees with $\vecnu$ except that it takes some $\tpi^w$ for agent with index $(w,1)$.
By choosing $\delta = \frac{\epsilon_0}{8(\vecL_r+\vecL_T)H}$ and $\epsilon = \epsilon_0 / \Big(4(\vecL_r + \vecL_T)H\frac{(1 + \vecL_T)^H - 1}{\vecL_T}\Big)$, 
we have:
\begin{align*}
    \max_{\tpi^w} J^{(w,1)}_{\bvecM}(\tvecnu) - J^{(w,1)}_{\bvecM}(\vecnu) \leq & \max_{\tpi^w} J^{(w,1)}_{\bvecM}(\tvecnu) - J^{(w,1)}_{\bvecM}(\vecnu) - \Big(\max_{\tpi^w} J^{(w,1)}_\vecM(\vecpi^{-w}\circ\tpi^w;\vecpi) - J^{(w,1)}_\vecM(\vecpi;\vecpi)\Big) + \teps\\
    \leq & 2 \max_{\tpi^w}|J^{(w,1)}_{\bvecM}(\tvecnu) - J^{(w,1)}_\vecM(\vecpi^{-w}\circ\tpi^w;\vecpi)| + \teps \leq \epsilon_0 + \teps.
\end{align*}
To satisfy the requirements in $\delta$ and $\epsilon_0$, we need:
\begin{align*}
    \forall w\in[W],\quad N^w \geq O((\vecL_r + \vecL_T)^2W^2H^3\Big(\frac{(1 + \vecL_T)^H - 1}{\vecL_T}\Big)^2\frac{\sum_{w\in[W]}S^w}{\epsilon^2_0}\log\frac{2(\vecL_r+\vecL_T)(1+\vecL_T)HWS_{\max}}{\epsilon_0}).
\end{align*}
\end{proof}
\newpage
\section{Basic Lemma}\label{appx:basic_lemma}
\subsection{Lemma from \citep{huang2023statistical}}
\begin{lemma}[Lem. D.4 in \citep{huang2023statistical}]\label{lem:concentration}
    Let $X_1,X_2,...$ be a sequence of random variable taking value in $[0,C]$ for some $C \geq 1$. Define $\cF_k = \sigma(X_1,..,X_{k-1})$ and $Y_k = \EE[X_k|\cF_k]$ for $k\geq 1$. For any $\delta > 0$, we have:
    \begin{align*}
        \Pr(\exists n \sum_{k=1}^n X_k \leq 3 \sum_{k=1}^n Y_k + C \log \frac{1}{\delta}) \leq \delta,\quad 
        \Pr(\exists n \sum_{k=1}^n Y_k \leq 3 \sum_{k=1}^n X_k + C \log \frac{1}{\delta}) \leq \delta.
    \end{align*}
\end{lemma}

\begin{lemma}[Lem. 4.6 in \citep{huang2023statistical}]\label{lem:exploitability_diff}
    Under Assump.~\ref{assump:Lipschitz}, given two arbitrary model $M$ and $\tM$, and two policies $\pi$ and $\tpi$, we have:
    \begin{align*}
        |\Delta_M(\tilde\pi,\pi) - \Delta_\tM(\tilde\pi,\pi)| \leq & \EE_{\tpi,M(\pi)}[\sum_{h=1}^H \|\mP_{M,h}(\cdot|s_h,a_h,\mu^\pi_{M,h}) - \mP_{\tM,h}(\cdot|s_h,a_h,\mu^\pi_{\tM,h})\|_1]\\
        + (2L_rH + 1) & \EE_{\pi,M(\pi)}[\sum_{h=1}^H  \|\mP_{M,h}(\cdot|s_h,a_h,\mu^\pi_{M,h}) - \mP_{\tM,h}(\cdot|s_h,a_h,\mu^\pi_{\tM,h})\|_1].\numberthis\label{eq:exploitability_diff}
    \end{align*}
\end{lemma}

\subsection{Other Lemma}

\begin{lemma}[Density Difference Lemma]\label{lem:density_differences}
    Given arbitrary Multi-Type Mean-Field MDPs $\vecM$ and $\vecM'$, and two arbitrary policies $\vecpi$ and $\vecpi'$, for any $h\in[H]$, we have:
    \begin{align*}
        \|\vecmu^{\vecpi}_{\vecM,h} - \vecmu^{\vecpi'}_{\vecM',h}\|_1 \leq & \|\vecmu^{\vecpi}_{\vecM,h-1} - \vecmu^{\vecpi'}_{\vecM,h-1}\|_1 + W\cdot d_{\infty,1}(\vecpi,\vecpi') \\
        & + \sum_{w\in[W]} \EE_{\vecpi,\vecM(\vecpi)}[\|\mP^w_{\vecM,h}(\cdot|s_{h-1}^w,a_{h-1}^w,\vecmu^{\vecpi}_{\vecM,h-1})-\mP^w_{\vecM',h}(\cdot|s_{h-1}^w,a_{h-1}^w,\vecmu^{\vecpi'}_{\vecM',h-1})\|_1].
    \end{align*}
\end{lemma}
\begin{proof}
    For any $w\in[W]$, we have:
    \begin{align*}
        &\|\mu^{w,\vecpi}_{\vecM,h} - \mu^{w,\vecpi'}_{\vecM',h}\|_1 \\
        =& |\sum_{s_h^w}\Big(\sum_{s_{h-1}^w,a_{h-1}^w}\mu^{w,\vecpi}_{\vecM,h}(s_{h-1}^w)\pi^w_{h-1}(a_{h-1}^w|s_{h-1}^w)\mP^w_{\vecM,h}(s_h^w|s_{h-1}^w,a_{h-1}^w,\vecmu^{\vecpi}_{\vecM,h-1}) \\
        &\quad - \sum_{s_{h-1}^w,a_{h-1}^w}\mu^{w,\vecpi'}_{\vecM',h}(s_h^w)\pi'^w_{h-1}(a_{h-1}^w|s_{h-1}^w)\mP^w_{\vecM',h}(s_h^w|s_{h-1}^w,a_{h-1}^w,\vecmu^{\vecpi'}_{\vecM',h-1})\Big)|\\
        \leq & |\Big(\sum_{s_{h-1}^w} \mu^{w,\vecpi}_{\vecM,h}(s_{h-1}^w) - \mu^{w,\vecpi'}_{\vecM',h}(s_{h-1}^w)\Big)\sum_{a_{h-1}^w}\pi'^w_{h-1}(a_{h-1}^w|s_{h-1}^w)\sum_{s_{h}}\mP^w_{\vecM',h}(s_h^w|s_{h-1}^w,a_{h-1}^w,\vecmu^{\vecpi}_{\vecM',h-1})| \\
        & + |\sum_{s_{h-1}^w}\mu^{w,\vecpi}_{\vecM,h}(s_{h-1}^w)\sum_{a_{h-1}^w}\Big(\pi^w_{h-1}(a_{h-1}^w|s_{h-1}^w) - \pi'^w_{h-1}(a_{h-1}^w|s_{h-1}^w)\Big)\sum_{s_h^w}\mP^w_{\vecM',h}(s_h^w|s_{h-1}^w,a_{h-1}^w,\vecmu^{\vecpi}_{\vecM',h-1})|\\
        & + \sum_{s_{h-1}^w,a_{h-1}^w}\mu^{w,\vecpi}_{\vecM,h}(s_h^w)\pi^w_{h-1}(a_{h-1}^w|s_{h-1}^w)\sum_{s_h^w}|\mP^w_{\vecM,h}(s_h^w|s_{h-1}^w,a_{h-1}^w,\vecmu^{\vecpi}_{\vecM,h-1})-\mP^w_{\vecM',h}(s_h^w|s_{h-1}^w,a_{h-1}^w,\vecmu^{\vecpi'}_{\vecM',h-1})| \tag{Assump.~\ref{assump:Lipschitz}}\\
        \leq & \|\mu^{w,\vecpi}_{\vecM,h-1} - \mu^{w,\vecpi'}_{\vecM',h-1}\|_1 + d_{\infty,1}(\vecpi,\vecpi') + \EE_{\vecpi,\vecM(\vecpi)}[\|\mP^w_{\vecM,h}(\cdot|s_{h-1}^w,a_{h-1}^w,\vecmu^{\vecpi}_{\vecM,h-1})-\mP^w_{\vecM',h}(\cdot|s_{h-1}^w,a_{h-1}^w,\vecmu^{\vecpi'}_{\vecM',h-1})\|_1]
    \end{align*}
    By repeating the above discussion for every $w\in[W]$, we have:
    \begin{align*}
        &\|\vecmu^{\vecpi}_{\vecM,h} - \vecmu^{\vecpi'}_{\vecM',h}\|_1 \\
        \leq& \|\vecmu^{\vecpi}_{\vecM,h-1} - \vecmu^{\vecpi'}_{\vecM,h-1}\|_1 + W\cdot d_{\infty,1}(\vecpi,\vecpi') \\
        & + \sum_{w\in[W]} \EE_{\vecpi,\vecM(\vecpi)}[\|\mP^w_{\vecM,h}(\cdot|s_{h-1}^w,a_{h-1}^w,\vecmu^{\vecpi}_{\vecM,h-1})-\mP^w_{\vecM',h}(\cdot|s_{h-1}^w,a_{h-1}^w,\vecmu^{\vecpi'}_{\vecM',h-1})\|_1].
    \end{align*}
\end{proof}
\begin{lemma}\label{lem:density_est_err_MT}
    Given two model $\vecM$ and $\vecM'$ and a policy $\vecpi$, for any $h\in[H],w\in[W]$, we have:
    \begin{align}
        \|\vecmu^\vecpi_{\vecM,h+1} - \vecmu^\vecpi_{\vecM',h+1}\|_1 \leq& \sum_{w\in[W]} \EE_{\vecpi,\vecM(\vecpi)}[\sum_{\ph=1}^h \|\mP_{\vecM,\ph}^w(\cdot|\cdot,\cdot,\vecmu^\vecpi_{\vecM,\ph}) - \mP_{\vecM',\ph}^w(\cdot|\cdot,\cdot,\vecmu^\vecpi_{\vecM',\ph})\|_1].\label{eq:density_diff1_MT}
    \end{align}
    Besides, under Assump.~\ref{assump:Lipschitz}, we have:
    \begin{align}
        \|\vecmu^\vecpi_{\vecM,h+1} - \vecmu^\vecpi_{\vecM',h+1}\|_1 \leq& \sum_{w\in[W]} \EE_{\vecpi,\vecM(\vecpi)}[\sum_{\ph=1}^h (1+\vecL_T)^{h-\ph} \|\mP_{\vecM,\ph}^w(\cdot|\cdot,\cdot,\vecmu^\vecpi_{\vecM,\ph}) - \mP_{\vecM',\ph}^w(\cdot|\cdot,\cdot,\vecmu^\vecpi_{\vecM,\ph})\|_1].
        \label{eq:density_diff2_MT}
    \end{align}
\end{lemma}
\begin{proof}
    By applying Lem.~\ref{lem:density_differences} to the case when $\vecpi = \vecpi'$, and combining with Assump.~\ref{assump:Lipschitz_MT}, we finish the proof.
\end{proof}
\begin{lemma}\label{lem:density_est_err}
    Given two model $M$ and $M'$, and two arbitrary policies $\pi$ and $\pi'$, for any $h\in[H]$, we have:
    \begin{align*}
        &\|\mu^{\pi}_{M,h} - \mu^{\pi'}_{M',h}\|_1 \\
        \leq& \|\mu^{\pi}_{M,h-1} - \mu^{\pi'}_{M,h-1}\|_1 + d_{\infty,1}(\pi,\pi') + \EE_{\pi,M(\pi)}[\|\mP_{M,h}(\cdot|s_{h-1},a_{h-1},\mu^{\pi}_{M,h-1})-\mP_{M',h}(\cdot|s_{h-1},a_{h-1},\mu^{\pi'}_{M',h-1})\|_1].
    \end{align*}
    Moreover, as a special case when $\pi=\pi'$, we have:
    \begin{align}
        \|\mu^\pi_{M,h+1} - \mu^\pi_{M',h+1}\|_1 \leq& \EE_{\pi,M(\pi)}[\sum_{\ph=1}^h \|\mP_{M,\ph}(\cdot|s_\ph,a_\ph,\mu^\pi_{M,\ph}) - \mP_{M',\ph}(\cdot|s_\ph,a_\ph,\mu^\pi_{M',\ph})\|_1].\label{eq:density_diff1}
    \end{align}
    Besides, under Assump.~\ref{assump:Lipschitz}, we have:
    \begin{align}
        \|\mu^\pi_{M,h+1} - \mu^\pi_{M',h+1}\|_1 \leq \EE_{\pi,M(\pi)}[\sum_{\ph=1}^h (1+L_T)^{h-\ph} \|\mP_{M,\ph}(\cdot|s_\ph,a_\ph,\mu^\pi_{M,\ph}) - \mP_{M',\ph}(\cdot|s_\ph,a_\ph,\mu^\pi_{M,\ph})\|_1].\label{eq:density_diff2}
    \end{align}
\end{lemma}
\begin{proof}
    The proof is simply completed by setting $W=1$ in Lem.~\ref{lem:density_est_err_MT}.
\end{proof}

\begin{lemma}[Concentration w.r.t. $l_1$-distance]\label{lem:con_l1_distance}
    Given a discrete domain $\cX$ and a distribution $p$ on $\cX$, suppose we draw $N$ i.i.d. samples $\{x^n\}_{n\in[N]}$ from $p$ and provide an estimation $\hat p \in \Delta(\cX)$ with $\hat p(x) = \frac{1}{N}\sum_{n=1}^N \delta(x^n = x)$, then for any $\delta\in(0,1)$ and $\epsilon > 0$, as long as $N \geq \max\{\frac{2|\cX|}{\epsilon^2}, \frac{2}{\epsilon^2}\log\frac{2}{\delta}\}$, we have:
    $$
    \mP(\|p - \hat p\|_1 \geq \epsilon) \leq \delta.
    $$
\end{lemma}
\begin{proof}
    We first provide an upper bound for $\EE[\|p - \hat p\|_1]$:
    \begin{align*}
        \EE[\|p - \hat p\|_1] =&  \sum_{x\in\cX}\EE[|p(x) - \hat p (x)|] \leq \sum_{x\in\cX}\sqrt{\EE[|p(x) - \hat p (x)|^2]} = \sqrt{\frac{1}{N}}\sum_{x\in\cX}\sqrt{p(x)(1-p(x))} \leq \frac{1}{2}\sqrt{\frac{|\cX|}{N}}.
    \end{align*}
    where we use the fact that $\hat p(x)$ is a Bernoulli random variable with mean $p(x)$ and variance $\frac{1}{N}p(x)(1-p(x))$.

    Next, and note that deviation of any $x^n$ will only result in $2/N$ deviation of $\|p - \hat p\|_1$. By McDiarmid's inequality, for any $\epsilon$, we have:
    \begin{align*}
        \mP(\|p - \hat p\|_1 \geq \frac{\epsilon}{2} + \frac{1}{2}\sqrt{\frac{2|\cX|}{N}}) \leq \mP(|\|p - \hat p\|_1 - \EE[\|p - \hat p\|_1]| \geq \frac{\epsilon}{2}) \leq 2 \exp(-\frac{N}{2}\epsilon^2),
    \end{align*}
    By assigning appropriate values for $N$, we finish the proof.
\end{proof}

\begin{restatable}{lemma}{LemModelDiff}[Model Difference Lemma]\label{lem:model_diff_conversion}
    For any policies $\tpi$, $\pi$ and $\pi'$, and any bounded functions $f_1,f_2,...,f_H\in\{f|f:\cS\times\cA\rightarrow [0,1]\}$, 
    
    (i) Given any two MF-MDPs $M$ and $M'$, we have:
    \begin{align*}
        \Big|\EE_{\tpi,M(\pi)}[\sum_{h=1}^H f_h(s_h,a_h)] - \EE_{\tpi,M'(\pi')}[\sum_{h=1}^H f_h(s_h,a_h)]\Big| \leq & H\cdot \EE_{\tpi,M(\pi)}[\sum_{h=1}^H\|\mP_{M}(\cdot|s_{h},a_{h},\mu^\pi_{M,h})-\mP_{M'}(\cdot|s_{h},a_{h},\mu^{\pi'}_{M',h})\|_1].
    \end{align*}

    (ii) Given any two PAMs $\dM$ and $\dM'$, we have:
    \begin{align*}
        \Big|\EE_{\tpi,\dM(\pi)}[\sum_{h=1}^H f_h(s_h,a_h)] - \EE_{\tpi,\dM'(\pi')}[\sum_{h=1}^H f_h(s_h,a_h)]\Big| \leq & H\cdot \EE_{\tpi,\dM(\pi)}[\sum_{h=1}^H\|\dmP_{\dM}(\cdot|s_{h},a_{h},\pi)-\dmP_{\dM'}(\cdot|s_{h},a_{h},\pi')\|_1].
    \end{align*}
\end{restatable}
\begin{proof}
    We first proof (ii). We use $\mu^\pi_{\dM(\pi'),h}$ to denote the density induced by $\pi$ in model $\dM$ given $\pi'$ as the reference policy.
    \begin{align*}
        &\|\mu^\tpi_{\dM(\pi),h} - \mu^\tpi_{\dM'(\pi'),h}\|_1 \\
        = & |\sum_{s_h}\Big( \sum_{s_{h-1},a_{h-1}} \mu^\tpi_{\dM(\pi),h-1}(s_{h-1}) \tpi(a_{h-1}|s_{h-1})\dmP_{\dM}(s_h|s_{h-1},a_{h-1},\pi) \\
        &\quad - \sum_{s_{h-1},a_{h-1}} \mu^\tpi_{\dM'(\pi'),h-1}(s_{h-1})\tpi(a_{h-1}|s_{h-1})\dmP_{\dM'}(s_h|s_{h-1},a_{h-1},\pi')\Big)| \\
        \leq & |\sum_{s_h}\sum_{s_{h-1},a_{h-1}} \mu^\tpi_{\dM(\pi),h-1}\tpi(a_{h-1}|s_{h-1})(\dmP_{\dM}(s_h|s_{h-1},a_{h-1},\pi)-\dmP_{\dM'}(s_h|s_{h-1},a_{h-1},\pi'))|\\
        & + |\sum_{s_h} \sum_{s_{h-1},a_{h-1}} (\mu^\tpi_{\dM(\pi),h-1}(s_{h-1}) - \mu^\tpi_{\dM'(\pi'),h-1}(s_{h-1}))\tpi(a_{h-1}|s_{h-1})\dmP_{\dM'}(s_h|s_{h-1},a_{h-1},\pi')| \\
        \leq & \EE_{\tpi,\dM(\pi)}[\|\dmP_{\dM}(\cdot|s_{h-1},a_{h-1},\pi)-\dmP_{\dM'}(\cdot|s_{h-1},a_{h-1},\pi')\|_1] + \|\mu^\tpi_{\dM(\pi),h-1} - \mu^\tpi_{\dM'(\pi'),h-1}\|_1\\
        \leq & \EE_{\tpi,\dM(\pi)}[\sum_{\ph=1}^{h-1}\|\dmP_{\dM}(\cdot|s_{\ph},a_{\ph},\pi)-\dmP_{\dM'}(\cdot|s_{\ph},a_{\ph},\pi')\|_1].
    \end{align*}
    Therefore,
    \begin{align*}
        \Big|\EE_{\tpi,\dM(\pi)}[\sum_{h=1}^H f_h(s_h,a_h)] - \EE_{\tpi,\dM'(\pi')}[\sum_{h=1}^Hf_h(s_h,a_h)]\Big| \leq& \sum_{h=1}^H \|\mu^\tpi_{\dM(\pi),h} - \mu^\tpi_{\dM'(\pi'),h}\|_1 \\
        \leq & H\cdot \EE_{\tpi,\dM(\pi)}[\sum_{h=1}^H\|\dmP_{\dM}(\cdot|s_{h},a_{h},\pi)-\dmP_{\dM'}(\cdot|s_{h},a_{h},\pi')\|_1].
    \end{align*}
    The proof for (i) can be directly obtained by replacing $\dmP_{\dM,h}(\cdot|\cdot,\cdot,\pi)$ and $\dmP_{\dM',h}(\cdot|\cdot,\cdot,\pi')$ with $\mP_{M,h}(\cdot|\cdot,\cdot,\mu^\pi_{M,h})$ and $\mP_{M',h}(\cdot|\cdot,\cdot,\mu^\pi_{M',h})$.

\end{proof}

\begin{restatable}{lemma}{LemChangeOfMeasure}\label{lem:change_of_measure}
    Given three arbitrary models $M,\tM,\bM$, and two arbitrary policies $\pi,\tpi$, we have:
    \begin{align*}
        &\EE_{\tpi,\tM(\pi)}[\sum_{h=1}^H\|\mP_{\tM,h}(\cdot|s_h,a_h,\mu^\pi_{\tM,h}) - \mP_{\bM,h}(\cdot|s_h,a_h,\mu^\pi_{\bM,h})\|_1] \\
        \leq & \EE_{\tpi, M(\pi)}[\sum_{h=1}^H \|\mP_{M,h}(\cdot|s_h,a_h,\mu^\pi_{M,h}) - \mP_{\bM,h}(\cdot|s_h,a_h,\mu^\pi_{\bM,h})\|_1] \\
        & + (H+1) \cdot \EE_{\tpi,M(\pi)}[\sum_{h=1}^H\|\mP_{M,h}(\cdot|s_h,a_h,\mu^\pi_{M,h}) - \mP_{\tM,h}(\cdot|s_h,a_h,\mu^\pi_{\tM,h})\|_1].\numberthis\label{eq:md_1}
    \end{align*}
\end{restatable}
\begin{proof}
    By applying Lem.~\ref{lem:model_diff_conversion} with $f_h(s_h,a_h) = \|\mP_{\tM,h}(\cdot|s_h,a_h,\mu^\pi_{\tM,h}) - \mP_{\bM,h}(\cdot|s_h,a_h,\mu^\pi_{\bM,h})\|_1$, we have:
    \begin{align*}
        &\EE_{\tpi,\tM(\pi)}[\sum_{h=1}^H\|\mP_{\tM,h}(\cdot|s_h,a_h,\mu^\pi_{\tM,h}) - \mP_{\bM,h}(\cdot|s_h,a_h,\mu^\pi_{\bM,h})\|_1]\\
        \leq & \EE_{\tpi,M(\pi)}[\sum_{h=1}^H\|\mP_{\tM,h}(\cdot|s_h,a_h,\mu^\pi_{\tM,h}) - \mP_{\bM,h}(\cdot|s_h,a_h,\mu^\pi_{\bM,h})\|_1] \\
        & + H\cdot \EE_{\tpi,M(\pi)}[\sum_{h=1}^H\|\mP_{M,h}(\cdot|s_h,a_h,\mu^\pi_{M,h}) - \mP_{\tM,h}(\cdot|s_h,a_h,\mu^\pi_{\tM,h})\|_1]\\
        \leq & \EE_{\tpi,M(\pi)}[\sum_{h=1}^H\|\mP_{M,h}(\cdot|s_h,a_h,\mu^\pi_{M,h}) - \mP_{\bM,h}(\cdot|s_h,a_h,\mu^\pi_{\bM,h})\|_1] \\
        & + (H+1) \EE_{\tpi,M(\pi)}[\sum_{h=1}^H\|\mP_{M,h}(\cdot|s_h,a_h,\mu^\pi_{M,h}) - \mP_{\tM,h}(\cdot|s_h,a_h,\mu^\pi_{\tM,h})\|_1]
        . \numberthis\label{eq:decompose_1}
    \end{align*}
\end{proof}
\newpage
\section{Details of Experiments}\label{appx:experiments}

\subsection{Algorithm Design}
In the following, we provide the missing algorithm details for Sec.~\ref{sec:experiments}.

\begin{algorithm}[h]
    \textbf{Input}: Model Class $\cM$; Accuracy level $\epsilon_0,\teps,\beps$; Confidence level $\delta$; Batch size $T$\\
    $\cM^1 \gets \cM$; $\delta_0 \gets \frac{\delta}{\log_2|\cM| + 1}$\\
    $\forall M \in \cM^1$, compute (one of) its NE policy $\pi^{\NE}_M \gets \texttt{NE\_Compute}(M)$.\\
    \For{$k=1,2,...,$}{
        \If{$\exists M^k \in \cM^k$, s.t. $\max_{\tM\in\cM^k} |\cB^{\epsilon_0}_{\pi^\NE_{M^k}}(\tM,\cM^k)| \leq \frac{|\cM^k|}{2}$}{
            $\cM^{k+1} \gets \texttt{ModelElim\_Exp}(\pi^{\NE}_{M^k},\cM^k,\teps,\delta_0, T)$. 
        }
        \Else{
            $M^k \gets \argmax_{M\in\cM^k} |\cB^{\epsilon_0}_{\pi^{\NE}_M}(M,\cM^k)|$,\\
            $\pi^\NE_\Bridge \gets \pi^\NE_{M^k}$, \label{line:simple_argmax}\\
            $\cM^{k+1} \gets \texttt{ModelElim\_Exp}(\pi^\NE_\Bridge, \cM^k,\teps,\delta_0, T)$.\\
            \lIf{$M^k \in \cM^{k+1}$}{
                \Return $\pi^\NE_{M^k}$
            }
        }
    }
    \caption{A Heuristic Oracle-Efficient NE Finding}\label{alg:empirical_version}
\end{algorithm}

Here \texttt{ModelElim\_Exp} is the same algorithm as Alg.~\ref{alg:elimination_DCP_formal} except that we replac Line~\ref{line:max_margin_formal} with:
\begin{align*}
    \tpi^t & \gets \arg\max_{\tpi\in\Pi^\NE}\max_{M,M' \in \bcM^t}\EE_{\tpi,M(\pi)}[\sum_{h=1}^H \|\mP_{M,h}(\cdot|\cdot,\cdot,\mu^\pi_{M,h}) \\
    &\qquad\qquad\qquad\qquad\qquad\qquad - \mP_{M',h}(\cdot|\cdot,\cdot,\mu^\pi_{M',h})\|_1],
\end{align*}
In another word, we only consider policies from $\Pi^\NE := \{\pi^\NE_M\}_{M\in\bcM}$, including the NE policies of models in $\bcM$.

\subsection{Experiment Setup}\label{appx:exp_set_up}

\paragraph{Environments}
We consider the linear style MFG, such that 
$$
\mP(s_{h+1}|s_h,a_h,\mu_h) = \frac{|\phi(s_h,a_h)\trans G(\mu_h)\psi(s_{h+1})|}{\sum_{s_{h+1}}|\phi(s_h,a_h)\trans G(\mu_h)\psi(s_{h+1})|},
$$
where $\phi\in\mR^{d_\phi}$ and $G(\cdot)\in\mR^{d_\phi\times d_\psi}$ are known but $\psi\in\mR^{d_\psi}$ are unknown.
Note that our environment is different from linear model in Prop.~\ref{prop:Linear_MF_MDP}, where features are self-normalized.
We choose $H=3,S=100,A=50$ and $d_\phi=d_\psi=5$, where the number of states and actions is much larger than the feature dimension.
We consider a model set with $|\cM|=200$.

To construct the environment, for each $h$, we first generate a random matrix $\Phi_h \in \mR^{SA\times d_\phi}$ using as feature $\phi(s_h,a_h)$, and generate another random matrix $U_h \in \mR^{S\times d_\phi d_\psi}$, and define the function $G_h(\mu_h)$ by
\begin{align*}
    \forall \mu_h\in\Delta(\cS_h),\quad G_h(\mu_h) := (\mu_h\trans U_h).\text{reshape}(d_\phi, d_\psi) \in \mR^{d_\phi, d_\psi}.
\end{align*}
After that, we generate 200 random matrices $\{\Psi_h^i\}_{i\in[200]}$ with $\Psi_h^i \in \mR^{d_\psi\times S}$ as the next feature function.
Then, the model class is specified by $\cM := \{(\Phi_h, U_h, \Psi_h^i)\}$.
In order to make the model elimination process more challenging, $\{\Psi_h^i\}_{i=2,...,200}$ is generated by randomly perturbing from $\Psi_h^1$, i.e.:
\begin{align*}
    \Psi_h^i = (1-\beta)\tilde{\Psi}_h^i + \beta \Psi_h^1,
\end{align*}
where $\tilde{\Psi}_h^i$ is a random matrix independent w.r.t. $\Psi_h^1$ and $\beta \sim \text{Uniform}(0,0.1)$.
In this way, the difference between models in $\cM$ will be small and harder to distinguish.

\paragraph{Training Procedure}
We construct 5 model classes $\cM^1,...,\cM^5$ with different $\Phi,U$ to increase the randomness in experiments.
For each model class $\cM^i$, we repeat 5 trials, where in each trial, we first randomly select one model from $\cM^i$ as the true model, and run Alg.~\ref{alg:empirical_version} for model elimination.

We set $\epsilon$=1e-3, i.e. we want to find a 1e-3-approximate NE.
Besides, we set batch size $T=50$, $\delta = 0.001$.
For the \texttt{NE\_Oracle} in Alg.~\ref{alg:empirical_version}, we implement it by repeatedly update 
\begin{equation}
    \pi_{i+1} \gets (1-\alpha)\pi_i + \alpha \text{BestReponse}(\pi_i; M). \label{eq:update_rules}
\end{equation}
where $\alpha = 0.02$, and $\text{BestReponse}(\pi_i; M)$ return the policy maximizing the NE gap of $\pi_i$ in $M$.
We stop the update process as long as $\cE^\NE_M(\pi_{i}) \leq$ 5e-4.

\paragraph{Experiments Results}
We provide our experiment results in Fig.~\ref{fig:experiments}.
On the LHS, we report the number of uneliminated models verses the number of trajectories consumed, and as we can see, our algorithm can eliminate unqualified models very quickly.
The total consumed trajectories is much less than the number of states actions $SA = 100 * 50 = 5000$.

On the RHS\footnote{In the RHS sub-plot of Fig.~\ref{fig:experiments}, we set the normalized NE Gap to 0 as long as it is lower than 1e-3}, we report the normalized worst case NE Gap w.r.t. the remaining models.
At each iteration $t$, we compute the NE gap for every uneliminated model's NE policy, and pick out the largest one denoted as $\text{Gap}_{t}$.
The normalized gap is defined to be $\frac{\text{Gap}_{t}}{\text{Gap}_{0}}$, where the normalization term $\text{Gap}_{0}$ is the maximal NE gap at the beginning of the algorithm, i.e. the worst NE gap without starting the algorithm.
As we can see, our algorithm can gradually eliminate inaccurate models and return the (approximate) NE.

\begin{figure*}[h]
    \centering
    \includegraphics[width=0.7\textwidth]{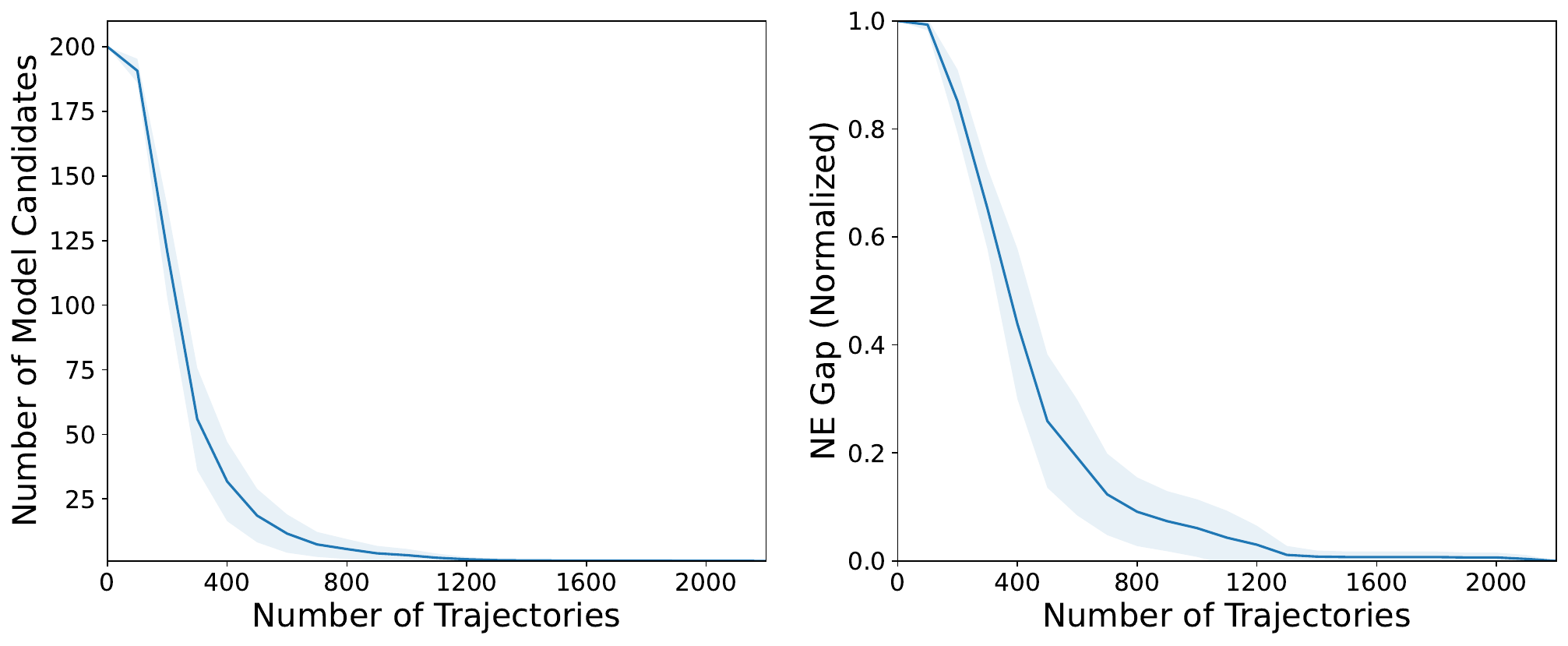}
    \caption{Experiment results in linear style MFG. We report the number of remaining models and the normalized maximal NE Gap by the NE policies of remaining models during the model elimination process. Error bars correspond to 95\% confidence intervals.}\label{fig:experiments}
\end{figure*} 

\end{document}